\documentclass[twoside,11pt]{article}

\usepackage[T1]{fontenc}							
\usepackage[utf8]{inputenc}
\usepackage{amsmath}
\usepackage{amssymb}
\usepackage{graphicx}
\usepackage{xcolor}
\usepackage{xspace}
\usepackage{tikz}
\usepackage{times}
\usepackage{geometry}
\usepackage{color,colortbl}

\usetikzlibrary{decorations.pathreplacing}
\usetikzlibrary{fit,backgrounds}
\usetikzlibrary{arrows,decorations.pathmorphing,backgrounds,positioning,fit,matrix,calc}

\usepackage{pgfplots}
\usepackage{geometry}
\usepackage{algorithmicx}
\usepackage{subfigure}
\usetikzlibrary{external}
\usepackage{jmlr2e}
\usepackage[pdftex]{hyperref}

\ShortHeadings{Cells in MDRNNs <\the\day.\the\month.\the\year>}{Gundram Leifert et al. <\the\day.\the\month.\the\year>}
\jmlrheading{<no>}{<jear>}{<pagerange>}{5/14; Revised 12/15}{<publishdate>}{Gundram Leifert, Tobias Strau\ss, Tobias Gr\"uning, Welf Wustlich, Roger Labahn} 
\title{Cells in Multidimensional Recurrent Neural Networks} 
\author{\name Gundram Leifert \email gundram.leifert@uni-rostock.de\\
\name Tobias Strau\ss \email tobias.strauss@uni-rostock.de\\
\name Tobias Gr\"uning \email tobias.gruening@uni-rostock.de\\
\name Welf Wustlich \email welf.wustlich@uni-rostock.de\\
\name Roger Labahn \email roger.labahn@uni-rostock.de\\
\addr University of Rostock\\
Institute of Mathematics\\
18051 Rostock, Germany}
\editor{Yoshua Bengio}

\newcommand{\LTD}{NVG\xspace}
\newcommand{\NGEC}{NEG\xspace}

\newcommand{\COD}{COD\xspace}
\newcommand{\IR}{\mathbb{R}}

\newcommand{\IZ}{\mathbb{Z}}
\newcommand{\IN}{\mathbb{N}}

\newcommand{\ZZ}{\mathcal{Z}}
\newcommand{\FF}{\mathcal{F}}

\newcommand{\fp}{{\pmb{p}}}

\newcommand{\fq}{{\pmb{q}}}

\newcommand{\fy}{\pmb{y}_\Gamma}

\newcommand{\la}{\lambda}

\newcommand{\eps}{\varepsilon}
\newcommand{\fconv}[1]{
\ifthenelse{\equal{#1}{}}{g_{conv}}{g_{conv}\left(#1\right)}
}
\newcommand{\frein}[1]{
\ifthenelse{\equal{#1}{}}{g_{int}}{g_{int}\left(#1\right)}
}
\newcommand{\fraus}[1]{
\ifthenelse{\equal{#1}{}}{g_{out}}{g_{out}\left(#1\right)}
}
\newcommand{\fl}{f_{log}}
\newcommand{\foutshort}{h_c}
\newcommand{\fout}[1]{h_c\left(#1\right)}
\newcommand{\foutder}[1]{h'_c\left(#1\right)}
\newcommand{\oneD}{\(1\)D\xspace}

\newcommand{\multiD}{MD\xspace}
\newcommand{\der}{derivative\xspace}
\newcommand{\asignal}{a signal\xspace}

\newcommand{\eqtr}{\underset{tr}{=}}
\newcommand{\intr}{\underset{tr}{\in}}
\newcommand{\cin}{{c_{in}}}
\newcommand{\yin}{y_{\cin}}
\newcommand{\mc}[3]{\multicolumn{#1}{#2}{#3}}
\definecolor{color1}{rgb}{0.7,0.7,0.7}
\definecolor{color2}{rgb}{1.0,0.0,0.0}
\definecolor{color3}{rgb}{0.0,0.0,1.0}
\definecolor{colorb}{rgb}{1.0,1.0,1.0}

\def\sigmoid[#1]{
\draw[shift={#1}] (-0.3,-0.3) parabola (0,0) parabola [bend at end]  (0.3,0.3);}
\newcommand{\drawtris}[8]
{
  \draw[->,shift={(#1,#2)},rotate=#4] (-#5-#8,.9*#7) -- (-#5,.9*#7);
  \draw[->,shift={(#1,#2)},rotate=#4] (-#5-#8,.6*#7) -- (-#5,.6*#7);
  \draw[->,shift={(#1,#2)},rotate=#4] (-#5-#8,.3*#7) -- (-#5,.3*#7);
  \draw[->,shift={(#1,#2)},rotate=#4] (-#5-#8,-.9*#7) -- (-#5,-.9*#7);
  \draw[->,shift={(#1,#2)},rotate=#4] (-#5-#8,-.6*#7) -- (-#5,-.6*#7);
  \draw[->,shift={(#1,#2)},rotate=#4] (-#5-#8,-.3*#7) -- (-#5,-.3*#7);
  \draw[->,shift={(#1,#2)},rotate=#4] (-#5-#8,0) -- (-#5,0);
  \draw[shift={(#1,#2)},rotate=#4] (-#5,-#7) -- (#6,0) -- (-#5,#7) -- cycle;
  \node[shift={(#1,#2)}] at (0,0) {#3};
   \node[shift={(#1,#2)},rotate=#4,shift={(-#5-#8-0.8,-.6*#7)},rotate=-#4] at (0,0) {$\mathbf{y}_I(t)$};
   \node[shift={(#1,#2)},rotate=#4,shift={(-#5-#8-0.8,.6*#7)},rotate=-#4] at (0,0) {$\mathbf{y}_H(t-1)$};
}
\newcommand{\drawtri}[3]
{
  \draw[->,rotate=\picrotation,shift={(#1,#2)}] (-\picleftside-\picarrowlength,.9*\picheight) -- (-\picleftside,.9*\picheight);
  \draw[->,rotate=\picrotation,shift={(#1,#2)}] (-\picleftside-\picarrowlength,.6*\picheight) -- (-\picleftside,.6*\picheight);
  \draw[->,rotate=\picrotation,shift={(#1,#2)}] (-\picleftside-\picarrowlength,.3*\picheight) -- (-\picleftside,.3*\picheight);
  \draw[->,rotate=\picrotation,shift={(#1,#2)}] (-\picleftside-\picarrowlength,-.9*\picheight) -- (-\picleftside,-.9*\picheight);
  \draw[->,rotate=\picrotation,shift={(#1,#2)}] (-\picleftside-\picarrowlength,-.6*\picheight) -- (-\picleftside,-.6*\picheight);
  \draw[->,rotate=\picrotation,shift={(#1,#2)}] (-\picleftside-\picarrowlength,-.3*\picheight) -- (-\picleftside,-.3*\picheight);
  \draw[->,rotate=\picrotation,shift={(#1,#2)}] (-\picleftside-\picarrowlength,0) -- (-\picleftside,0);
  \draw[rotate=\picrotation,shift={(#1,#2)}] (-\picleftside,-\picheight) -- (\picrightside,0) -- (-\picleftside,\picheight) -- cycle;
  \node[rotate=\picrotation,shift={(#1,#2)},rotate=-\picrotation] at (0,0) {#3};
  \node[rotate=\picrotation,shift={(#1-\picleftside-\picarrowlength-1.4,.7*\picheight+#2)},rotate=-\picrotation] at (0,0) {$\mathbf{y}_H(t-1)$};
  \node[rotate=\picrotation,shift={(#1-\picleftside-\picarrowlength-1.4,-.7*\picheight+#2)},rotate=-\picrotation] at (0,0) {$\mathbf{y}_I(t)$};
}
\newcommand{\drawuniteasy}[4]{\drawtris{#1}{#2}{#3}{#4}{0.35}{0.8}{1.2}{1}}

\pgfplotstableread[col sep=tab]{
f	omega	H_1_m05	H_1_01	H_1_02	H_1_05	H_1_08	H_1_09	ID	H_2_01	H_2_02	H_2_05	H_m05_01	H_01_01	H_02_01	H_05_01	H_08_01	H_09_01	H_m05_02	H_01_02	H_02_02	H_05_02	H_08_02	H_09_02	H_m05_05	H_01_05	H_02_05	H_05_05	H_08_05	H_09_05	PID_2_09	PID_2_05	PID_2_01	PID_05_09	PID_05_05	PID_05_01	PID_01_09	PID_01_05	PID_01_01	PID_09_09	PID_09_05	PID_09_01
0	0	1	1	1	1	1	1	1	1	1	1	1	1	1	1	1	1	1	1	1	1	1	1	1	1	1	1	1	1	0.1	0.5	0.9	0.1	0.5	0.9	0.1	0.5	0.9	0.1	0.5	0.9
0.01	0.0628318531	1.0004387934	0.9997564752	0.9993839224	0.996076667	0.9627270885	0.8590136893	1	0.9998223898	0.9996842267	0.9995065604	1.0002611053	0.9995789082	0.9992064216	0.9958997537	0.9625560984	0.8588611197	1.0001228816	0.9994407788	0.9990683437	0.9957621326	0.962423085	0.8587424357	0.9999451373	0.9992631557	0.9988907868	0.9955851633	0.9622520408	0.8585898179	0.1164125802	0.5019693931	0.9002192257	0.1159558549	0.5	0.8966873659	0.1163842309	0.5018471511	0.9	0.1	0.4311985803	0.7733006382
0.02	0.1256637061	1.0017569079	0.9990279256	0.997544915	0.9845929276	0.8719048811	0.6429101936	1	0.9992900711	0.9987375553	0.9980267284	1.0010457317	0.9983186868	0.9968367291	0.9838939366	0.8712858907	0.6424537731	1.0004922452	0.9977667081	0.9962855697	0.9833499335	0.8708041495	0.6420985551	0.9997801694	0.9970565722	0.995576488	0.9826500583	0.870184376	0.6416415572	0.1555427196	0.5078240824	0.9008757183	0.1531462617	0.5	0.8869958608	0.1553915205	0.5073304396	0.9	0.1	0.3264852791	0.5791821824
0.03	0.1884955592	1.0039595601	0.9978203877	0.9945103042	0.9663522837	0.7650525044	0.4886314038	1	0.9984045799	0.9971619328	0.9955619646	1.0023578229	0.9962284449	0.9929236424	0.9648105458	0.7638319243	0.4878518314	1.0011102554	0.9949885064	0.9916878171	0.9636097109	0.762881234	0.487244635	0.9995039521	0.9933920255	0.9900966322	0.962063578	0.7616571743	0.4864628403	0.2046532401	0.5174096532	0.9019659361	0.1977671259	0.5	0.8716168421	0.2042071754	0.5162819008	0.9	0.1	0.2528226052	0.4407288815
0.04	0.2513274123	1.0070554934	0.9961437986	0.9903244947	0.9425295265	0.6656801858	0.3876367634	1	0.9971684757	0.994960608	0.9921147013	1.0042039914	0.9933231933	0.9875203668	0.9398607312	0.6637952962	0.3865391605	1.0019805461	0.9911238395	0.9853338614	0.9377797508	0.6623255624	0.3856833098	0.9991145601	0.9882889072	0.9825154902	0.9350973996	0.6604310987	0.3845801317	0.2579734676	0.5304873598	0.9034840163	0.2431476102	0.5	0.8515603621	0.2569786699	0.5284416937	0.9	0.1	0.2056364031	0.3502236198
0.05	0.3141592654	1.0110570505	0.9940118148	0.9850473476	0.9144826257	0.5814602765	0.3192782426	1	0.9955853419	0.9921381382	0.9876883406	1.0065935793	0.9896235925	0.9806987003	0.9104454976	0.5788933282	0.3178687383	1.0031082597	0.9861970313	0.9773030415	0.9072930897	0.5768889162	0.3167681211	0.9986092605	0.9817738799	0.9729197802	0.9032238271	0.5743015357	0.3153473976	0.3132064346	0.5467572439	0.9054218336	0.2864218427	0.5	0.8279925358	0.3113308965	0.5434831602	0.9	0.1	0.1745676919	0.2890814918
0.06	0.3769911184	1.0159802775	0.9914415678	0.9787520722	0.8835749499	0.5123867635	0.2707628725	1	0.9936597846	0.9887003972	0.9822872507	1.0095387436	0.9851556147	0.9725465732	0.8779728944	0.509138121	0.2690461776	1.0045001039	0.9802386719	0.9676925626	0.873590904	0.5065969966	0.2677033596	0.9979844736	0.9738804119	0.9614156821	0.8679244084	0.5033109852	0.2659669176	0.3693268544	0.5658829509	0.9077690801	0.3263279569	0.5	0.8020820195	0.3661659956	0.56103988	0.9	0.1	0.1532200933	0.2457901637
0.07	0.4398229715	1.0218450594	0.9884533645	0.9715227722	0.8510432447	0.4561069056	0.2348483739	1	0.9913974326	0.9846545876	0.9759167619	1.0130545684	0.9799501279	0.9631651821	0.8437220878	0.4521832152	0.2328280749	1.0061644255	0.97328514	0.9566143546	0.8379836351	0.449107757	0.2312445287	0.9972357215	0.9646482069	0.948125358	0.8305473676	0.4451223744	0.2291924646	0.4258066528	0.5875142105	0.9105133659	0.3623798754	0.5	0.7748862493	0.4208900186	0.5807303981	0.9	0.1	0.137976757	0.2138325834
0.08	0.5026548246	1.028675289	0.9850703437	0.9634517991	0.8179228424	0.4100381592	0.2073243731	1	0.9888049365	0.9800092538	0.9685831611	1.0171592038	0.9740424186	0.952665895	0.8087661442	0.405447756	0.2050033635	1.0081113024	0.9653780525	0.9441916788	0.8015719545	0.4018411905	0.2031798041	0.9963575632	0.9541225474	0.9331831892	0.7922262922	0.3971560565	0.2008108966	0.4823359575	0.6113046049	0.9136403362	0.3945135974	0.5	0.7472873007	0.4751348474	0.6021780372	0.9	0.1	0.126738344	0.1894199099
0.09	0.5654866776	1.0364990712	0.9813180976	0.9546370612	0.7850231207	0.3719841958	0.1856282252	1	0.9858899668	0.9747743001	0.9602936857	1.0218740349	0.9674716666	0.9411671005	0.7739464184	0.3667354864	0.1830090048	1.0103526567	0.9565636617	0.9305556732	0.765220363	0.3626006341	0.1809456233	0.9953435133	0.9423535727	0.916731942	0.7538527459	0.3572140744	0.1782576125	0.5387111787	0.6369239158	0.9171338043	0.4229007307	0.5	0.7199712411	0.5286470291	0.6250249653	0.9	0.1	0.1182310561	0.1702459203
0.1	0.6283185307	1.0453489648	0.9772242725	0.9451794243	0.7529377602	0.3402202929	0.1681279282	1	0.9826612127	0.9689610096	0.9510565163	1.0272238814	0.9602803887	0.9287911593	0.7398827325	0.3343212856	0.1652127938	1.0129023883	0.9468922177	0.9158420093	0.7295673323	0.3296601985	0.162909407	0.9941859447	0.9293955122	0.8989190506	0.7160863632	0.3235687265	0.1598991617	0.5947851798	0.6640655131	0.9209758961	0.4478362211	0.5	0.6934375283	0.5812385146	0.6489409379	0.9	0.1	0.1116479589	0.1548417693
0.11	0.6911503838	1.0552622648	0.9728181579	0.9351803142	0.7220739806	0.313427725	0.1537413611	1	0.9791283796	0.962582068	0.940880769	1.0332372314	0.9525138665	0.9156615857	0.7070031266	0.3068859805	0.1505325297	1.0157765331	0.9364173142	0.9001878008	0.6950554655	0.3016999077	0.1479886773	0.9928759711	0.9153058964	0.8798931731	0.6793855221	0.2948981189	0.14465229	0.6504430513	0.6924498229	0.9251472053	0.4696680032	0.5	0.6680247252	0.632762811	0.6736277611	0.9	0.1	0.1064581782	0.1422333905
0.12	0.7539822369	1.0662813284	0.968130276	0.9247396016	0.6926886432	0.290603697	0.1417254752	1	0.9753021854	0.9556515896	0.9297764859	1.0399465099	0.9442195739	0.9019005544	0.6755807475	0.2834264207	0.1382251657	1.0189934465	0.9251952372	0.8837288703	0.661969003	0.2777158849	0.1354401757	0.9914033065	0.9001447659	0.8598011372	0.6440456125	0.2701964841	0.1317730143	0.7055894491	0.7218250291	0.9296269545	0.4887537981	0.5	0.6439420338	0.6831025081	0.6988206646	0.9	0.1	0.1023009953	0.1317518219
0.13	0.8168140899	1.0784539475	0.9631919815	0.913953822	0.6649236631	0.2709826523	0.1315544416	1	0.971194357	0.9481851475	0.9177546257	1.0473883881	0.9354466172	0.8876267945	0.6457701094	0.2631768227	0.1277649313	1.0225740153	0.913284331	0.8665974395	0.6304707415	0.2569417261	0.1247379676	0.9897560989	0.8839738965	0.8387853478	0.6102367675	0.2486955826	0.1207346973	0.7601415719	0.7519660192	0.9343931607	0.5054361184	0.5	0.6213001232	0.7321622669	0.7242876401	0.9	0.1	0.0989244697	0.1229235705
0.14	0.879645943	1.0918337723	0.9580350802	0.9029147514	0.6388369269	0.2539753459	0.122846339	1	0.966817624	0.9401998068	0.9048270525	1.0556041336	0.9262452	0.8729538947	0.6176387998	0.2455478405	0.1187700056	1.0265419018	0.9007443973	0.8489202748	0.6006343553	0.2387875711	0.1155001042	0.987920734	0.8668560577	0.8169816931	0.5780369336	0.2298037636	0.1111546908	0.8140250724	0.7826723518	0.9394228026	0.5200292757	0.5	0.6001379763	0.7798645755	0.7498275693	0.9	0.1	0.0961484331	0.1154046521
0.15	0.9424777961	1.106480789	0.9526914727	0.8917083379	0.6144274079	0.2391236895	0.1153174326	1	0.9621857125	0.9317141626	0.8910065242	1.0646400064	0.9166661235	0.8579890224	0.5911932733	0.2300813976	0.110956786	1.0309238217	0.8876361377	0.8308172873	0.5724707178	0.2227949281	0.1074428851	0.9858816019	0.8488543178	0.7945179467	0.5474588291	0.2130607675	0.1027485848	0.8671715781	0.8137657818	0.9446919866	0.532813985	0.5	0.5804446486	0.8261469679	0.7752677212	0.9	0.1	0.0938413807	0.1089394545
0.16	1.0053096491	1.1224618561	0.9471928309	0.8804139662	0.5916545768	0.2260677825	0.1087527309	1	0.9573133359	0.9227483809	0.87630668	1.0745477039	0.9067603287	0.8428320309	0.5663988166	0.216417703	0.1041104396	1.0357498603	0.8740206511	0.8124005618	0.5459483028	0.2086036803	0.1003514064	0.9836208226	0.830031405	0.7715126398	0.5184708579	0.1981047079	0.0953007446	0.9195171391	0.8450876907	0.9501761105	0.5440365238	0.5	0.5621760445	0.870960042	0.8004610022	0.9	0.1	0.0919055942	0.1033342469
0.17	1.0681415022	1.1398513035	0.9415703081	0.8691040234	0.5704528554	0.2145218855	0.1029864814	1	0.9522161841	0.9133242446	0.860742027	1.0853848587	0.8965784859	0.8275749168	0.5431944412	0.2042712112	0.0980653943	1.0410538307	0.8599589903	0.7937737756	0.5210084233	0.195928039	0.0940600503	0.9811179214	0.8104491355	0.7480743588	0.4910127471	0.1846480026	0.0886447928	0.9710012289	0.8764966206	0.9558500223	0.5539104237	0.5	0.5452673746	0.9142659263	0.8252831931	0.9	0.1	0.0902673029	0.0984396305
0.18	1.1309733553	1.1587315961	0.9358542878	0.8578437222	0.5507420583	0.2042568347	0.0978889188	1	0.9469109105	0.9034652031	0.8443279255	1.0972155907	0.8861706358	0.8123015801	0.5215036639	0.1934130253	0.0926920852	1.0468736769	0.8455117842	0.7750319527	0.4975762856	0.1845389426	0.0884392319	0.9783494447	0.7901679094	0.7243014104	0.4650068996	0.1724597495	0.0826503478	1.0215660895	0.907866019	0.9616881728	0.5626194109	0.5	0.5296421238	0.9560370051	0.8496303066	0.9	0.1	0.088870023	0.0941386155
0.19	1.1938052084	1.1791940664	0.9300741707	0.8466911377	0.532434733	0.1950870716	0.0933570578	1	0.9414151154	0.8931964268	0.8270805743	1.1101111181	0.8755858827	0.7970878351	0.5012421056	0.1836579181	0.0878877453	1.0532519266	0.8307389259	0.7562614988	0.4755688011	0.1742510753	0.0833861904	0.9752885056	0.7692462792	0.7002817924	0.4403664248	0.1613527272	0.077213809	1.0711562935	0.9390822367	0.9676647609	0.5703208152	0.5	0.5152183286	0.9962548014	0.8734161325	0.9	0.1	0.0876699546	0.090338335
0.2	1.2566370614	1.2013397144	0.9242582	0.8356974125	0.5154411821	0.1868609908	0.0893081648	1	0.9357473265	0.8825448647	0.8090169944	1.1241504259	0.8648721396	0.7820016195	0.4823227081	0.1748546726	0.0835698765	1.0602361957	0.8156993281	0.7375404599	0.4548999683	0.1649132078	0.0788184622	0.9719042449	0.747740591	0.6760934089	0.4170006759	0.1511737171	0.0722518231	1.1197184513	0.9700427855	0.9737538709	0.5771490021	0.5	0.5019128463	1.0349089603	0.8965699988	0.9	0.1	0.086632741	0.0869641712
0.21	1.3194689145	1.2252800733	0.9184333234	0.8249070883	0.4996727933	0.1794536857	0.0856750456	1	0.9299269755	0.8715393071	0.7901550124	1.1394209927	0.8540759226	0.7671033537	0.4646592094	0.1668788232	0.0796715361	1.0678797462	0.8004507423	0.7189389522	0.43548448	0.1564009409	0.0746691699	0.9681611915	0.725704694	0.6518044706	0.3948189621	0.1417962292	0.0676965667	1.1672010123	1.000654842	0.9799296008	0.5832185901	0.5	0.4896441608	1.0719963048	0.9190347521	0.9	0.1	0.0857311493	0.0839555133
0.22	1.3823007676	1.2511381347	0.9126250901	0.8143585263	0.4850441602	0.1727614442	0.082402587	1	0.9239743701	0.8602104514	0.7705132428	1.15601957	0.8432421927	0.7524464064	0.4481683724	0.1596271466	0.0761378784	1.0762420996	0.7850496407	0.7005197155	0.417240056	0.1486111999	0.0708835665	0.9640185014	0.7031897176	0.6274740289	0.3737329487	0.1331149806	0.0634922845	1.213554133	1.0308339757	0.9861661813	0.5886273453	0.5	0.4783341472	1.1075199499	0.9407649499	0.9	0.1	0.0849433863	0.0812626445
0.23	1.4451326207	1.2790493176	0.9068575796	0.8040843849	0.4714743531	0.1666975302	0.0794451809	1	0.9179106613	0.8485909702	0.7501110696	1.1740530049	0.8324142406	0.7380776295	0.4327713352	0.1530134402	0.0729235786	1.0853897014	0.7695511533	0.6823387483	0.4000888787	0.1414580189	0.0674164632	0.9594290518	0.680243909	0.603152598	0.3536581313	0.1250416627	0.0595927097	1.2587295895	1.0605030723	0.9924380854	0.593458719	0.5	0.4679091043	1.1414884689	0.9617252493	0.9	0.1	0.0842518585	0.0788444233
0.24	1.5079644737	1.309162459	0.90115336	0.7941121276	0.4588876041	0.1611889211	0.0767647807	1	0.9117578042	0.8367155827	0.7289686274	1.193639089	0.8216336088	0.7240379298	0.4183943543	0.1469652568	0.0699908879	1.0953966297	0.7540090587	0.6644459915	0.3839584091	0.1348692821	0.0642302882	0.9543383608	0.656912528	0.5788828277	0.3345146669	0.1175016666	0.0559591168	1.3026807225	1.0895914282	0.998720129	0.5977840357	0.5	0.4583002872	1.1739151102	0.9818889766	0.9	0.1	0.0836422471	0.0766665317
0.25	1.5707963268	1.3416407865	0.8955334712	0.7844645406	0.4472135955	0.1561737619	0.0743294146	1	0.9055385138	0.8246211251	0.7071067812	1.2149074039	0.8109400486	0.7103628542	0.4049691346	0.1414213562	0.0673081477	1.1063453349	0.7384758186	0.6468860321	0.3687817783	0.1287841832	0.0612936055	0.9486832981	0.6332377903	0.5547001962	0.316227766	0.1104315261	0.0525588331	1.3453624047	1.1180339887	1.0049875621	0.6016643583	0.5	0.4494441011	1.2048170643	1.0012368589	0.9	0.1	0.0831028119	0.0747001372
0.26	1.6336281799	1.3766628124	0.8900174302	0.7751602432	0.4363874863	0.1515993606	0.0721120377	1	0.8992762126	0.8123466217	0.6845471059	1.23800012	0.8003715037	0.6970831676	0.3924328859	0.1363296988	0.0648486402	1.1183273849	0.7230026526	0.6296988048	0.3544979003	0.1231512284	0.0585799702	0.9423905441	0.609258856	0.5306337011	0.2987277908	0.1037769036	0.0493640867	1.3867310248	1.1457707098	1.0112161509	0.6051520662	0.5	0.4412820743	1.234214783	1.0197559027	0.9	0.1	0.0826238607	0.0729208572
0.27	1.6964600329	1.4144230593	0.8846232556	0.7662141804	0.4263497725	0.1474205988	0.070089635	1	0.8929949708	0.7999333505	0.6613118653	1.2630726786	0.7899641184	0.6842254097	0.3807282027	0.1316458533	0.0625896915	1.1314441769	0.7076396448	0.6129202765	0.341051402	0.1179266535	0.0560670365	0.9353747517	0.5850118553	0.5067065289	0.2819501634	0.0974909912	0.0463511072	1.4267444832	1.1727460226	1.017382252	0.6082921859	0.5	0.4337606917	1.2621313497	1.0374384046	0.9	0.1	0.0821973406	0.0713079507
0.28	1.759291886	1.4551324805	0.8793675073	0.7576380875	0.4170460449	0.14359866	0.0682425139	1	0.8867194389	0.7874249039	0.6374239897	1.2902942567	0.7797522627	0.6718124199	0.369802835	0.1273317232	0.0605119636	1.1458075537	0.692435875	0.5965830983	0.3283924419	0.113073161	0.0537358549	0.9275363513	0.560529945	0.4829366925	0.2658351539	0.0915332308	0.0434994155	1.4653621963	1.1989083846	1.0234628781	0.6111235083	0.5	0.4268311454	1.2885919019	1.0542810777	0.9	0.1	0.0818165221	0.0698436797
0.29	1.8221237391	1.4990183798	0.8742653381	0.7494409238	0.4084266911	0.140100006	0.0665537424	1	0.8804747698	0.7748672377	0.6129070537	1.3198478629	0.7697685723	0.6598638249	0.3596093968	0.1233545206	0.058598891	1.1615402313	0.6774395676	0.5807172185	0.316476462	0.1085589047	0.0515703146	0.9187589385	0.5358433925	0.4593376285	0.2503275999	0.0858682819	0.0407912582	1.5025451064	1.2242099032	1.0294357568	0.6136795261	0.5	0.4204490399	1.3136231055	1.070284285	0.9	0.1	0.0814757506	0.0685128022
0.3	1.8849555922	1.5463235462	0.8693305552	0.7416292705	0.4004465715	0.1368955476	0.065008698	1	0.8742865326	0.7623087051	0.5877852523	1.3519298514	0.7600439968	0.6483964833	0.3501050444	0.1196859336	0.0568362291	1.1787759001	0.6626982498	0.5653504488	0.3052639073	0.1043566676	0.0495566964	0.9089061757	0.5109796797	0.4359187479	0.235376589	0.080465184	0.0382111539	1.5382556972	1.2486060205	1.035279382	0.61598922	0.5	0.414574079	1.3372526793	1.085451365	0.9	0.1	0.0811702517	0.0673021646
0.31	1.9477874452	1.5973041991	0.8645756889	0.7342076928	0.3930646888	0.1339599668	0.0635947029	1	0.8681806151	0.7498000688	0.5620833779	1.3867485421	0.7506078534	0.6374248864	0.3412511433	0.1163014464	0.0552116883	1.1976587983	0.648258911	0.5505089786	0.2947199307	0.1004431923	0.0476833126	0.8978181397	0.4859636236	0.41268594	0.220935128	0.0752966706	0.0357455254	1.5724580105	1.2720552475	1.0409730595	0.6180777186	0.5	0.4091697517	1.3595089677	1.0997880419	0.9	0.1	0.0808959755	0.0662003724
0.32	2.0106192983	1.652226175	0.8600120653	0.7271790664	0.3862438654	0.1312711622	0.0623007285	1	0.8621831172	0.7373944851	0.535826795	1.4245215138	0.7414878833	0.6269615142	0.3330129398	0.1131797799	0.0537146363	1.2183424697	0.6341681541	0.5362178333	0.2848140962	0.0967986311	0.0459402136	0.885307056	0.4608175086	0.3896420285	0.2069598125	0.0703386061	0.0333823997	1.605117667	1.2945189421	1.0464969462	0.6199668521	0.5	0.4042030256	1.3804205598	1.113301909	0.9	0.1	0.0806494732	0.0651975221
0.33	2.0734511514	1.7113585663	0.8556498826	0.7205448693	0.3799504345	0.1288097908	0.0611171524	1	0.8563202314	0.72514745	0.5090414158	1.4654709635	0.7327103055	0.6170171492	0.325359244	0.1103024298	0.0523358541	1.2409873005	0.6204723305	0.5225012746	0.2755200887	0.0934060913	0.0443189472	0.8711523875	0.4355612276	0.3667871804	0.1934105071	0.0655695183	0.0311111618	1.6362018865	1.3159611218	1.0518320849	0.6216756177	0.5	0.3996440577	1.4000159522	1.1260019794	0.9	0.1	0.0804277964	0.0642849818
0.34	2.1362830044	1.7749637303	0.851498288	0.7143054405	0.374153952	0.1265588874	0.0600355587	1	0.8506181146	0.7131166984	0.4817536741	1.5098163018	0.7242998683	0.6076011471	0.3182621293	0.1076532822	0.0510673337	1.2657562752	0.6072176479	0.5093831374	0.266815431	0.090251256	0.0428123594	0.8550952985	0.4102124287	0.3441192704	0.1802500411	0.060970209	0.028922351	1.6656795103	1.336348306	1.056960434	0.6232205716	0.5	0.3954659235	1.4183232513	1.1378982947	0.9	0.1	0.0802284172	0.0634552102
0.35	2.1991148575	1.8432822238	0.8475654541	0.7084602088	0.3688269297	0.1245035494	0.0590485732	1	0.8451027479	0.7013620458	0.4539904997	1.5577628726	0.7162798943	0.5987216692	0.3116966518	0.1052182917	0.0499021114	1.2928081914	0.5944502408	0.4968871014	0.2586812099	0.0873220641	0.0414144281	0.8368326179	0.3847866641	0.3216342042	0.1674439221	0.0565234286	0.0268074912	1.6935210226	1.355649384	1.0618648928	0.6246161591	0.5	0.3916443681	1.4353699146	1.1490015857	0.9	0.1	0.0800491618	0.0627016068
0.36	2.2619467106	1.9165107808	0.8438586546	0.7030078929	0.3639445908	0.1226306713	0.0581497255	1	0.8397997868	0.6899451596	0.4257792916	1.609485345	0.7086723182	0.5903858786	0.3056405897	0.1029852116	0.0488341271	1.3222873365	0.5822161941	0.4850368929	0.2511018088	0.0846084381	0.0401201217	0.8160106025	0.3592975401	0.2993262026	0.15496007	0.0522136003	0.0247589489	1.7196985729	1.3738355032	1.0665293235	0.6258749934	0.5	0.3881575782	1.451182524	1.1593229793	0.9	0.1	0.0798881574	0.0620183874
0.37	2.3247785637	1.9947709892	0.8403843367	0.6979466771	0.3594846492	0.1209287234	0.057333335	1	0.8347344014	0.6789292497	0.3971478906	1.6651039676	0.7014977163	0.5826001018	0.3000742035	0.1009433656	0.0478581071	1.354308371	0.5705615072	0.4738564138	0.2440646431	0.0821020475	0.0389252781	0.7922190907	0.3337568666	0.2771880506	0.1427685701	0.0480265874	0.022769813	1.7441859966	1.3908799754	1.0709385702	0.6270080911	0.5	0.3849859762	1.4657865918	1.1688737456	0.9	0.1	0.0797437875	0.0614004798
0.38	2.3876104167	2.0780659296	0.8371481912	0.6932743625	0.3554271095	0.1193875648	0.0565944141	1	0.8299311098	0.6683786645	0.3681245527	1.7246515632	0.6947753274	0.5753699611	0.2949800154	0.0990834542	0.0469694649	1.3889349309	0.5595319901	0.4633697926	0.2375598968	0.0797961011	0.0378264989	0.7649870908	0.3081748034	0.2552113146	0.1308414457	0.0439494939	0.0208337934	1.7669588364	1.4067581979	1.0750784741	0.6280250717	0.5	0.3821120345	1.4792063938	1.1776650809	0.9	0.1	0.0796146559	0.0608434363
0.39	2.4504422698	2.1662219419	0.8341552179	0.6889884989	0.3517540883	0.1179982858	0.055928588	1	0.8254136032	0.6583583844	0.3387379202	1.7880290583	0.688523064	0.5687004794	0.2903426094	0.0973973902	0.0461642173	1.4261503779	0.5491730816	0.453601355	0.2315802533	0.0776851608	0.0368210548	0.7337815154	0.2825600037	0.2333865312	0.1191524483	0.0399704939	0.0189451336	1.7879943616	1.4214475871	1.0789358871	0.6289343265	0.5	0.3795201093	1.4914648262	1.1857079217	0.9	0.1	0.0794995565	0.0603433607
0.4	2.5132741229	2.2588132805	0.8314097874	0.6850864978	0.3484496549	0.1167530747	0.0553320269	1	0.8212045671	0.6489334032	0.3090169944	1.8549477821	0.6827575146	0.5625961609	0.286148448	0.0958781582	0.0454389132	1.4658193893	0.5395295828	0.4445755125	0.2261206204	0.0757649701	0.0359068005	0.6980116908	0.2569197536	0.2117033704	0.1076768651	0.0360786842	0.0170985366	1.8072715872	1.4349275223	1.0824986831	0.6297431609	0.5	0.3771962926	1.5025832861	1.1930127863	0.9	0.1	0.0793974482	0.0598968462
0.41	2.5761059759	2.355069362	0.8289156987	0.6815657292	0.3454996904	0.1156451057	0.0548013882	1	0.8173255003	0.6401679966	0.278991106	1.9248582445	0.6774939381	0.5570610506	0.2823857073	0.0945196939	0.0447905721	1.5076400353	0.5306453021	0.4363165674	0.2211778446	0.0740322956	0.0350820949	0.6570434061	0.2312601076	0.1901507766	0.0963913408	0.0322639559	0.0152890999	1.8247712914	1.4471792997	1.0857557668	0.6304579162	0.5	0.3751282813	1.5125815699	1.1995896403	0.9	0.1	0.0793074347	0.0595009233
0.42	2.638937829	2.4537685272	0.8266762307	0.6784236043	0.342891763	0.1146684435	0.0543337692	1	0.8137965333	0.6321248788	0.2486898872	1.9968683209	0.6727462507	0.5520987773	0.279044128	0.0933167818	0.044216633	1.5510881328	0.5225626121	0.4288484386	0.2167504141	0.0724847759	0.0343457273	0.6102274182	0.2055860185	0.1687170896	0.0852737138	0.0285168823	0.0135122589	1.8404760319	1.4581860924	1.0886970818	0.6310840713	0.5	0.3733052618	1.5214777888	1.2054477825	0.9	0.1	0.0792287466	0.059153016
0.43	2.7017696821	2.5531297341	0.8246941906	0.675657645	0.3406150198	0.1138179631	0.0539266658	1	0.8106362505	0.6248642598	0.2181432414	2.0696595146	0.6685270064	0.54771258	0.2761148825	0.0922649669	0.0437149102	1.5953595215	0.515321925	0.4221943143	0.2128381522	0.0711207773	0.0336968461	0.5569479959	0.1799014639	0.1473901488	0.0743028645	0.0248286194	0.0117637377	1.8543701611	1.4679329182	1.0913136169	0.6316263291	0.5	0.3717178092	1.5292882991	1.2105957499	0.9	0.1	0.0791607279	0.0588509047
0.44	2.7646015352	2.6507248498	0.8229719557	0.6732655433	0.3386600925	0.1130892832	0.0535779391	1	0.8078615182	0.6184428223	0.1873813146	2.1414186014	0.6648473735	0.5439053239	0.2735904565	0.09136048	0.0432835552	1.6393217574	0.508961099	0.4163762428	0.2094419034	0.0699392555	0.0331348919	0.496696307	0.1542095669	0.1261573826	0.0634585733	0.0211908185	0.0100395047	1.8664398395	1.4764066127	1.0935974109	0.6320886887	0.5	0.3703578003	1.5360276449	1.2150412374	0.9	0.1	0.0791028235	0.0585926954
0.45	2.8274333882	2.7434478772	0.8215115111	0.6712452093	0.3370190168	0.1124787092	0.0532857868	1	0.8054873227	0.6129126486	0.156434465	2.2098124856	0.6617171076	0.5406795066	0.2714645456	0.0906001743	0.0429210258	1.6814939047	0.5035147961	0.4114146791	0.2065632182	0.0689396235	0.0326595327	0.429169801	0.1285127138	0.1050058852	0.0527213896	0.0175955467	0.0083357336	1.8766730481	1.4835958062	1.0955415571	0.6324745056	0.5	0.3692183385	1.5417085115	1.2187910326	0.9	0.1	0.0790545699	0.0583767939
0.46	2.8902652413	2.8275885794	0.820314482	0.6695948125	0.3356851645	0.1119831877	0.0530487206	1	0.8035266212	0.6083201365	0.1253332336	2.2720426974	0.6591445241	0.5380372572	0.269731966	0.0899814725	0.0426260592	1.7200790705	0.4990138176	0.4073280077	0.2042040451	0.068121628	0.032270605	0.3543908198	0.1028126666	0.083922483	0.0420725071	0.014035215	0.0066487677	1.8850595985	1.4894909067	1.0971402063	0.6327865414	0.5	0.3682936906	1.5463416881	1.2218509616	0.9	0.1	0.079015587	0.0582018843
0.47	2.9530970944	2.899056851	0.8193821611	0.6683128131	0.3346531867	0.111600269	0.0528655468	1	0.8019902087	0.6047049527	0.0941083133	2.325015209	0.6571364704	0.5359803324	0.2683885791	0.089502323	0.0423976509	1.7530740358	0.4954844509	0.404132068	0.2023664394	0.0674852354	0.0319680579	0.2728253505	0.0771106731	0.0628937916	0.0314936469	0.0105025131	0.0049750874	1.8915911427	1.4940840842	1.0983885697	0.6330270039	0.5	0.3675792351	1.5499360385	1.2242258458	0.9	0.1	0.078985572	0.0580669123
0.48	3.0159289474	2.9537787827	0.8187155307	0.6673979885	0.3339189693	0.1113280769	0.0527353516	1	0.8008866048	0.6020990746	0.0627905195	2.3656418606	0.6556983017	0.5345101091	0.2674312296	0.0891611655	0.0422350367	1.7784674715	0.4929478634	0.4018397113	0.2010523024	0.0670305321	0.0317519064	0.1854693043	0.0514075735	0.0419062664	0.0209669456	0.0069903478	0.0033112801	1.8962611799	1.4973692602	1.0992829209	0.6331975787	0.5	0.3670714199	1.5524984782	1.2259194685	0.9	0.1	0.0789642944	0.0579710713
0.49	3.0787608005	2.9882300011	0.8183152802	0.6668494521	0.3334795978	0.1111652871	0.0526574902	1	0.8002219623	0.6005259752	0.0314107591	2.3912472752	0.6548338592	0.5336275771	0.2668576981	0.0889569042	0.0421376802	1.7945097356	0.4914195816	0.4004604176	0.2002631607	0.0667576424	0.0316221907	0.0938625726	0.0257039041	0.0209462475	0.0104748473	0.0034917861	0.0016540117	1.8990650624	1.4993420985	1.099820597	0.6332994532	0.5	0.3667677304	1.5540339586	1.2269345494	0.9	0.1	0.0789515919	0.0579137924
0.5	3.1415926536	3	0.8181818182	0.6666666667	0.3333333333	0.1111111111	0.0526315789	1	0.8	0.6	6.04901474817726E-016	2.4	0.6545454545	0.5333333333	0.2666666667	0.0888888889	0.0421052632	1.8	0.4909090909	0.4	0.2	0.0666666667	0.0315789474	1.81470442445318E-015	4.94919388487231E-016	4.03267649878484E-016	2.01633824939242E-016	6.72112749797473E-017	3.18369197272487E-017	1.9	1.5	1.1	0.6333333333	0.5	0.3666666667	1.5545454545	1.2272727273	0.9	0.1	0.0789473684	0.0578947368
0.51	3.2044245067	2.9882300011	0.8183152802	0.6668494521	0.3334795978	0.1111652871	0.0526574902	1	0.8002219623	0.6005259752	0.0314107591	2.3912472752	0.6548338592	0.5336275771	0.2668576981	0.0889569042	0.0421376802	1.7945097356	0.4914195816	0.4004604176	0.2002631607	0.0667576424	0.0316221907	0.0938625726	0.0257039041	0.0209462475	0.0104748473	0.0034917861	0.0016540117	1.8990650624	1.4993420985	1.099820597	0.6332994532	0.5	0.3667677304	1.5540339586	1.2269345494	0.9	0.1	0.0789515919	0.0579137924
0.52	3.2672563597	2.9537787827	0.8187155307	0.6673979885	0.3339189693	0.1113280769	0.0527353516	1	0.8008866048	0.6020990746	0.0627905195	2.3656418606	0.6556983017	0.5345101091	0.2674312296	0.0891611655	0.0422350367	1.7784674715	0.4929478634	0.4018397113	0.2010523024	0.0670305321	0.0317519064	0.1854693043	0.0514075735	0.0419062664	0.0209669456	0.0069903478	0.0033112801	1.8962611799	1.4973692602	1.0992829209	0.6331975787	0.5	0.3670714199	1.5524984782	1.2259194685	0.9	0.1	0.0789642944	0.0579710713
0.53	3.3300882128	2.899056851	0.8193821611	0.6683128131	0.3346531867	0.111600269	0.0528655468	1	0.8019902087	0.6047049527	0.0941083133	2.325015209	0.6571364704	0.5359803324	0.2683885791	0.089502323	0.0423976509	1.7530740358	0.4954844509	0.404132068	0.2023664394	0.0674852354	0.0319680579	0.2728253505	0.0771106731	0.0628937916	0.0314936469	0.0105025131	0.0049750874	1.8915911427	1.4940840842	1.0983885697	0.6330270039	0.5	0.3675792351	1.5499360385	1.2242258458	0.9	0.1	0.078985572	0.0580669123
0.54	3.3929200659	2.8275885794	0.820314482	0.6695948125	0.3356851645	0.1119831877	0.0530487206	1	0.8035266212	0.6083201365	0.1253332336	2.2720426974	0.6591445241	0.5380372572	0.269731966	0.0899814725	0.0426260592	1.7200790705	0.4990138176	0.4073280077	0.2042040451	0.068121628	0.032270605	0.3543908198	0.1028126666	0.083922483	0.0420725071	0.014035215	0.0066487677	1.8850595985	1.4894909067	1.0971402063	0.6327865414	0.5	0.3682936906	1.5463416881	1.2218509616	0.9	0.1	0.079015587	0.0582018843
0.55	3.4557519189	2.7434478772	0.8215115111	0.6712452093	0.3370190168	0.1124787092	0.0532857868	1	0.8054873227	0.6129126486	0.156434465	2.2098124856	0.6617171076	0.5406795066	0.2714645456	0.0906001743	0.0429210258	1.6814939047	0.5035147961	0.4114146791	0.2065632182	0.0689396235	0.0326595327	0.429169801	0.1285127138	0.1050058852	0.0527213896	0.0175955467	0.0083357336	1.8766730481	1.4835958062	1.0955415571	0.6324745056	0.5	0.3692183385	1.5417085115	1.2187910326	0.9	0.1	0.0790545699	0.0583767939
0.56	3.518583772	2.6507248498	0.8229719557	0.6732655433	0.3386600925	0.1130892832	0.0535779391	1	0.8078615182	0.6184428223	0.1873813146	2.1414186014	0.6648473735	0.5439053239	0.2735904565	0.09136048	0.0432835552	1.6393217574	0.508961099	0.4163762428	0.2094419034	0.0699392555	0.0331348919	0.496696307	0.1542095669	0.1261573826	0.0634585733	0.0211908185	0.0100395047	1.8664398395	1.4764066127	1.0935974109	0.6320886887	0.5	0.3703578003	1.5360276449	1.2150412374	0.9	0.1	0.0791028235	0.0585926954
0.57	3.5814156251	2.5531297341	0.8246941906	0.675657645	0.3406150198	0.1138179631	0.0539266658	1	0.8106362505	0.6248642598	0.2181432414	2.0696595146	0.6685270064	0.54771258	0.2761148825	0.0922649669	0.0437149102	1.5953595215	0.515321925	0.4221943143	0.2128381522	0.0711207773	0.0336968461	0.5569479959	0.1799014639	0.1473901488	0.0743028645	0.0248286194	0.0117637377	1.8543701611	1.4679329182	1.0913136169	0.6316263291	0.5	0.3717178092	1.5292882991	1.2105957499	0.9	0.1	0.0791607279	0.0588509047
0.58	3.6442474782	2.4537685272	0.8266762307	0.6784236043	0.342891763	0.1146684435	0.0543337692	1	0.8137965333	0.6321248788	0.2486898872	1.9968683209	0.6727462507	0.5520987773	0.279044128	0.0933167818	0.044216633	1.5510881328	0.5225626121	0.4288484386	0.2167504141	0.0724847759	0.0343457273	0.6102274182	0.2055860185	0.1687170896	0.0852737138	0.0285168823	0.0135122589	1.8404760319	1.4581860924	1.0886970818	0.6310840713	0.5	0.3733052618	1.5214777888	1.2054477825	0.9	0.1	0.0792287466	0.059153016
0.59	3.7070793312	2.355069362	0.8289156987	0.6815657292	0.3454996904	0.1156451057	0.0548013882	1	0.8173255003	0.6401679966	0.278991106	1.9248582445	0.6774939381	0.5570610506	0.2823857073	0.0945196939	0.0447905721	1.5076400353	0.5306453021	0.4363165674	0.2211778446	0.0740322956	0.0350820949	0.6570434061	0.2312601076	0.1901507766	0.0963913408	0.0322639559	0.0152890999	1.8247712914	1.4471792997	1.0857557668	0.6304579162	0.5	0.3751282813	1.5125815699	1.1995896403	0.9	0.1	0.0793074347	0.0595009233
0.6	3.7699111843	2.2588132805	0.8314097874	0.6850864978	0.3484496549	0.1167530747	0.0553320269	1	0.8212045671	0.6489334032	0.3090169944	1.8549477821	0.6827575146	0.5625961609	0.286148448	0.0958781582	0.0454389132	1.4658193893	0.5395295828	0.4445755125	0.2261206204	0.0757649701	0.0359068005	0.6980116908	0.2569197536	0.2117033704	0.1076768651	0.0360786842	0.0170985366	1.8072715872	1.4349275223	1.0824986831	0.6297431609	0.5	0.3771962926	1.5025832861	1.1930127863	0.9	0.1	0.0793974482	0.0598968462
0.61	3.8327430374	2.1662219419	0.8341552179	0.6889884989	0.3517540883	0.1179982858	0.055928588	1	0.8254136032	0.6583583844	0.3387379202	1.7880290583	0.688523064	0.5687004794	0.2903426094	0.0973973902	0.0461642173	1.4261503779	0.5491730816	0.453601355	0.2315802533	0.0776851608	0.0368210548	0.7337815154	0.2825600037	0.2333865312	0.1191524483	0.0399704939	0.0189451336	1.7879943616	1.4214475871	1.0789358871	0.6289343265	0.5	0.3795201093	1.4914648262	1.1857079217	0.9	0.1	0.0794995565	0.0603433607
0.62	3.8955748905	2.0780659296	0.8371481912	0.6932743625	0.3554271095	0.1193875648	0.0565944141	1	0.8299311098	0.6683786645	0.3681245527	1.7246515632	0.6947753274	0.5753699611	0.2949800154	0.0990834542	0.0469694649	1.3889349309	0.5595319901	0.4633697926	0.2375598968	0.0797961011	0.0378264989	0.7649870908	0.3081748034	0.2552113146	0.1308414457	0.0439494939	0.0208337934	1.7669588364	1.4067581979	1.0750784741	0.6280250717	0.5	0.3821120345	1.4792063938	1.1776650809	0.9	0.1	0.0796146559	0.0608434363
0.63	3.9584067435	1.9947709892	0.8403843367	0.6979466771	0.3594846492	0.1209287234	0.057333335	1	0.8347344014	0.6789292497	0.3971478906	1.6651039676	0.7014977163	0.5826001018	0.3000742035	0.1009433656	0.0478581071	1.354308371	0.5705615072	0.4738564138	0.2440646431	0.0821020475	0.0389252781	0.7922190907	0.3337568666	0.2771880506	0.1427685701	0.0480265874	0.022769813	1.7441859966	1.3908799754	1.0709385702	0.6270080911	0.5	0.3849859762	1.4657865918	1.1688737456	0.9	0.1	0.0797437875	0.0614004798
0.64	4.0212385966	1.9165107808	0.8438586546	0.7030078929	0.3639445908	0.1226306713	0.0581497255	1	0.8397997868	0.6899451596	0.4257792916	1.609485345	0.7086723182	0.5903858786	0.3056405897	0.1029852116	0.0488341271	1.3222873365	0.5822161941	0.4850368929	0.2511018088	0.0846084381	0.0401201217	0.8160106025	0.3592975401	0.2993262026	0.15496007	0.0522136003	0.0247589489	1.7196985729	1.3738355032	1.0665293235	0.6258749934	0.5	0.3881575782	1.451182524	1.1593229793	0.9	0.1	0.0798881574	0.0620183874
0.65	4.0840704497	1.8432822238	0.8475654541	0.7084602088	0.3688269297	0.1245035494	0.0590485732	1	0.8451027479	0.7013620458	0.4539904997	1.5577628726	0.7162798943	0.5987216692	0.3116966518	0.1052182917	0.0499021114	1.2928081914	0.5944502408	0.4968871014	0.2586812099	0.0873220641	0.0414144281	0.8368326179	0.3847866641	0.3216342042	0.1674439221	0.0565234286	0.0268074912	1.6935210226	1.355649384	1.0618648928	0.6246161591	0.5	0.3916443681	1.4353699146	1.1490015857	0.9	0.1	0.0800491618	0.0627016068
0.66	4.1469023027	1.7749637303	0.851498288	0.7143054405	0.374153952	0.1265588874	0.0600355587	1	0.8506181146	0.7131166984	0.4817536741	1.5098163018	0.7242998683	0.6076011471	0.3182621293	0.1076532822	0.0510673337	1.2657562752	0.6072176479	0.5093831374	0.266815431	0.090251256	0.0428123594	0.8550952985	0.4102124287	0.3441192704	0.1802500411	0.060970209	0.028922351	1.6656795103	1.336348306	1.056960434	0.6232205716	0.5	0.3954659235	1.4183232513	1.1378982947	0.9	0.1	0.0802284172	0.0634552102
0.67	4.2097341558	1.7113585663	0.8556498826	0.7205448693	0.3799504345	0.1288097908	0.0611171524	1	0.8563202314	0.72514745	0.5090414158	1.4654709635	0.7327103055	0.6170171492	0.325359244	0.1103024298	0.0523358541	1.2409873005	0.6204723305	0.5225012746	0.2755200887	0.0934060913	0.0443189472	0.8711523875	0.4355612276	0.3667871804	0.1934105071	0.0655695183	0.0311111618	1.6362018865	1.3159611218	1.0518320849	0.6216756177	0.5	0.3996440577	1.4000159522	1.1260019794	0.9	0.1	0.0804277964	0.0642849818
0.68	4.2725660089	1.652226175	0.8600120653	0.7271790664	0.3862438654	0.1312711622	0.0623007285	1	0.8621831172	0.7373944851	0.535826795	1.4245215138	0.7414878833	0.6269615142	0.3330129398	0.1131797799	0.0537146363	1.2183424697	0.6341681541	0.5362178333	0.2848140962	0.0967986311	0.0459402136	0.885307056	0.4608175086	0.3896420285	0.2069598125	0.0703386061	0.0333823997	1.605117667	1.2945189421	1.0464969462	0.6199668521	0.5	0.4042030256	1.3804205598	1.113301909	0.9	0.1	0.0806494732	0.0651975221
0.69	4.335397862	1.5973041991	0.8645756889	0.7342076928	0.3930646888	0.1339599668	0.0635947029	1	0.8681806151	0.7498000688	0.5620833779	1.3867485421	0.7506078534	0.6374248864	0.3412511433	0.1163014464	0.0552116883	1.1976587983	0.648258911	0.5505089786	0.2947199307	0.1004431923	0.0476833126	0.8978181397	0.4859636236	0.41268594	0.220935128	0.0752966706	0.0357455254	1.5724580105	1.2720552475	1.0409730595	0.6180777186	0.5	0.4091697517	1.3595089677	1.0997880419	0.9	0.1	0.0808959755	0.0662003724
0.7	4.398229715	1.5463235462	0.8693305552	0.7416292705	0.4004465715	0.1368955476	0.065008698	1	0.8742865326	0.7623087051	0.5877852523	1.3519298514	0.7600439968	0.6483964833	0.3501050444	0.1196859336	0.0568362291	1.1787759001	0.6626982498	0.5653504488	0.3052639073	0.1043566676	0.0495566964	0.9089061757	0.5109796797	0.4359187479	0.235376589	0.080465184	0.0382111539	1.5382556972	1.2486060205	1.035279382	0.61598922	0.5	0.414574079	1.3372526793	1.085451365	0.9	0.1	0.0811702517	0.0673021646
0.71	4.4610615681	1.4990183798	0.8742653381	0.7494409238	0.4084266911	0.140100006	0.0665537424	1	0.8804747698	0.7748672377	0.6129070537	1.3198478629	0.7697685723	0.6598638249	0.3596093968	0.1233545206	0.058598891	1.1615402313	0.6774395676	0.5807172185	0.316476462	0.1085589047	0.0515703146	0.9187589385	0.5358433925	0.4593376285	0.2503275999	0.0858682819	0.0407912582	1.5025451064	1.2242099032	1.0294357568	0.6136795261	0.5	0.4204490399	1.3136231055	1.070284285	0.9	0.1	0.0814757506	0.0685128022
0.72	4.5238934212	1.4551324805	0.8793675073	0.7576380875	0.4170460449	0.14359866	0.0682425139	1	0.8867194389	0.7874249039	0.6374239897	1.2902942567	0.7797522627	0.6718124199	0.369802835	0.1273317232	0.0605119636	1.1458075537	0.692435875	0.5965830983	0.3283924419	0.113073161	0.0537358549	0.9275363513	0.560529945	0.4829366925	0.2658351539	0.0915332308	0.0434994155	1.4653621963	1.1989083846	1.0234628781	0.6111235083	0.5	0.4268311454	1.2885919019	1.0542810777	0.9	0.1	0.0818165221	0.0698436797
0.73	4.5867252742	1.4144230593	0.8846232556	0.7662141804	0.4263497725	0.1474205988	0.070089635	1	0.8929949708	0.7999333505	0.6613118653	1.2630726786	0.7899641184	0.6842254097	0.3807282027	0.1316458533	0.0625896915	1.1314441769	0.7076396448	0.6129202765	0.341051402	0.1179266535	0.0560670365	0.9353747517	0.5850118553	0.5067065289	0.2819501634	0.0974909912	0.0463511072	1.4267444832	1.1727460226	1.017382252	0.6082921859	0.5	0.4337606917	1.2621313497	1.0374384046	0.9	0.1	0.0821973406	0.0713079507
0.74	4.6495571273	1.3766628124	0.8900174302	0.7751602432	0.4363874863	0.1515993606	0.0721120377	1	0.8992762126	0.8123466217	0.6845471059	1.23800012	0.8003715037	0.6970831676	0.3924328859	0.1363296988	0.0648486402	1.1183273849	0.7230026526	0.6296988048	0.3544979003	0.1231512284	0.0585799702	0.9423905441	0.609258856	0.5306337011	0.2987277908	0.1037769036	0.0493640867	1.3867310248	1.1457707098	1.0112161509	0.6051520662	0.5	0.4412820743	1.234214783	1.0197559027	0.9	0.1	0.0826238607	0.0729208572
0.75	4.7123889804	1.3416407865	0.8955334712	0.7844645406	0.4472135955	0.1561737619	0.0743294146	1	0.9055385138	0.8246211251	0.7071067812	1.2149074039	0.8109400486	0.7103628542	0.4049691346	0.1414213562	0.0673081477	1.1063453349	0.7384758186	0.6468860321	0.3687817783	0.1287841832	0.0612936055	0.9486832981	0.6332377903	0.5547001962	0.316227766	0.1104315261	0.0525588331	1.3453624047	1.1180339887	1.0049875621	0.6016643583	0.5	0.4494441011	1.2048170643	1.0012368589	0.9	0.1	0.0831028119	0.0747001372
0.76	4.7752208335	1.309162459	0.90115336	0.7941121276	0.4588876041	0.1611889211	0.0767647807	1	0.9117578042	0.8367155827	0.7289686274	1.193639089	0.8216336088	0.7240379298	0.4183943543	0.1469652568	0.0699908879	1.0953966297	0.7540090587	0.6644459915	0.3839584091	0.1348692821	0.0642302882	0.9543383608	0.656912528	0.5788828277	0.3345146669	0.1175016666	0.0559591168	1.3026807225	1.0895914282	0.998720129	0.5977840357	0.5	0.4583002872	1.1739151102	0.9818889766	0.9	0.1	0.0836422471	0.0766665317
0.77	4.8380526865	1.2790493176	0.9068575796	0.8040843849	0.4714743531	0.1666975302	0.0794451809	1	0.9179106613	0.8485909702	0.7501110696	1.1740530049	0.8324142406	0.7380776295	0.4327713352	0.1530134402	0.0729235786	1.0853897014	0.7695511533	0.6823387483	0.4000888787	0.1414580189	0.0674164632	0.9594290518	0.680243909	0.603152598	0.3536581313	0.1250416627	0.0595927097	1.2587295895	1.0605030723	0.9924380854	0.593458719	0.5	0.4679091043	1.1414884689	0.9617252493	0.9	0.1	0.0842518585	0.0788444233
0.78	4.9008845396	1.2511381347	0.9126250901	0.8143585263	0.4850441602	0.1727614442	0.082402587	1	0.9239743701	0.8602104514	0.7705132428	1.15601957	0.8432421927	0.7524464064	0.4481683724	0.1596271466	0.0761378784	1.0762420996	0.7850496407	0.7005197155	0.417240056	0.1486111999	0.0708835665	0.9640185014	0.7031897176	0.6274740289	0.3737329487	0.1331149806	0.0634922845	1.213554133	1.0308339757	0.9861661813	0.5886273453	0.5	0.4783341472	1.1075199499	0.9407649499	0.9	0.1	0.0849433863	0.0812626445
0.79	4.9637163927	1.2252800733	0.9184333234	0.8249070883	0.4996727933	0.1794536857	0.0856750456	1	0.9299269755	0.8715393071	0.7901550124	1.1394209927	0.8540759226	0.7671033537	0.4646592094	0.1668788232	0.0796715361	1.0678797462	0.8004507423	0.7189389522	0.43548448	0.1564009409	0.0746691699	0.9681611915	0.725704694	0.6518044706	0.3948189621	0.1417962292	0.0676965667	1.1672010123	1.000654842	0.9799296008	0.5832185901	0.5	0.4896441608	1.0719963048	0.9190347521	0.9	0.1	0.0857311493	0.0839555133
0.8	5.0265482457	1.2013397144	0.9242582	0.8356974125	0.5154411821	0.1868609908	0.0893081648	1	0.9357473265	0.8825448647	0.8090169944	1.1241504259	0.8648721396	0.7820016195	0.4823227081	0.1748546726	0.0835698765	1.0602361957	0.8156993281	0.7375404599	0.4548999683	0.1649132078	0.0788184622	0.9719042449	0.747740591	0.6760934089	0.4170006759	0.1511737171	0.0722518231	1.1197184513	0.9700427855	0.9737538709	0.5771490021	0.5	0.5019128463	1.0349089603	0.8965699988	0.9	0.1	0.086632741	0.0869641712
0.81	5.0893800988	1.1791940664	0.9300741707	0.8466911377	0.532434733	0.1950870716	0.0933570578	1	0.9414151154	0.8931964268	0.8270805743	1.1101111181	0.8755858827	0.7970878351	0.5012421056	0.1836579181	0.0878877453	1.0532519266	0.8307389259	0.7562614988	0.4755688011	0.1742510753	0.0833861904	0.9752885056	0.7692462792	0.7002817924	0.4403664248	0.1613527272	0.077213809	1.0711562935	0.9390822367	0.9676647609	0.5703208152	0.5	0.5152183286	0.9962548014	0.8734161325	0.9	0.1	0.0876699546	0.090338335
0.82	5.1522119519	1.1587315961	0.9358542878	0.8578437222	0.5507420583	0.2042568347	0.0978889188	1	0.9469109105	0.9034652031	0.8443279255	1.0972155907	0.8861706358	0.8123015801	0.5215036639	0.1934130253	0.0926920852	1.0468736769	0.8455117842	0.7750319527	0.4975762856	0.1845389426	0.0884392319	0.9783494447	0.7901679094	0.7243014104	0.4650068996	0.1724597495	0.0826503478	1.0215660895	0.907866019	0.9616881728	0.5626194109	0.5	0.5296421238	0.9560370051	0.8496303066	0.9	0.1	0.088870023	0.0941386155
0.83	5.215043805	1.1398513035	0.9415703081	0.8691040234	0.5704528554	0.2145218855	0.1029864814	1	0.9522161841	0.9133242446	0.860742027	1.0853848587	0.8965784859	0.8275749168	0.5431944412	0.2042712112	0.0980653943	1.0410538307	0.8599589903	0.7937737756	0.5210084233	0.195928039	0.0940600503	0.9811179214	0.8104491355	0.7480743588	0.4910127471	0.1846480026	0.0886447928	0.9710012289	0.8764966206	0.9558500223	0.5539104237	0.5	0.5452673746	0.9142659263	0.8252831931	0.9	0.1	0.0902673029	0.0984396305
0.84	5.277875658	1.1224618561	0.9471928309	0.8804139662	0.5916545768	0.2260677825	0.1087527309	1	0.9573133359	0.9227483809	0.87630668	1.0745477039	0.9067603287	0.8428320309	0.5663988166	0.216417703	0.1041104396	1.0357498603	0.8740206511	0.8124005618	0.5459483028	0.2086036803	0.1003514064	0.9836208226	0.830031405	0.7715126398	0.5184708579	0.1981047079	0.0953007446	0.9195171391	0.8450876907	0.9501761105	0.5440365238	0.5	0.5621760445	0.870960042	0.8004610022	0.9	0.1	0.0919055942	0.1033342469
0.85	5.3407075111	1.106480789	0.9526914727	0.8917083379	0.6144274079	0.2391236895	0.1153174326	1	0.9621857125	0.9317141626	0.8910065242	1.0646400064	0.9166661235	0.8579890224	0.5911932733	0.2300813976	0.110956786	1.0309238217	0.8876361377	0.8308172873	0.5724707178	0.2227949281	0.1074428851	0.9858816019	0.8488543178	0.7945179467	0.5474588291	0.2130607675	0.1027485848	0.8671715781	0.8137657818	0.9446919866	0.532813985	0.5	0.5804446486	0.8261469679	0.7752677212	0.9	0.1	0.0938413807	0.1089394545
0.86	5.4035393642	1.0918337723	0.9580350802	0.9029147514	0.6388369269	0.2539753459	0.122846339	1	0.966817624	0.9401998068	0.9048270525	1.0556041336	0.9262452	0.8729538947	0.6176387998	0.2455478405	0.1187700056	1.0265419018	0.9007443973	0.8489202748	0.6006343553	0.2387875711	0.1155001042	0.987920734	0.8668560577	0.8169816931	0.5780369336	0.2298037636	0.1111546908	0.8140250724	0.7826723518	0.9394228026	0.5200292757	0.5	0.6001379763	0.7798645755	0.7498275693	0.9	0.1	0.0961484331	0.1154046521
0.87	5.4663712172	1.0784539475	0.9631919815	0.913953822	0.6649236631	0.2709826523	0.1315544416	1	0.971194357	0.9481851475	0.9177546257	1.0473883881	0.9354466172	0.8876267945	0.6457701094	0.2631768227	0.1277649313	1.0225740153	0.913284331	0.8665974395	0.6304707415	0.2569417261	0.1247379676	0.9897560989	0.8839738965	0.8387853478	0.6102367675	0.2486955826	0.1207346973	0.7601415719	0.7519660192	0.9343931607	0.5054361184	0.5	0.6213001232	0.7321622669	0.7242876401	0.9	0.1	0.0989244697	0.1229235705
0.88	5.5292030703	1.0662813284	0.968130276	0.9247396016	0.6926886432	0.290603697	0.1417254752	1	0.9753021854	0.9556515896	0.9297764859	1.0399465099	0.9442195739	0.9019005544	0.6755807475	0.2834264207	0.1382251657	1.0189934465	0.9251952372	0.8837288703	0.661969003	0.2777158849	0.1354401757	0.9914033065	0.9001447659	0.8598011372	0.6440456125	0.2701964841	0.1317730143	0.7055894491	0.7218250291	0.9296269545	0.4887537981	0.5	0.6439420338	0.6831025081	0.6988206646	0.9	0.1	0.1023009953	0.1317518219
0.89	5.5920349234	1.0552622648	0.9728181579	0.9351803142	0.7220739806	0.313427725	0.1537413611	1	0.9791283796	0.962582068	0.940880769	1.0332372314	0.9525138665	0.9156615857	0.7070031266	0.3068859805	0.1505325297	1.0157765331	0.9364173142	0.9001878008	0.6950554655	0.3016999077	0.1479886773	0.9928759711	0.9153058964	0.8798931731	0.6793855221	0.2948981189	0.14465229	0.6504430513	0.6924498229	0.9251472053	0.4696680032	0.5	0.6680247252	0.632762811	0.6736277611	0.9	0.1	0.1064581782	0.1422333905
0.9	5.6548667765	1.0453489648	0.9772242725	0.9451794243	0.7529377602	0.3402202929	0.1681279282	1	0.9826612127	0.9689610096	0.9510565163	1.0272238814	0.9602803887	0.9287911593	0.7398827325	0.3343212856	0.1652127938	1.0129023883	0.9468922177	0.9158420093	0.7295673323	0.3296601985	0.162909407	0.9941859447	0.9293955122	0.8989190506	0.7160863632	0.3235687265	0.1598991617	0.5947851798	0.6640655131	0.9209758961	0.4478362211	0.5	0.6934375283	0.5812385146	0.6489409379	0.9	0.1	0.1116479589	0.1548417693
0.91	5.7176986295	1.0364990712	0.9813180976	0.9546370612	0.7850231207	0.3719841958	0.1856282252	1	0.9858899668	0.9747743001	0.9602936857	1.0218740349	0.9674716666	0.9411671005	0.7739464184	0.3667354864	0.1830090048	1.0103526567	0.9565636617	0.9305556732	0.765220363	0.3626006341	0.1809456233	0.9953435133	0.9423535727	0.916731942	0.7538527459	0.3572140744	0.1782576125	0.5387111787	0.6369239158	0.9171338043	0.4229007307	0.5	0.7199712411	0.5286470291	0.6250249653	0.9	0.1	0.1182310561	0.1702459203
0.92	5.7805304826	1.028675289	0.9850703437	0.9634517991	0.8179228424	0.4100381592	0.2073243731	1	0.9888049365	0.9800092538	0.9685831611	1.0171592038	0.9740424186	0.952665895	0.8087661442	0.405447756	0.2050033635	1.0081113024	0.9653780525	0.9441916788	0.8015719545	0.4018411905	0.2031798041	0.9963575632	0.9541225474	0.9331831892	0.7922262922	0.3971560565	0.2008108966	0.4823359575	0.6113046049	0.9136403362	0.3945135974	0.5	0.7472873007	0.4751348474	0.6021780372	0.9	0.1	0.126738344	0.1894199099
0.93	5.8433623357	1.0218450594	0.9884533645	0.9715227722	0.8510432447	0.4561069056	0.2348483739	1	0.9913974326	0.9846545876	0.9759167619	1.0130545684	0.9799501279	0.9631651821	0.8437220878	0.4521832152	0.2328280749	1.0061644255	0.97328514	0.9566143546	0.8379836351	0.449107757	0.2312445287	0.9972357215	0.9646482069	0.948125358	0.8305473676	0.4451223744	0.2291924646	0.4258066528	0.5875142105	0.9105133659	0.3623798754	0.5	0.7748862493	0.4208900186	0.5807303981	0.9	0.1	0.137976757	0.2138325834
0.94	5.9061941887	1.0159802775	0.9914415678	0.9787520722	0.8835749499	0.5123867635	0.2707628725	1	0.9936597846	0.9887003972	0.9822872507	1.0095387436	0.9851556147	0.9725465732	0.8779728944	0.509138121	0.2690461776	1.0045001039	0.9802386719	0.9676925626	0.873590904	0.5065969966	0.2677033596	0.9979844736	0.9738804119	0.9614156821	0.8679244084	0.5033109852	0.2659669176	0.3693268544	0.5658829509	0.9077690801	0.3263279569	0.5	0.8020820195	0.3661659956	0.56103988	0.9	0.1	0.1532200933	0.2457901637
0.95	5.9690260418	1.0110570505	0.9940118148	0.9850473476	0.9144826257	0.5814602765	0.3192782426	1	0.9955853419	0.9921381382	0.9876883406	1.0065935793	0.9896235925	0.9806987003	0.9104454976	0.5788933282	0.3178687383	1.0031082597	0.9861970313	0.9773030415	0.9072930897	0.5768889162	0.3167681211	0.9986092605	0.9817738799	0.9729197802	0.9032238271	0.5743015357	0.3153473976	0.3132064346	0.5467572439	0.9054218336	0.2864218427	0.5	0.8279925358	0.3113308965	0.5434831602	0.9	0.1	0.1745676919	0.2890814918
0.96	6.0318578949	1.0070554934	0.9961437986	0.9903244947	0.9425295265	0.6656801858	0.3876367634	1	0.9971684757	0.994960608	0.9921147013	1.0042039914	0.9933231933	0.9875203668	0.9398607312	0.6637952962	0.3865391605	1.0019805461	0.9911238395	0.9853338614	0.9377797508	0.6623255624	0.3856833098	0.9991145601	0.9882889072	0.9825154902	0.9350973996	0.6604310987	0.3845801317	0.2579734676	0.5304873598	0.9034840163	0.2431476102	0.5	0.8515603621	0.2569786699	0.5284416937	0.9	0.1	0.2056364031	0.3502236198
0.97	6.094689748	1.0039595601	0.9978203877	0.9945103042	0.9663522837	0.7650525044	0.4886314038	1	0.9984045799	0.9971619328	0.9955619646	1.0023578229	0.9962284449	0.9929236424	0.9648105458	0.7638319243	0.4878518314	1.0011102554	0.9949885064	0.9916878171	0.9636097109	0.762881234	0.487244635	0.9995039521	0.9933920255	0.9900966322	0.962063578	0.7616571743	0.4864628403	0.2046532401	0.5174096532	0.9019659361	0.1977671259	0.5	0.8716168421	0.2042071754	0.5162819008	0.9	0.1	0.2528226052	0.4407288815
0.98	6.157521601	1.0017569079	0.9990279256	0.997544915	0.9845929276	0.8719048811	0.6429101936	1	0.9992900711	0.9987375553	0.9980267284	1.0010457317	0.9983186868	0.9968367291	0.9838939366	0.8712858907	0.6424537731	1.0004922452	0.9977667081	0.9962855697	0.9833499335	0.8708041495	0.6420985551	0.9997801694	0.9970565722	0.995576488	0.9826500583	0.870184376	0.6416415572	0.1555427196	0.5078240824	0.9008757183	0.1531462617	0.5	0.8869958608	0.1553915205	0.5073304396	0.9	0.1	0.3264852791	0.5791821824
0.99	6.2203534541	1.0004387934	0.9997564752	0.9993839224	0.996076667	0.9627270885	0.8590136893	1	0.9998223898	0.9996842267	0.9995065604	1.0002611053	0.9995789082	0.9992064216	0.9958997537	0.9625560984	0.8588611197	1.0001228816	0.9994407788	0.9990683437	0.9957621326	0.962423085	0.8587424357	0.9999451373	0.9992631557	0.9988907868	0.9955851633	0.9622520408	0.8585898179	0.1164125802	0.5019693931	0.9002192257	0.1159558549	0.5	0.8966873659	0.1163842309	0.5018471511	0.9	0.1	0.4311985803	0.7733006382
1	6.2831853072	1	1	1	1	1	1	1	1	1	1	1	1	1	1	1	1	1	1	1	1	1	1	1	1	1	1	1	1	0.1	0.5	0.9	0.1	0.5	0.9	0.1	0.5	0.9	0.1	0.5	0.9
}\tablefrequencyresponse

\begin{document}

\begin{titlepage}
\maketitle
\begin{abstract}%
The transcription of handwritten text on images is one task in machine learning and one solution to solve it is using multi-dimensional recurrent neural networks (MDRNN) with connectionist temporal classification (CTC). The RNNs can contain special units, the long short-term memory (LSTM) cells. They are able to learn long term dependencies but they get unstable when the dimension is chosen greater than one. We defined some useful and necessary properties for the one-dimensional LSTM cell and extend them in the multi-dimensional case. Thereby we introduce several new cells with better stability. We present a method to design cells using the theory of linear shift invariant systems. The new cells are compared to the LSTM cell on the IFN/ENIT and Rimes database, where we can improve the recognition rate compared to the LSTM cell. So each application where the LSTM cells in MDRNNs are used could be improved by substituting them by the new developed cells.
\end{abstract}
\end{titlepage}

\begin{section}{Introduction}
Since the last decade, artificial neural networks (NN) became state-of-the-art in many fields of machine learning, for example they can be applied to pattern recognition. Typical NN are feedforward NN (FFNN) or recurrent NN (RNN), whereas the latter contain recurrent connections. 
When nearby inputs depend on each other, providing these inputs as additional information to the NN can improve its recognition result.  FFNNs obtain these dependencies by making this nearby inputs accessible. If RNNs are used, the recurrent connections can be used to learn if the surrounding input is relevant, but these connections result in a vanishing dependency over time. In \citet{sh97} the authors develop the long short-term memory (LSTM) which is able to have a long term dependency. This LSTM is extended in \citet{ag07} to the multi-dimensional (\multiD) case and is used in a hierarchical multi-dimensional RNN (MDRNN) which performed best in three competitions at the International Conference on Document Analysis and Recognition (ICDAR) in 2009 without any feature extraction and knowledge of the recognized language model.\\
In this paper we analyse these \multiD LSTM regarding the ability to provide long term dependencies in MDRNNs and show that it can easily have an unwanted growing dependency for higher dimensions. We define a more general description of an LSTM---a cell---and change the LSTM architecture which leads to new \multiD cell types, which also can provide long term dependencies. In two experiments we show that substituting the LSTM in MDRNNs by these cells works well. Due to this we assume that substituting the LSTM cell by the best performing cell, the \emph{LeakyLP cell}, will improve the performance of an MDRNN also in other scenarios. Furthermore the new cell types could also be used for the one-dimensional (\oneD) case, so using them in a bidirectional RNN with LSTMs (BLSTM) could lead to better recognition rates.\\
In Section \ref{S:previous_work} we introduce the reader to the development of the LSTM cells \citep{sh97} and its extension \citep{js99}.
Based on that in Section \ref{S:cell_properties} we define two properties that probably lead to the good performance of the \oneD LSTM cells. Both together guarantee that the cell can have a long term dependency. A third property ensures that gradient cannot explode over time.
In Section \ref{S:discussion} we show that the \multiD version of the LSTM is still able to provide long term dependency whereas the gradient can explode easily for dimension greater than \(1\). In Section \ref{S:Stable} we change the architecture of the \multiD LSTM cell and reduce it to the \oneD LSTM cell so that the cell fulfills the two properties for any dimension.
Nevertheless the internal cell state can linearly grow over time. This problem is solved in Section \ref{S:leaky} using a trainable convex combination of the input and the previous internal cell states. The new cell type can provide long term dependencies and does not suffer from exploding gradients.
Motivated by the last sections we introduce a more general way to define \multiD cells in Section \ref{S:13_leakyLP}. Using the theory of linear shift-invariant systems and their frequency analysis we are able to get a new interpretation of the cells and we create \(5\) new cell types.
To test the performance of the cells in Section \ref{S:expriments} we take two data sets from the ICDAR 2009 competitions, where the MDRNNs with LSTM cell won. On these data sets we compare the recognition results of the MDRNNs when we substitute the LSTM cells by the new developed cells. On both data sets, the IFN/ENIT data set and the RIMES data set we can improve the recognition rate using the new developed cells.\\

\end{section}

\begin{section}{Previous Work}
\label{S:previous_work}
In this section we briefly want to introduce a recurrent neural network (RNN) and the development of the LSTM cell. In previous literature there are various notation to describe the update equations of RNNs an LSTMs. To unify the notations we will refer to their notation using ``\(\triangleq\)'' \citep{js99,sh97,ag08}. Therefore we concentrate on a simple hierarchical RNN with one input layer with the set of neurons \(I\), one recurrent hidden layer with the set of neurons \(H\) and one output layer with the set of neurons \(O\). For each time step \(t\in\IN\) the layers are updated asynchronously in the order \(I,H,O\). In one specific layer all neurons can be updated synchronously. In the hidden layer for one neuron \(c\in H\) at time \(t\in\IN\) we calculate the neuron's \emph{input activation} \(net_c\) by
\begin{align}
\label{E:03_unit}
   \left(a_c(t)\triangleq\right)\quad net_c(t) =\sum_{i\in I}w_{c,i}y_i(t)+\sum_{h\in H}w_{c,h}y_h(t-1).
\end{align}
with weights \(w_{\text{[target neuron]},\text{[source neuron]}}\). A bias in \eqref{E:03_unit} can be added by extending the set \(I:=I\cup\left\{bias\right\}\) with \(y_{bias}(t)=1\forall t\in\IN\) and hence we will not write the bias in the equations, but we use them in our RNNs in Section \ref{S:expriments}. The neuron's \emph{output activation} is calculated by
\begin{align*}
 \left(y^c(t),b_c(t)\triangleq \right)\quad  y_c(t)=f_ c\left(net_c(t)\right)
\end{align*}
{
\begin{figure}
\centering
\begin{tikzpicture}[scale=0.5,transform shape,line width =1pt]

\newcommand{\picleftside}{0.35};
\newcommand{\picrightside}{0.8};
\newcommand{\picheight}{1.2};
\newcommand{\picarrowlength}{1.6};
\newcommand{\picrotation}{0};

\huge
\draw [fill=black!10] (-15,2) rectangle (-12.5,-2);
\draw [fill=black!10] (-1.5,2) rectangle (6.5,-2);
{
\node at (-3.5,.6) {$\mathbf{y}_H(t-1)$};
\node at (-3.5,-.6) {$\mathbf{y}_I(t)$};
 \node (in) at (0,0) [circle,draw,inner sep=2mm] {$\sum$};
 \draw[->] (-2,.8*\picheight) -- (in);
 \draw[->] (-2,.4*\picheight) -- (in);
 \draw[->] (-2,.0) -- (in);
 \draw[->] (-2,-.4*\picheight) -- (in);
 \draw[->] (-2,-.8*\picheight) -- (in);
 \node (input) at (2.5,0) {$net_\gamma(t)$};
 \node (squash_in) at (5,0) [circle,draw,inner sep=2mm] {$f_\gamma$};
 \node (output) at (8,0) {$y_\gamma(t)$};
 \draw (in) -- (input);
 \draw[->] (input) -- (squash_in);
 \draw[->] (squash_in) -- (output);
}
\drawtri{-14}{0}{$\gamma$};
 \node (output) at (-10.5,0) {$y_\gamma(t)$};
\draw[->] (0.8-14,0) -- (output);
{\Huge \node at (-7.5,0) {$:=$};
}\end{tikzpicture}
\caption{{Schematic diagram of a unit:} The unit \(\gamma\in H\) is a simple neuron with the network's  feed forward input \(\mathbf{y}_I(t)=\big(y_i(t)\big)_{i\in I}\) and recurrent input \(\mathbf{y}_H(t-1)=\big(y_h(t-1)\big)_{h\in H}\) and an output activation \(y_\gamma(t)\). {\em Right:} A unit has an input activation \(net_\gamma(t)\), which is a linear combination of the source activations \(\mathbf{y}_I(t),\mathbf{y}_H(t-1)\). The output activation \(y_\gamma(t)\) is computed by applying the activation function \(f_\gamma\) to the input activation. {\em Left:} The short notation of a unit.}
\label{Fig:unit}
\end{figure}
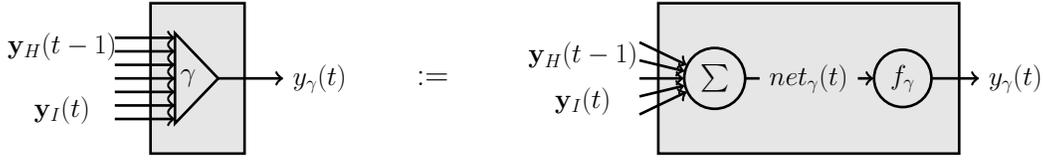
}
with a differentiable sigmoid \emph{activation function} \(f_c\). To make \eqref{E:03_unit} suitable for \(t\leq0\) we define \(\forall h\in H,\forall t\in\IZ\setminus\IN:y_h(t)=0\). This simple neuron with a linear function of activations as input and one activation function we call \emph{unit} (compare to Figure \ref{Fig:unit}). In \eqref{E:03_unit} the activation of the unit is dependent on the current activations of the layer below and the previous activations of the units from the same layer. When there are no recurrent connections (\(\forall c,h\in H:w_{c,h}=0\)), the layer is called feed-forward layer, otherwise recurrent layer. 
\subsection{The Long Short-Term Memory}
\label{SS:LSTM}
A \emph{standard LSTM} cell \(c\) has one input with an input activation \(\yin(t)\) a set of \emph{gates}, one \emph{internal state \(s_c\)} and one output(-activation) \(y_c\) (\(\triangleq y^c\)). The gates are also units and their task is to learn whether \asignal should pass the gate or not. They almost always have the logistic activation function \(\fl(x):=\frac{1}{1+\exp(-x)}\) (\(\triangleq f_1(x)\)). The input of the standard LSTM cell is calculated from a unit with an odd activation function with a slope of \(1\) at \(x=0\). We use \(f_c(x)=\tanh\left(x\right)\) in this paper, another solution could be \(f_c(x)=2\tanh\left(\frac{x}{2}\right)\) \citep[see][]{sh97}. The standard LSTM has two gates: The \emph{input gate} (IG or \(\iota\)) and the \emph{output gate} (OG or \(\omega\)).
These both gates are calculated like a unit, so that
\begin{align*}
net_\iota(t)&=\sum_{i\in I}w_{\iota,i}y_i(t)+\sum_{h\in H}w_{\iota,h}y_h(t-1)\\
\left(y^{in_c}(t),b_\iota(t)\triangleq\right)\quad y_\iota(t)&=\fl\left(net_\iota(t)\right)
\end{align*}
and
\begin{align*}
net_\omega(t)&=\sum_{i\in I}w_{\omega,i}y_i(t)+\sum_{h\in H}w_{\omega,h}y_h(t-1)\\
\left(y^{out_c}(t),b_\omega(t)\triangleq\right)\quad y_\omega(t)&=\fl\left(net_\omega(t)\right).
\end{align*}
The input of an LSTM is defined like in \eqref{E:03_unit} by
\begin{align*}
\left(net_c(t)\triangleq\right)\quad net_\cin(t)&=\sum_{i\in I}w_{c,i}y_i(t)+\sum_{h\in H}w_{c,h}y_h(t-1),\\
  \left(g\left(net_c(t)\right),f_2\left(net_c(t)\right)\triangleq\right)\quad \yin(t)&=f_c\left(net_\cin(t)\right).
\end{align*}
The internal state \(s_c(t)\) is calculated by
\begin{align}
\label{E:state_lstmd1}
   s_c(t)=\yin(t)\cdot y_\iota(t)+s_c(t-1),
\end{align}
the output activation \(y_c(t)\) of the LSTM is calculated from
\begin{align}
\label{E:out_lstmd1}
 \left(y^c(t),b_c(t)\triangleq\right)\quad y_c(t)=\fout{s_c(t)}\cdot y_\omega(t)
\end{align}
with \(\fout{x}:=\tanh(x)\) (\(\triangleq f_3(x)\)). The LSTM can be interpreted as a kind of memory module where the internal state stores the information. For a given input \(\yin(t)\in(-1,1)\) the IG ``decides'' if the new input is relevant for the internal state. If so, the input is added to the internal state. The information of the input is now saved in the activation of the internal state. The OG determines whether or not the internal activation should be displayed to the rest of the network. So the information, stored in the LSTM is just ``readable'' when the OG is active. To sum up, an open IG can be seen as a ``write''-operation into the memory and an open OG as a ``read''-operation of the memory.\\

Another way to understand the LSTM is to take a look at the gradient propagated through it. To analyse the LSTM properly, we have to ignore gradients comming from recurrent weights. We define the \emph{truncated gradient} similar to \citet{sh97} and \citet{js99}.
\begin{definition}[truncated gradient]
\label{defi:truncated_gradient}
Let \(\gamma\in\left\{\cin,\iota,\omega\right\}\) be any input or gate unit and \(y_c(t-1)\) any previous output activation.
The \emph{truncated gradient} differs from the exact gradient only by setting recurrent weighted gradient propagation \(\frac{\partial net_\gamma(t)}{\partial y_c(t-1)}\) to zero. We write 
\begin{align*}
\frac{\partial net_\gamma(t)}{\partial y_c(t-1)} \Big(=w_{\gamma,c}\Big)&\eqtr0.
\end{align*}
\end{definition}
Now, let \(E\) be an arbitrary error which is used to train the RNN and \(\frac{\partial E(t)}{\partial y_c(t)}\) the resulting \der at the output of the LSTM. The OG can eliminate the gradient coming from the output, because
\begin{align*}
   \frac{\partial y_c(t)}{\partial s_c(t)}=\underbrace{\foutder{s_c(t)}}_{\in(0,1]}\cdot \underbrace{y_\omega(t)}_{\in(0,1)},
\end{align*}
so the OG decides when the gradient should go into the internal state. Especially for \(\left|s_c(t)\right|\ll 1\) we get
\begin{align*}
   \frac{\partial y_c(t)}{\partial s_c(t)}\approx y_\omega(t).
\end{align*}
The key idea of the LSTMs is that an error that occurs at the internal state neither explode nor vanish over time. Therefore, we take a look at the partial \der \(\frac{\partial s_c(t)}{\partial s_c(t-1)}\), which is also known as error carousel \citep[for more details see][]{sh97}. Using the truncated gradient of Definition \ref{defi:truncated_gradient} for this \der, we get
\begin{align}
\frac{\partial s_c(t)}{\partial s_c(t-1)}
=&\yin(t)\cdot\frac{\partial y_\iota(t)}{\partial s_c(t-1)}
+y_\iota(t)\cdot\frac{\partial \yin(t)}{\partial s_c(t-1)}+1\notag\\
=&\yin(t)\cdot\frac{\partial y_\iota(t)}{\partial y_c(t-1)}\frac{\partial y_c(t-1)}{\partial s_c(t-1)}
+y_\iota(t)\cdot\frac{\partial \yin(t)}{\partial y_c(t-1)}\frac{\partial y_c(t-1)}{\partial s_c(t-1)}+1\notag\\
=&\yin(t)\cdot\frac{\partial y_\iota(t)}{\partial net_\iota(t)}\underbrace{\frac{\partial net_\iota(t)}{\partial y_c(t-1)}}_{\eqtr0}\frac{\partial y_c(t-1)}{\partial s_c(t-1)}\notag\\
&+y_\iota(t)\cdot\frac{\partial \yin(t)}{\partial net_\cin(t)}\underbrace{\frac{\partial net_\cin(t)}{\partial y_c(t-1)}}_{\eqtr0}\frac{\partial y_c(t-1)}{\partial s_c(t-1)}+1\notag\\
\label{E:derivate_state_lstmd1}
\Rightarrow\frac{\partial s_c(t)}{\partial s_c(t-1)}\eqtr&1.
\end{align}
So, once having a gradient at the internal state we can use the chain rule and get \(\forall \tau\in\IN:\frac{\partial s_c(t)}{\partial s_c(t-\tau)}\eqtr1\). This is called \emph{constant error carousel}.\\
Like the OG can eliminate the gradient coming from the LSTM output, the IG can do the same with the gradient coming from the internal state, that means it decides when the gradient should be injected to the source activations. This can be seen by taking a look at the partial \der
\begin{align*}
   \frac{\partial s_c(t)}{\partial net_\cin(t)}&=\frac{\partial s_c(t)}{\partial \yin(t)}\frac{\partial \yin(t)}{\partial net_\cin(t)}=y_\iota(t)f'_c\left(net_\cin(t)\right).
\end{align*}
If there is a small input \(\left|net_\cin(t)\right|\ll 1\), we get \(f'_c\left(net_\cin(t)\right)\approx1\) and can estimate
\begin{align*}
    \frac{\partial s_c(t)}{\partial net_\cin(t)}\approx y_\iota(t).
\end{align*}
All in all, this LSTM is able to store information and learn long-term dependencies, but it has one drawback which will be discussed in \ref{SS:Learning_to_forget}.
\subsection{Learning to Forget}
\label{SS:Learning_to_forget}
For long time series the internal state is unbounded \citep[compare with][2.1]{js99}. Assuming a positive or negative input and a non zero activation of the IG, the absolute activation of the internal state grows over time. Using the weight-space symmetries in a network with at least one hidden layer \citep[5.1.1]{cb06} we assume without loss of generality \(\yin(t)\geq 0\), so \(s_c(t)\xrightarrow{t\to \infty}\infty\). Hence, the activation function \(\foutshort\) saturates and \eqref{E:out_lstmd1} can be simplified to
\begin{align*}
   y_c(t)= \underbrace{\fout{s_c(t)}}_{\to1} y_\omega(t)\approx y_\omega(t).
\end{align*}
Thus, for great activations of \(s_c(t)\) the whole LSTM works like a unit with a logistic activation function. A similar problem can be observed for the gradient. The gradient coming from the output is multiplied by the activation of the OG and the \der of \(\foutshort\). For great values of \(s_c(t)\) we get \(\foutder{s_c(t)}\to0\) and we can estimate the partial \der
\begin{align*}
   \frac{\partial y_c(t)}{\partial s_c(t)}=\foutder{(s_c(t)}\cdot y_\omega(t)\approx 0,
\end{align*}
which can be interpreted that the OG is not able to propagate back the gradient into the LSTM. Some solutions to solve the linear growing state problem are introduced in \citet{js99}. They tried to stabilize the LSTM with a ``state decay'' by multiplying the internal state in each time step with a value \(\in\left(0,1\right)\), which did not improve the performance. Another solution was to add an additional gate, the forget gate (FG or \(\phi\)). The last state \(s_c(t-1)\) is multiplied by the activation of the FG before it is added to the current state \(s_c(t)\). So we can substitute \eqref{E:state_lstmd1} by
\begin{align*}
    s_c(t)=\yin(t)\cdot y_\iota(t)+s_c(t-1)\cdot y_\phi(t),
\end{align*}
so that the truncated gradient in \eqref{E:derivate_state_lstmd1} is changed to
\begin{align*}
   \frac{\partial s_c(t)}{\partial s_c(t-1)}
&=\yin(t)\cdot\frac{\partial y_\iota(t)}{\partial s_c(t-1)}
+y_\iota(t)\cdot\frac{\partial \yin(t)}{\partial s_c(t-1)}+y_\phi(t)\\
&\eqtr y_\phi(t)
\end{align*}
and for longer time series we get \(\forall \tau\in\IN\)
\begin{align*}
     \frac{\partial s_c(t)}{\partial s_c(t-\tau)}\eqtr\prod_{t'=0}^{\tau-1}y_\phi(t-t').
\end{align*}
Now, the \emph{Extended LSTM} is able to learn to forget its previous state. However, an Extended LSTM is still able to work like an standard LSTM without FG by having an activation \(y_\phi(t)\approx 1\). In this paper we denote the Extended LSTM as \emph{LSTM}\\
Another point of view was introduce in \citet{be94}:
To learn long-term dependencies a system must have an architecture to that an input can be saved over long time and does not suffer from the ``vanishing gradient'' problem.
On the other hand the system should avoid an ``exploding gradient'', which means that a small disturbance has a growing influence over time. In this paper we do not want to solve the problem of vanishing and exploding gradient for a whole system, we want to solve this problem only for one single cell. But we think that it is an necessary condition to provide long time dependencies of a system.
\end{section}

\begin{section}{Cells and Their Properties}
\label{S:cell_properties}
\begin{figure}
\centering
\begin{tikzpicture}[scale=1,line width =1pt]
\large
\draw[fill,color=black!15,rounded corners=8mm] (-1,-3) rectangle  (9,4.5);
\draw[fill,color=black!30,rounded corners=8mm] (1,1.2) rectangle  (9,4.5);
\node (bgamma) at (5.2,2)  [rectangle,draw,inner xsep=28mm] {$\mathbf{y}_{\Gamma}(t)$};
\node at (1.5,2.8) {\Large{$\Gamma$}};
\node at (-0.5,3.5) {\Huge{$c$}};
\drawuniteasy{3.3}{3.5}{$\gamma_1$}{-90};
\node (gamma1) at (3.3,3.5-0.8) [inner sep=0mm, outer sep=-0.5pt] {};
\node at (5.2,3.5) {\huge{$\cdots$}};
\drawuniteasy{7.1}{3.5}{$\gamma_{\left|\Gamma\right|}$}{-90};
\node (gamma2) at (7.1,3.5-0.8) [inner sep=0mm, outer sep=-0.5pt] {};
\drawuniteasy{0}{0}{$\cin$}{0};
\node (frein) at (3.3,0)  [rectangle,draw,inner sep=3mm,rounded corners=3mm] {$\frein{\cdot}$};
\node (fraus) at (7.1,0)  [rectangle,draw,inner sep=3mm,rounded corners=3mm] {$\fraus{\cdot}$};
\node (memory) at (5.2,-1.5)  [rectangle,draw,inner sep=3mm] {$\substack{\text{memory}\\s_c(t-1),\dots, s_c(t-k)}$};
\node (point) at (5.2,0) [circle,fill,inner sep=0.5mm, outer sep=-0.5pt] {};
 \node (pointrein) at (3.3,-1.5) [inner sep=0mm, outer sep=-0.5pt] {};
 \node (pointraus) at (7.1,-1.5) [inner sep=0mm, outer sep=-0.5pt] {};
\draw[->] (gamma1) -- (bgamma.169);
\draw[->] (gamma2) -- (bgamma.11);
\draw[->] (0+0.8,0) -- (frein);
\node (point1) at (2,0) [circle,fill,inner sep=0.5mm, outer sep=-0.5pt] {};
\node (point2) at (2,-2.5) [inner sep=0mm, outer sep=-0.5pt] {};
\node (point3) at (7.55,-2.5) [inner sep=0mm, outer sep=-0.5pt] {};
\draw (point1) -- (point2);
\draw (point2) -- (point3);
\draw[->] (point3) -- (fraus.310);

\draw[->] (bgamma.349) -- (fraus);
\draw[->] (bgamma.191) -- (frein);
\draw[->] (frein) -- (fraus);
\draw[->] (fraus) -- (10,0);
\draw[->] (point) -- (memory);
 \draw (memory) -- (pointrein);
 \draw (memory) -- (pointraus);
 \draw[->] (pointrein) -- (frein);
 \draw[->] (pointraus) -- (fraus);
\node at (1.4,0.5) {$\yin(t)$};
\node at (5.2,0.5) {$s_c(t)$};
\node at (10,0.5) {$y_c(t)$};
\end{tikzpicture}
\caption{{Schematic diagram of a cell:} The function \(g_{int}\) calculates the internal state \(s_c(t)\)
from the previous internal states \(s_c(t-1),\dots, s_c(t-k)\) and the cell input \(y_{c_{in}}(t)\) using the gate activations $\mathbf{y}_{\Gamma}(t)$. The function \(g_{out}\) calculates the output \(y_c(t)\) of the cell from the actual and previous internal states \(s_c(t),\dots, s_c(t-k)\), the cell input \(y_{c_{in}}(t)\) also using the gate activations $\mathbf{y}_{\Gamma}(t)$.}
\label{Fig:cell}
\end{figure}
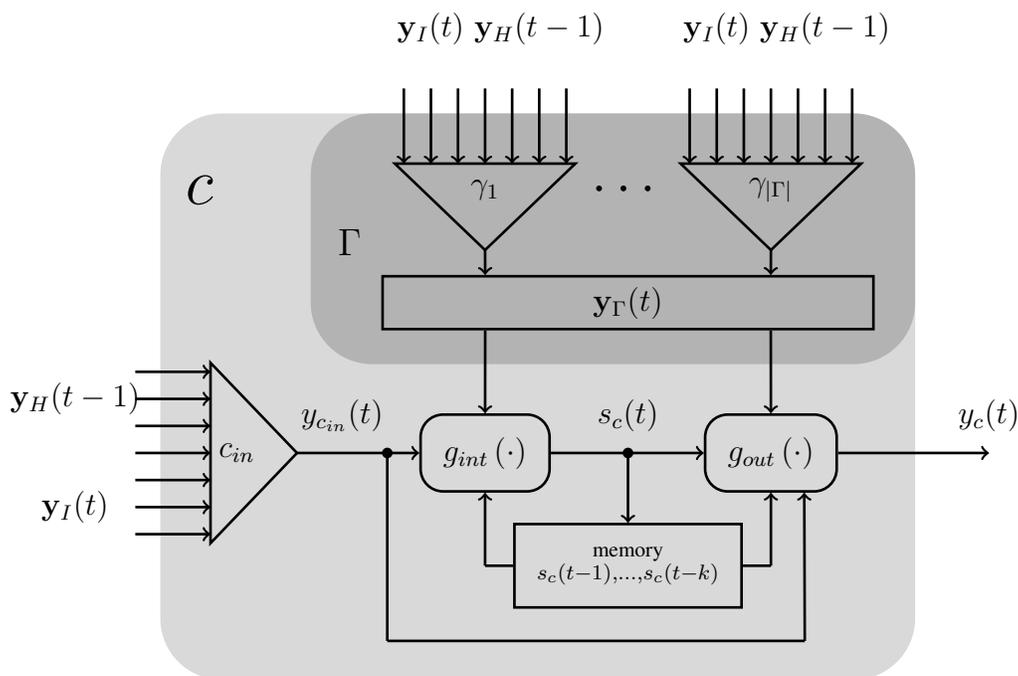

In this section we want to introduce a general cell and figure out properties for these cells which probably lead to the good performance observed by LSTM cells.
\begin{definition}[Cell, cf. Fig. \ref{Fig:cell}]
\label{defi:cell_d1}
	A \emph{cell}, $c$, \emph{of order $k$} consists of
	\begin{itemize}
		\item one designated \emph{input} unit, $\cin$, with sigmoid activation function $f_c$ (typically $f_c=\tanh$ unless specified otherwise);
		\item a set $\Gamma$ (not containing $\cin$) of units called \emph{gates} $\gamma_1, \gamma_2, \dots$ with sigmoid activation functions $f_{\gamma_i}$, $i=1,\dots$ (typically logistic $f_{\gamma_i}=\fl$ unless specified otherwise);
		\item an arbitrary function, $\frein{}$, and a cell activation function, $\fraus{}$, mapping into $[-1,1]$.
	\end{itemize}
	Each unit of \(\Gamma\cup\{\cin\}\) receives the same set of input activations.
	The cell update in time step $t\in\IN$ is performed in three subsequent phases:
	\begin{enumerate}
		\item Following the classical update scheme of neurons (see Section \ref{S:previous_work}), all units in $\Gamma\cup\{\cin\}$ calculate synchronously their activations, which will be denoted by $\fy(t):=\big(y_\gamma(t)\big)_{\gamma\in\Gamma}$ and $\yin(t)$. Furthermore, we call \(\yin(t)\) the \emph{input activation} of the cell.
		\item Then, the cell computes it's so-called \emph{internal state} 
			\[ s_c(t):=\frein{\fy(t),\yin(t),s_c(t-1),\ldots,s_c(t-k)} \,.\]
		\item Finally, the cell computes it's so-called \emph{output activation} 
			\[ y_c(t):=\fraus{\fy(t),\yin(t),s_c(t),s_c(t-1),\ldots,s_c(t-k)} \,.\]
	\end{enumerate}
\end{definition}
In this paper we concentrate on first order cells (\(k=1\)). Now, we use Definition \ref{defi:cell_d1} to re-introduce the (Extended) LSTM cell.
\begin{remark}[LSTM cell]
  An \emph{LSTM cell} is a cell of order \(1\) where \(\foutshort=\tanh\) and
	\begin{itemize}
		\item $\Gamma = \{\iota, \phi, \omega\}$
		\item $s_c(t):=\frein{\fy(t),\yin(t),s_c(t-1)}:=\yin(t)y_\iota(t)+s_c(t-1)y_\phi(t)$
		\item $y_c(t):=\fraus{\fy(t),s_c(t)}:= \fout{s_c(t)}y_\omega(t)$
	\end{itemize}
\end{remark}
{
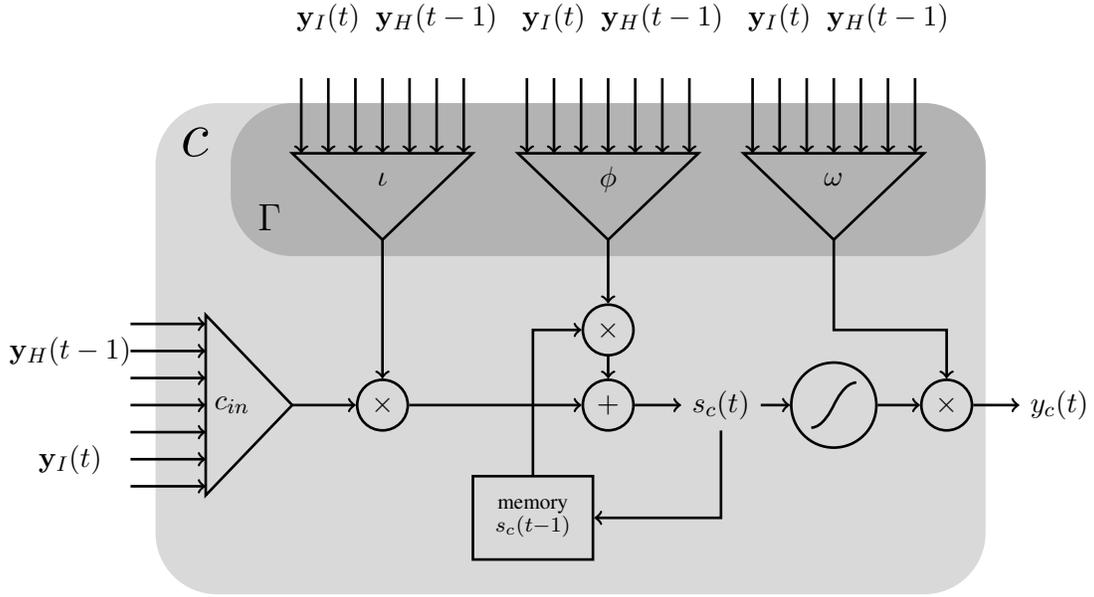
\begin{figure}
\centering
\begin{tikzpicture}[scale=1,line width =1pt]
\draw[fill,color=black!15,rounded corners=8mm] (1,-2.5) rectangle  (12,4);
\draw[fill,color=black!30,rounded corners=8mm] (2,2) rectangle  (12,4);
\node at (2.5,2.5) {\Large{$\Gamma$}};
\node at (1.5,3.5) {\Huge{$c$}};
\drawuniteasy{2}{0}{$\cin$}{0};
\drawuniteasy{4}{3}{$\iota$}{-90};
\drawuniteasy{7}{3}{$\phi$}{-90};
\drawuniteasy{10}{3}{$\omega$}{-90};
 \node (times_in) at (4,0) [circle,draw,inner sep=1mm] {$\times$};
 \node (past_1_plus) at (7,0) [circle,draw,inner sep=1mm] {$+$};
 \node (currstate) at (8.5,0) {$s_c(t)$};
 \node (squash_state) at (10,0) [circle,draw,inner sep=4mm] {};
 \sigmoid[(10,0)]
 \node (times_og) at (11.5,0) [circle,draw,inner sep=1mm] {$\times$};
 \node (out) at (13,0) {$y_c(t)$};
\node (state1) at (6,-1.5)  [rectangle,draw,inner sep=3mm] {$\substack{\text{memory}\\s_c(t-1)}$};
 \node (times_fg1) at (7,1)   [circle,draw,inner sep=1mm] {$\times$};
 
\draw[->] (2+0.8,0) -- (times_in);
\draw[->] (4,3-0.8) -- (times_in);
\draw[->] (7,3-0.8) -- (times_fg1);
\draw[->] (10,3-0.8) -- (10,1) -- (11.5,1) -- (times_og);

\draw[->]  (times_in) -- (past_1_plus); 
 \draw[->] (past_1_plus) -- (currstate);
\draw[->] (currstate) -- (squash_state) ;
\draw[->] (squash_state) -- (times_og); 
\draw[->] (times_og) -- (out);
\draw[->] (state1) -- (6,1) -- (times_fg1);
\draw[->] (currstate) -- (8.5,-1.5) -- (state1) ;
\draw[->] (times_fg1) -- (past_1_plus);
\end{tikzpicture}
\caption{{Schematic diagram of a one-dimensional LSTM cell:} The input (\(\cin\)) is multiplied by the IG (\(\iota\)). The previous state \(s_c(t-1)\) is gated by the FG (\(\phi\)) and added to the activation coming from the IG and input. The output of the cell is the squashed internal state (squashed by \(\fout{x}=\tanh(x)\)) and gated by the OG (\(\omega\)).}
\label{Fig:lstm_1d}
\end{figure}
}
\paragraph{Properties of cells.} Developing the \oneD LSTM cells, the main idea is to save \emph{exactly one piece of information} over a long time series and to propagate the gradient back over this long time, so that the system can learn precise storage of this piece of information. In instance a given input \(\yin\) (which represent the information) at time \(t_{in}\) should be stored into the cell state \(s_c\) until the information is required at time \(t_{out}\).\\
To be able to prove the following properties, we will assume the truncated gradient defined in Definition \ref{defi:truncated_gradient}. Nevertheless we will use the full gradient in our Experiments, because it turned out that it works much better. The next two properties of a cell ensure the ability to work as such a memory.\\
The first property should ensure that an input \(\yin\) at time \(t_{in}\) can be memorized (the cell input is open) in the internal activation \(s_c\) until \(t_{out}\) (the cell memorizes) and has a negligibly influence on the internal activation for \(t> t_{out}\) (the cell forgets). In addition, the cell is able to prevent influence of other inputs at time steps \(t\neq t_{in}\) (the cell input is closed).
\begin{definition}[Not vanishing gradient (\LTD)]
\label{defi:LTD}
   A cell \(c\) allows an \LTD \(:\Leftrightarrow\)\\
For arbitrary \(t_{in},t_{out}\in\IN,t_{in}\leq t_{out},\forall\delta>0\) there exist gate activations \(\fy(t)\) such that for any \(t_1,t_2\in\IN\)
\begin{align}
\label{E:LTD_1D}
\frac{\partial s_c\left(t_2\right)}{\partial \yin\left(t_1\right)}&\intr\left\{
\begin{array}{ccl}
 \left[1-\delta,1\right]&\text{ for }&t_1=t_{in}\text{ and } t_{in}\leq t_2\leq t_{out}\\
\left[0,\delta\right]&\text{ otherwise }&\quad 
\end{array}\right.
\end{align} 
holds.
\end{definition}
The next definition guaranties that at any time \(t\in\IN\) the gate activations can (the cell output is open) or not (the cell output is closed) distribute the piece of information saved in \(s_c\) to the network. This is an important property because the piece of information can be memorized in the cell without presenting it to the network. Note that the decision is just dependent on gate activations at time \(t\) and there are no constraints to previous gate activations. In Definition \ref{defi:cell_d1} we require \(y_c(t)\in\left[-1,1\right]\) whereas \(s_c(t)\in\IR\). So we cannot have arbitrarily small intervals of the \der as in \eqref{E:LTD_1D}, but we can ensure two distinct intervals for open and closed cell output.
When we take Definition \ref{defi:LTD} and \ref{defi:COG} together, a cell is able to save an input over long term series, can decide at each time step whether or not it is presented to the network and can forget the saved input.\\
\begin{definition}[Controllable output dependency (\COD)]
\label{defi:COG}
  A cell \(c\) of order \(k\) allows an \COD \(:\Leftrightarrow\)\\
There exist \(\delta_1,\delta_2\in(0,1),\delta_2<\delta_1\) so that for any time \(t\in\IN\) there exists a gate vector \(\fy(t)\) leading to open output dependency
\begin{align}
\label{E:05_COG_1}
\frac{\partial y_c\left(t\right)}{\partial s_c\left(t\right)}\in\left[\delta_1,1\right]
\end{align}
and there exists another gate vector \(\fy(t)\) leading to a closed output dependency
\begin{align}
\label{E:05_COG_2}
\frac{\partial y_c\left(t\right)}{\partial s_c\left(t\right)}\in\left[0,\delta_2\right].
\end{align}
\end{definition}
The third property is a kind of stability criterion. An unwanted case is that a small change (caused by any noisy signal) at time step \(t_{in}\) has a growing influence at later time steps. This is equivalent to an exploding gradient over time. Controlling the gradient of the whole system and avoiding him not to explode is a hard problem. But we can at least avoid the exploding gradient in one cell. This should be prohibited for \emph{any} gate activations.  
\begin{definition}[Not exploding gradient (\NGEC)]
\label{defi:not_growing_1d}
  A cell \(c\) has an \NGEC \(:\Leftrightarrow\)\\
For any time steps \(t_{in},t\in\IN, t_{in}< t\) and any gate activations \(\fy(t)\) the truncated gradient in bounded by
\begin{align*}
\frac{\partial s_c\left(t\right)}{\partial s_c\left(t_{in}\right)}\intr\left[0,1\right].
\end{align*}
\end{definition}
We think that a cell fulfilling these three properties can work as stable memory. To be able to prove these properties for the LSTM cell we have to considerate the gate activations. In general, the activation function of the gates does not have to be the logistic activation function \(\fl\), whereas for this paper we set \(\forall\gamma\in\Gamma:f_\gamma:=\fl\). So the activation of gates can never be exactly \(0\) or \(1\), because of a finite input activation \(net_\gamma(t)\) to the gate activation function. But a gate can have an activation \(y_\gamma(t)\in\left[1-\eps,1\right)\) if it is opened or \(y_\gamma(t)\in\left(0,\eps\right]\) if it is closed, because for a realistic large input activation \(net_\gamma(t)\geq 7\) (low input activation \(net_\gamma(t)<-7\)) we get an activation within the interval \(y_\gamma(t)\in[1-\eps,1)\) (\(y_\gamma(t)\in(0,\eps]\)) with \(\eps<\frac{1}{1000}\). Handling with these activation intervals we can prove the definitions for the LSTM cell.
Now we can prove whether or not the LSTM cell has these properties.
\begin{theorem}[Properties of the LSTM cell]
\label{theo:ltd_1D}
   The \oneD LSTM cell allows \LTD and has an \NGEC, but does not allow \COD.
\end{theorem}
\begin{proof}
   see \ref{proof:theo:ltd_1D} in appendix.
\end{proof}

\end{section}

\begin{section}{Expanding to More Dimensions}
\label{S:discussion}
In \citet{ag07} the \oneD LSTM cell is extended to an arbitrary number of dimensions; this is solved by using one FG for each dimension.
In many publications using the \multiD LSTM cell in MDRNNs outperform state-of-the-art recognition systems \citep[for example see][]{ag08}.\\
But by expanding the cell to the \multiD case, the absolute value of the internal state \(\left|s_c\right|\) can grow faster than linear over time.
When \(\left|s_c^\fp\right|\to\infty\) and there are peephole connections \cite[for peephole connection details see][]{fg02}, the cells have an output activation of \(y_c^\fp\in\left\{-1,0,1\right\}\): The internal state multiplied by the peephole weight overlays the other activation-weight-products and this leads to an activation of the OG \(y_\omega^\fp\in\left\{0,1\right\}\) and a squashed internal state \(\fout{s_c^\fp}\in\left\{-1,1\right\}\). So the output of the cell is \(y_\omega^\fp \fout{s_c^\fp}=y_c^\fp\in\left\{-1,0,1\right\}\). But also without peephole connections the internal state can grow, which leads to \(\fout{s_c^\fp}\in\left\{-1,1\right\}\) and the cell works like a conventional unit with a logistic activation function \(y_c(t)\approx \pm y_\omega(t)\).\\
Our goal is to transfer the Definitions \ref{defi:LTD}, \ref{defi:COG} and \ref{defi:not_growing_1d} defined in Section \ref{S:cell_properties} into the \multiD case and we will see that the \multiD LSTM cell has an exploding gradient. In the next sections we will provide alternative cell types, that fulfill two or all of these definitions.\\
In the \oneD case it is clear, that there is just one way to come from date \(t_1\) to date \(t_2\), when \(t_1<t_2\), by incrementing \(t_1\) as long as \(t_2\) is reached. For the \multiD case the number of paths depends on the number of dimensions and the distance between these two dates. An \multiD path is defined as follows.
\begin{definition}[\multiD path]
\label{defi:md_path}
   Let \(\fp,\fq\in\IN^D\) be two dates. A \(\fp\)-\(\fq\)-path \(\pi\) of length \(k\geq0\) is a sequence
\begin{align*}
   \pi:=\{\fp=\fp_0,\fp_1,\ldots,\fp_k=\fq\}
\end{align*}
with \(\forall i\in\{1,\ldots,k\} \exists!d\in\{1,\ldots, D\} : \left(\fp_i\right)_d^-=\fp_{i-1}\). Further, let \(\pi_i:=\fp_i\).
\end{definition}
We can define the distance vector
\begin{align*}
   \overrightarrow{\fp\fq}:=\fq-\fp=\begin{pmatrix}
                       \fq_1-\fp_1\\\vdots\\\fq_D-\fp_D
                    \end{pmatrix}=\begin{pmatrix}
                       \overrightarrow{\fp\fq}_1\\\vdots\\\overrightarrow{\fp\fq}_D
                    \end{pmatrix}
\end{align*}
between the dates \(\fp\) and \(\fq\). When \(\overrightarrow{\fp\fq}\) has at least one negative component, there exists no \(\fp\)-\(\fq\)-path. Otherwise there exist exactly
\begin{align*}
\#\{\overrightarrow{\fp\fq}\}:=   \begin{pmatrix}
\sum_{i=1}^D\overrightarrow{\fp\fq}_i\\[10px]
\overrightarrow{\fp\fq}_1,\ldots,\overrightarrow{\fp\fq}_D
   \end{pmatrix}=\frac{\left(\sum_{i=1}^D\overrightarrow{\fp\fq}_i\right)!}{\prod_{i=1}^D\overrightarrow{\fp\fq}_i!}
\end{align*}
\(\fp\)-\(\fq\)-paths (compare with the multinomial coefficient). We write \(\fp<\fq\) when \(\#\{\overrightarrow{\fp\fq}\}\geq1\) and \(\fp\leq\fq\) when \(\fp=\fq \lor \fp<\fq\). Now we can extend the definitions of the \oneD case to the \multiD case, whereas we concentrate on the \multiD cells of order 1.
\begin{definition}[\multiD cell]
\label{Def:md_cell}
An \multiD cell, \(c\), of order \(1\) and dimension \(D\) consists of the same parts as a \oneD cell of order \(1\). The cell update in date $\fp\in\IN^D$ is performed in three subsequent phases:
\begin{enumerate}
\item Following the classical update scheme of neurons (see Section \ref{S:previous_work}), all units in $\Gamma\cup\{c_{in}\}$ synchronously calculate their activations, which will be denoted by $\fy^\fp=\big(y_\gamma^\fp\big)_{\gamma\in\Gamma}$. Furthermore, we call \(\yin^\fp\) the \emph{input activation} of the cell.
\item Then, the cell computes it's so-called \emph{internal state} 
\[ s_c^\fp:=\frein{\fy^\fp,\yin^\fp,s_c^{\fp_1^-},\ldots,s_c^{\fp_D^-}} \,.\]
\item Finally, the cell computes it's so-called \emph{output activation} 
\[ y_c^\fp:=\fraus{\fy^\fp,\yin^\fp,s_c^\fp,s_c^{\fp_1^-},\ldots,s_c^{\fp_D^-}} \,.\]
\end{enumerate}
\end{definition}
Using this, we can reintroduce the LSTM cell as well as Definition \ref{defi:LTD}, \ref{defi:COG} and \ref{defi:not_growing_1d} for the \multiD case:
\begin{definition}[\multiD LSTM cell]
\label{defi:md_LSTM}
   An \multiD LSTM cell is a cell of dimension \(D\) and order \(1\) where \(\foutshort=\tanh\) and
\begin{itemize}
\item \(\Gamma=\left\{\iota,\left(\phi,1\right),\ldots,\left(\phi,D\right),\omega\right\}\)
\item \(s_c^\fp=\frein{\fy^\fp,\yin^\fp,s_c^{\fp_1^-},\ldots,s_c^{\fp_D^-}}=y_\iota^\fp \yin^\fp +\sum\limits_{d=1}^D s_c^{\fp_d^-}y_{\phi,d}^\fp\)
\item \(y_c^\fp=\fraus{\fy^\fp,s_c^\fp}=\fout{s_c^\fp}y_\omega^\fp\)
\end{itemize}
\end{definition}
\begin{definition}[\multiD Not vanishing gradient (\LTD)]
\label{defi:LTD_MD}
An \multiD cell \(c\) allows an \LTD \(:\Leftrightarrow\)\\
For arbitrary \(\fp_{in},\fp_{out}\in\IN^D,\fp_{in}\leq \fp_{out},\forall \delta>0\) there exist \(\forall\fp\in\IN^D\) gate activations \(\fy^\fp\) such that for any \(\fp_1,\fp_2\in\IN^D\)\begin{align}
\label{E:LTD_MD}
\frac{\partial s_c^{\fp_2}}{\partial \yin^{\fp_1}}&\intr\left\{ \begin{array}{ccl}
 \left[1-\delta,1\right]&\text{for}&\fp_1=\fp_{in}\text{ and }\fp_{in}\leq \fp_2\leq \fp_{out}\\
\left[0,\delta\right]&\text{otherwise}&\quad
\end{array}\right.
\end{align}
holds.
\end{definition}
\begin{definition}[\multiD Controllable output dependency (\COD)]
\label{defi:COG_MD}
  An \multiD cell \(c\) allows an \COD \(:\Leftrightarrow\)\\
There exist \(\delta_1,\delta_2\in(0,1),\delta_2<\delta_1\) so that for any time \(t\in\IN\) there exists a gate vector \(\fy^\fp\) leading to open output dependency
\begin{align}
\label{E:07_COG_1}
\frac{\partial y_c^\fp}{\partial s_c^\fp}\in\left[\delta_1,1\right]
\end{align}
and there exists another gate vector \(\fy^\fp\) leading to a closed output dependency
\begin{align}
\label{E:07_COG_2}
\frac{\partial y_c^\fp}{\partial s_c^\fp}\in\left[0,\delta_2\right].
\end{align}
\end{definition}

\begin{definition}[\multiD Not exploding gradient (\NGEC)]
\label{defi:not_growing_MD}
  An \multiD cell \(c\) has an \NGEC \(:\Leftrightarrow\)\\
For any time steps \(\fp_{in},\fp\in\IN^D, \fp_{in}< \fp\) and any gate activations \(\fy^\fp\) the truncated gradient in bounded by
\begin{align*}
\frac{\partial s_c^\fp}{\partial s_c^{\fp_{in}}}\intr\left[0,1\right]
\end{align*}
\end{definition}
We can now consider these definitions for the \multiD LSTM cell.
\begin{theorem}[\LTD of \multiD LSTM cells]
\label{theo:07_LTD_MD_LSTM}
   An \multiD LSTM cell allows an \LTD.
\end{theorem}
\begin{proof}
see \ref{proof:theo:07_LTD_MD_LSTM} in appendix   
\end{proof}
For arbitrary activations of FGs the partial \der \(\frac{\partial s_c^\fp}{\partial s_c^{\fp_{in}}}\) can grow over time:
\begin{theorem}[\NGEC of \multiD LSTM cells]
\label{theo:07_NGEC_MD_LSTM}
   An \multiD LSTM cell can have an exploding gradient, when \(D\geq 2\).
\end{theorem}
\begin{proof}
   see \ref{proof:theo:07_NGEC_MD_LSTM} in appendix.
\end{proof}
The \multiD LSTM cell does not allow the \COD, because the \oneD case is a special case of the \multiD case.\\
Our idea for the next section is to change the \multiD LSTM layout, so that it has an \NGEC.
\end{section}

\begin{section}{Reducing the \multiD LSTM Cell to One Dimension}
\label{S:Stable}
In the last section, we showed that the \multiD LSTM cell can have an exploding gradient. We tried different ways to solve this problem. For example we divided the activation of the FG by the number of dimensions. Then the gradient cannot explode over time, but the gradient vanishes along some paths rapidly. Another approach was to give the cells the opportunity to learn to stabilize itself, when the internal state starts diverging. Therefore we add an additional peephole connection between the square value of the previous internal states \(\left(s_c^{\fp_d^-}\right)^2\) and the FGs so that the cell is able to learn that it has to close the FG for large internal states. This also does not make a significant difference. Also forcing the cell to learn to stabilize itself by adding an error 
\begin{align*}
   Loss_{state}=\eps \left\|s_c^\fp\right\|_p
\end{align*}
with \(p=\{1,2,3,4\}\) and different learning rates \(\eps\) does not work. So we tried to change the layout of the \multiD LSTM cell.
\subsection{\multiD LSTM Stable Cell}
In Section \ref{S:cell_properties} we realized that \oneD LSTM cells work good and the gradient does not explode, but in the \multiD case it does. Our idea is to combine the previous states \(s_c^{\fp_d^-}\) at date \(\fp\) to one previous state \(s_c^{\fp^-}\) and take the \oneD form of the LSTM cell. For this reason we call this cell \emph{LSTM Stable cell}.\\
Therefore, a function 
\begin{align*}
s_c^{\fp^-}=f\left(s_c^{\fp_1^-},\ldots,s_c^{\fp_D^-}\right)
\end{align*}
is needed, so that the following two benefits of the \oneD LSTM cell remain:
\begin{enumerate}
   \item The \multiD LSTM Stable cell has an \NGEC
   \item The \multiD LSTM Stable cell allows \LTD.
\end{enumerate}
The convex combination
\begin{align}
\label{E:convexcombination_}
s_c^{\fp^-}=f\left(s_c^{\fp_1^-},\ldots,s_c^{\fp_D^-}\right)=\sum_{d=1}^D \lambda_d^\fp  s_c^{\fp_d^-}, \forall d=1,\ldots,D: \lambda_d^\fp \geq0 ,\sum_{d=1}^D\lambda_d^\fp =1
\end{align}
of all states satisfies these both points (see Theorems \ref{theo:MD_LTD_Stable} and \ref{theo:MD_NGEC_Stable}).
To calculate these \(D\) coefficients we want to use the activation of \(D\) gates and we call them lambda gates (LG or \(\lambda\)).
\begin{definition}[\multiD LSTM Stable cell]
   An \multiD LSTM Stable cell is a cell of dimension \(D\) and order \(1\) where \(\foutshort=\tanh\) and
\begin{itemize}
\item \(\Gamma=\left\{\iota,\left(\lambda,1\right),\ldots,\left(\lambda,D\right),\phi,\omega\right\}\)
\item \(s_c^{\fp^-}=\fconv{\fy^\fp,s_c^{\fp_1^-},\ldots,s_c^{\fp_D^-}}=\sum\limits_{d=1}^D s_c^{\fp_d^-}\frac{y_{\la,d}^\fp}{\sum\limits_{d'=1}^D y_{\la,d'}^\fp}\)
\item \(s_c^\fp=\frein{\fy^\fp,\yin^\fp,s_c^{\fp^-}}=y_\iota^\fp \yin^\fp +s_c^{\fp^-}y_\phi^\fp\)
\item \(y_c^\fp=\fraus{\fy^\fp,s_c^\fp}=y_\omega^\fp \fout{s_c^\fp}\)
\end{itemize}
\end{definition}
Using these equations we can test the cell for its properties. The \multiD LSTM Stable cell does not have the \COD, because the \oneD LSTM cell also does not have this property. For the other propertiese we get:
\begin{theorem}[LTD of \multiD LSTM Stable cells]
\label{theo:MD_LTD_Stable}
   An \multiD LSTM Stable cell allows \LTD.
\end{theorem}
\begin{proof}
   See \ref{proof:theo:MD_LTD_Stable} in appendix.
\end{proof}
\begin{theorem}[\NGEC of \multiD LSTM Stable cells]
\label{theo:MD_NGEC_Stable}
   An \multiD LSTM Stable cell has an \NGEC.
\end{theorem}
\begin{proof}
   See \ref{proof:theo:MD_NGEC_Stable} in appendix.
\end{proof}
\paragraph{Reducing the number of gates by one.} When \(D\geq2\) an \multiD LSTM Stable cell has one more gate  than a classical \multiD LSTM (for \(D=1\) the both cells are equivalent). But it is possible to reduce the number of LGs by one. One solution is to choose one dimension \(d'\in\left\{1,\ldots,D\right\}\) which does not get an LG. Its activation is calculated by
\begin{align*}
   y_{\la,d'}^\fp=\prod_{d\in\left\{1,\ldots,D\right\}\setminus\left\{d'\right\}}\left(1-y_{\la,d}^\fp\right).
\end{align*}
In the special case of \(D=2\) we can choose \(d'=2\) and we get \(\sum_{d'=1}^2 y_{\la,d'}=y_{\la,1}+\left(1-y_{\la,1}\right)=1\) and the update equation of the internal state can be simplified to 
\begin{align*}
s_c^\fp=\frein{y_\iota^\fp,y_{\la,1}^\fp,y_\phi^\fp,s_c^{\fp_1^-},s_c^{\fp_2^-}} &=y_\iota^\fp \yin^\fp +y_{\la,1}^\fp s_c^{\fp_1^-}+\left(1-y_{\la,1}^\fp\right)s_c^{\fp_2^-}.
\end{align*}
\end{section}

\begin{section}{Bounding the Internal State}
\label{S:leaky}
In the last sections we discussed the growing of the EC over time and we found a solution to have a NGEC for higher dimensions. Nevertheless it is possible that the internal state grows linearly over time. When we take a look at Definition \ref{defi:md_LSTM}, we see that the partial \der for \(\fp=\fp_{out}\) depends on \(\foutder{s_c^\fp}\). So having the inequality
\begin{align*}
   \frac{\partial y_c^\fp}{\partial s_c^{\fp}}\leq \foutder{s_c^\fp} \quad\text{ with }\quad \foutder{s_c^\fp}\xrightarrow{\left|s_c^\fp\right|\to \infty}0
\end{align*}
the cell allows \LTD defined in Definition \ref{defi:LTD_MD}, but actually we have \(\frac{\partial y_c^{\fp_{out}}}{\partial \yin^{\fp_{in}}}\xrightarrow{\left|s_c^{\fp_{out}}\right|\to \infty}0\) for arbitrary gate activations. Again, ideas like state decay, additional peephole connections or additional loss functions like mentioned in Section \ref{S:discussion} either do not work or destroy the \LTD of the LSTM and LSTM Stable cell. So, our solution is to change the architecture of the \multiD LSTM Stable cell, so that it fulfills has an \NGEC and allows \LTD and \COD. The key idea is to bound the internal state, so that for all inputs \(\left|\yin^\fp\right|\leq1\), \(\fp\in\IN^D\) the internal state is bounded by \(\left|s_c^\fp\right|\leq 1\).\\
Note that this is comparable with the well-known Bounded-Input-Bounded-Output-Stability (BIBO-Stability). To create an \multiD cell that has an \NGEC, allows \LTD and has a bounded internal state, we take the \multiD LSTM Stable cell proposed in the last section and change its layout.
Therefore we calculate the activation of the IG as function of the FG, so that we achieve \(\left|s_c^\fp\right|\leq 1\) by choosing \(y_\iota^\fp:=1-y_\phi^\fp\). 
So the activation of the FG controls how much leaks from the previous states. The activation of the FG can also be interpreted as switch, if the internal activation, the new activation or a convex combination of these both activations should be stored in the cell. So the \(s_c\) can be seen as \emph{time-dependent exponential moving average} of \(\yin\).
\begin{definition}[\multiD Leaky cell]
   An \multiD Leaky cell is a cell of dimension \(D\) and order \(1\) where \(\foutshort=\tanh\) and
\begin{itemize}
\item \(\Gamma=\left\{\left(\lambda,1\right),\ldots,\left(\lambda,D\right),\phi,\omega\right\}\)
\item \(s_c^{\fp^-}=\fconv{\fy^\fp,s_c^{\fp_1^-},\ldots,s_c^{\fp_D^-}}=\sum\limits_{d=1}^D s_c^{\fp_d^-}\frac{y_{\la,d}^\fp}{\sum\limits_{d'=1}^D y_{\la,d'}^\fp}\)
\item \(s_c^\fp=\frein{\fy^\fp,\yin^\fp,s_c^{\fp^-}}=\left(1-y_\phi^\fp\right) \yin^\fp +s_c^{\fp^-}y_\phi^\fp\)
\item \(y_c^\fp=\fraus{\fy^\fp,s_c^\fp}=y_\omega^\fp \fout{s_c^\fp}\)
\end{itemize}
\end{definition}
Now we can prove that the resulting cell has all benefits.
\begin{theorem}
\label{theo:11_Leaky}
   The \multiD Leaky cell has an \NGEC and allows \LTD and \COD.
\end{theorem}
\begin{proof}
   See \ref{proof:theo:11_Leaky} in appendix.
\end{proof}
The \multiD Leaky cell can have one gate less than the \multiD LSTM cell and the \multiD LSTM Stable cell and because of this, the update path requires less computations.
\end{section}

\begin{section}{General Derivation of Leaky Cells}
\label{S:13_leakyLP}
So far we proposed cells for the \multiD case, which are able to provide long term memory.
But especially in MDRNNs with more than one \multiD layer it is hard to measure if and how much long term dependencies are used and even if it is useful.
Another way to interpret the cell is to consider them as kind of \multiD feature extractor like ``feature maps'' in Convolutional Neural Networks \citep{be95}.
Then the aim is to construct an \multiD cell which is able to generate useful features.
Having a hierarchical Neural Network like in \citet{be95} and \citet{ag08} over the hierarchies the number of features increases with a simultaneously decreasing feature resolution.
Features in a layer with low resolution can be seen as low frequency features in comparison to features in a layer with high resolution.
So it would be useful to construct a cell as feature extractor which produces a low frequency output in comparison to its input.
In appendix \ref{ss:theory_1st_order} we take a closer look at the theory of linear shift invariant (LSI)-systems and their frequency analysis and analyse a first order LSI-system regarding its free selectable parameters using the \(\FF\)- and \(\ZZ\)-transform. There, we derive the \emph{\multiD LeakyLP cell} (see Definition \ref{defi:md_leakylp_cell}) and 5 additional first order \multiD cells, which we do not test in Section \ref{S:expriments}.
\begin{definition}[\multiD LeakyLP cell]
\label{defi:md_leakylp_cell}
   An \multiD LeakyLP cell is a cell of dimension \(D\) and order \(1\) where \(\foutshort=\tanh\) and
\begin{itemize}
\item \(\Gamma=\left\{\left(\lambda,1\right),\ldots,\left(\lambda,D\right),\phi,\omega_0,\omega_1\right\}\)
\item \(s_c^{\fp^-}=\fconv{\fy^\fp,s_c^{\fp_1^-},\ldots,s_c^{\fp_D^-}}=\sum\limits_{d=1}^D s_c^{\fp_d^-}\frac{y_{\la,d}^\fp}{\sum\limits_{d'=1}^D y_{\la,d'}^\fp}\)
\item \(s_c^\fp=\frein{\fy^\fp,\yin^\fp,s_c^{\fp^-}}=\left(1-y_\phi^\fp\right) \yin^\fp +s_c^{\fp^-}y_\phi^\fp\)
\item \(y_c^\fp=\fraus{\fy^\fp,s_c^\fp,s_c^{\fp^-}}=\fout{s_c^\fp y_{\omega_0}^\fp+s_c^{\fp^-} y_{\omega_1}^\fp }\)
\end{itemize}
\end{definition}
Setting the second OG (\(y_{\omega_1}^\fp\)) to zero, the LeakyLP cell corresponds to the Leaky cell, hence it fulfills all three properties, but has one more gate, which is as much gates as the LSTM cell.
\end{section}

\begin{section}{Experiments}
\label{S:expriments}
RNNs with \oneD LSTM cells are well studied. In some experiments the activations of the gates and the internal state are observed and one can see that the cell can really learn, when to ``forget'' information and when the internal state should be accessible for the network \citep[see][]{fg02}. However, we did not find experiments like these for the MD case and we do not want to transfer these experiments into the MD case. Instead we compare the different cell types with each other in two scenarios where the MD RNNs with LSTM cells perform very well. In both benchmarks the task is to transcribe a handwritten text on an image, so we have a 2D RNN. In this case we compare the cells on the IFN/ENIT \citep{pe02} and the Rimes database \citep{au06}. Both tasks are solved with the MD RNN layout described in \citet{ag08} and shown in Figure \ref{Fig:MDRNN}. All networks are trained with \emph{Backpropagation through time (BPTT)} .
\newcommand{\featZ}[3]
{
    \begin{scope}   
        
        \pgftransformcm{1}{0}{0.35}{0.5}{\pgfpoint{#1 cm}{#2 cm}}
        \node at (0,0) [circle,fill=black,minimum size = 2mm,#3!100] {};
        \node at (0,0.75) [circle,fill=black,minimum size = 2mm,#3!80] {};
        \draw [loosely dotted,very thick,#3!60] (0,1.25)--(0,2.25);
        \node at (0,2.75) [circle,fill=black,minimum size = 2mm,#3!40] {};            
	\end{scope}
}
\newcommand{\drawRec}[7]
{
    \begin{scope}   
    
     \begin{pgfscope}
		\pgftransformcm{1}{0}{0.35}{0.5}{\pgfpoint{#1 cm}{#2 cm}}    
    
		
		\node (h0) at (0, #5)[minimum size = 0mm] {};  
		
		\node (h1) at ($0*(h0)$)[minimum size = 0mm] {};    
		\node (h2) at ($1*(h0)$)[minimum size = 0mm] {};
		\node (h3) at ($3.6*(h0)$)[minimum size = 0mm] {};		
		\node (h4) at ($1.6*(h0)$)[minimum size = 0mm] {};	
		\node (h5) at ($3*(h0)$)[minimum size = 0mm] {};	
	\end{pgfscope}
		
	\begin{pgfscope}
		\draw [black!40] (h3) -- ($(h3) + (#3, 0)$);    
		\draw [black!40] ($(h3) + (#3,0)$) -- ($(h3) + (#3, -#4)$); 
		\draw [black!40] ($(h3) + (0,-#4)$) -- ($(h3) + (#3, -#4)$);
		\draw [black!40] (h3) -- ($(h3) + (0, -#4)$);
		
		\draw [loosely dotted,very thick,black!60] ($(h4) + (#3, -#4)$)--($(h5) + (#3, -#4)$);
		\draw [loosely dotted,very thick,black!60] ($(h4) + (#3, 0)$)--($(h5) + (#3, 0)$);
		\ifthenelse{\NOT 0 = #6}{
			\draw [loosely dotted,very thick,black!60] ($(h4) + (0, -#4)$)--($(h5) + (0, -#4)$);
			\draw [loosely dotted,very thick,black!60] ($(h4) + (0,0)$)--($(h5) + (0,0)$);
		}{}

		\draw [black!80] (h2) -- ($(h2) + (#3, 0)$);    
		\draw [black!80] ($(h2) + (#3,0)$) -- ($(h2) + (#3, -#4)$); 
		\draw [black!80] ($(h2) + (0,-#4)$) -- ($(h2) + (#3, -#4)$);
		\draw [black!80] (h2) -- ($(h2) + (0, -#4)$);

		\draw [black!100] (h1) -- ($(h1) + (#3, 0)$);    
		\draw [black!100] ($(h1) + (#3,0)$) -- ($(h1) + (#3, -#4)$); 
		\draw [black!100] ($(h1) + (0,-#4)$) -- ($(h1) + (#3, -#4)$);
		\draw [black!100] (h1) -- ($(h1) + (0, -#4)$);
		
		\ifthenelse{ 1 = #7}{
			\draw [->,thick] (h1)--($(h1) + (-#5, #5)$);
		}{}		
		\ifthenelse{ 2 = #7}{
			\draw [->,thick] ($(h1) + (0,-#4)$)--($(h1) + (-#5, -#4-#5)$);
		}{}	
		
		\ifthenelse{ 3 = #7}{
			\draw [->,thick] ($(h1) + (#3, -#4)$)--($(h1) + (#3+#5, -#4-#5)$);
		}{}	
			
		\ifthenelse{ 4 = #7}{
			\draw [->,thick] ($(h1) + (#3,0)$)--($(h1) + (#3+#5, #5)$);
		}{}

    \end{pgfscope}
	\end{scope}
}
\newcounter{posX}
\newcommand{\drawTanh}[3]
{
	\setcounter{posX}{#1}
	\pgfmathparse{\theposX -1.25}
	\drawRec{\pgfmathresult}{#2}{2.5}{0.5}{0.25}{1}{0}

		\draw [black!100,dotted] ($(#1-6.25,#2-0.85)$) -- ($(#1+8,#2-0.85)$);   
		\draw [black!100,dotted] ($(#1-6.25,#2-0.85)$) -- ($(#1-6.25,#2+0.7)$);  				
		\draw [black!100,dotted] ($(#1+8,#2+0.7)$) -- ($(#1+8,#2-0.85)$);  				 
		\draw [black!100,dotted] ($(#1+8,#2+0.7)$) -- ($(#1-6.25,#2+0.7)$);   
		\node (ta) at ($(#1+5.85,#2-0.65)$) [right] {\(0\)D LAYER #3};
}
\newcommand{\drawMDRNN}[3]
{
	
	\setcounter{posX}{#1}
	\pgfmathparse{\theposX -4.5-1.25}
	\drawRec{\pgfmathresult}{#2}{2.5}{0.5}{0.25}{1}{1}
	\setcounter{posX}{#1}
	\pgfmathparse{\theposX -1.5-1.25}
	\drawRec{\pgfmathresult}{#2}{2.5}{0.5}{0.25}{1}{2}
	\setcounter{posX}{#1}
	\pgfmathparse{\theposX 1.5-1.25}
	\drawRec{\pgfmathresult}{#2}{2.5}{0.5}{0.25}{1}{3}
	\setcounter{posX}{#1}
	\pgfmathparse{\theposX 4.5-1.25}
	\drawRec{\pgfmathresult}{#2}{2.5}{0.5}{0.25}{1}{4}

		\draw [black!100,dotted] ($(#1-6.25,#2-0.85)$) -- ($(#1+8,#2-0.85)$);   
		\draw [black!100,dotted] ($(#1-6.25,#2-0.85)$) -- ($(#1-6.25,#2+0.7)$);  				
		\draw [black!100,dotted] ($(#1+8,#2+0.7)$) -- ($(#1+8,#2-0.85)$);  				 
		\draw [black!100,dotted] ($(#1+8,#2+0.7)$) -- ($(#1-6.25,#2+0.7)$);   
		\node (ta) at ($(#1+5.85,#2-0.65)$) [right] {\(2\)D LAYER #3};
}
\newcommand{\drawOutput}[2]
{
\draw [black!100,dotted] ($(#1-2,#2-0.6)$) -- ($(#1+3.5,#2-0.6)$);   
\draw [black!100,dotted] ($(#1-2,#2-0.6)$) -- ($(#1-2,#2+0.6)$);  				
\draw [black!100,dotted] ($(#1+3.5,#2+0.6)$) -- ($(#1+3.5,#2-0.6)$);  				 
\draw [black!100,dotted] ($(#1+3.5,#2+0.6)$) -- ($(#1-2,#2+0.6)$);   
\node (ta) at ($(#1+2.25,#2-0.4)$) [right] {Output};
\node (h1) at (-1.5,#2) [minimum size=5mm,circle,fill=white,draw] {};
\node (h2) at (-0.75,#2) [minimum size=5mm,circle,fill=white,draw] {};
\draw [dotted,very thick] (0.0,#2)--(0.75,#2);
\node (h5) at (1.5,#2) [minimum size=5mm,circle,fill=white,draw] {};
\begin{scope}[on background layer]
\node (hA) [fit=(h1)(h2)(h5), minimum width=2cm, inner sep=2pt, fill=white,draw,rounded corners=5pt] {};
\end{scope}
}
\newcommand{\drawOutB}[4]
{
		\draw [black!100,dotted] ($(#1-6.25,#2-0.85)$) -- ($(#1+8,#2-0.85)$);   
		\draw [black!100,dotted] ($(#1-6.25,#2-0.85)$) -- ($(#1-6.25,#2+0.7)$);  				
		\draw [black!100,dotted] ($(#1+8,#2+0.7)$) -- ($(#1+8,#2-0.85)$);  				 
		\draw [black!100,dotted] ($(#1+8,#2+0.7)$) -- ($(#1-6.25,#2+0.7)$);   
		\node (ta) at ($(#1+6.05,#2-0.65)$) [right] {OUTPUT};
		\node (h1B) at ($(#3-1.5,#2+#4)$) [minimum size=5mm,circle,draw=black!40] {};
		\node (h2B) at ($(#3-0.75,#2+#4)$) [minimum size=5mm,circle,draw=black!40] {};
		\draw [dotted,very thick,black!40] ($(#3,#2+#4)$)--($(#3+0.75,#2+#4)$);
		\node (h5B) at ($(#3+1.5,#2+#4)$) [minimum size=5mm,circle,draw=black!40] {};
		\begin{scope}[on background layer]
		\node (hAB) [fit=(h1B)(h2B)(h5B), minimum width=2cm, inner sep=2pt,rounded corners=5pt,draw=black!40] {};
		\end{scope}	
		\node (h1) at (-1.5,#2) [minimum size=5mm,circle,fill=white,draw] {};
		\node (h2) at (-0.75,#2) [minimum size=5mm,circle,fill=white,draw] {};
		\draw [dotted,very thick] (0.0,#2)--(0.75,#2);
		\node (h5) at (1.5,#2) [minimum size=5mm,circle,fill=white,draw] {};
		\begin{scope}[on background layer]
		\node (hA) [fit=(h1)(h2)(h5), minimum width=2cm, inner sep=2pt,draw,rounded corners=5pt] {};
		\end{scope}		
		\draw [loosely dotted,very thick,black!60] ($(hAB) + (-1.85, 0.35)$)--($(hA) + (-1.85, 0.35)$);
		\draw [loosely dotted,very thick,black!60] ($(hAB) + (1.75, 0.3)$)--($(hA) + (1.75, 0.3)$);
		\draw [loosely dotted,very thick,black!60] ($(hAB) + (-1.85, -0.35)$)--($(hA) + (-1.85, -0.35)$);
		\draw [loosely dotted,very thick,black!60] ($(hAB) + (1.75, -0.3)$)--($(hA) + (1.75, -0.3)$);
		\node (h11) at (-1.5,#2) [minimum size=5mm,circle,fill=white,draw] {};
		\node (h21) at (-0.75,#2) [minimum size=5mm,circle,fill=white,draw] {};
		\draw [dotted,very thick] (0.0,#2)--(0.75,#2);
		\node (h51) at (1.5,#2) [minimum size=5mm,circle,fill=white,draw] {};		
}
\newcommand{\drawOutC}[4]
{
		\draw [black!100,dotted] ($(#1-6.25,#2-0.65)$) -- ($(#1+8,#2-0.65)$);   
		\draw [black!100,dotted] ($(#1-6.25,#2-0.65)$) -- ($(#1-6.25,#2+1.1)$);  				
		\draw [black!100,dotted] ($(#1+8,#2+1.1)$) -- ($(#1+8,#2-0.65)$);  				 
		\draw [black!100,dotted] ($(#1+8,#2+1.1)$) -- ($(#1-6.25,#2+1.1)$);   
		\node (ta) at ($(#1+4.95,#2-0.45)$) [right] {OUTPUT LAYER};	
		\draw [very thick,black] ($(#1,#2) + (-1.5, -0.25)$)--($(#1,#2) + (1.5, -0.25)$);	
		\draw [very thick,black!80] ($(#1+#3,#2+#4) + (-1.5, -0.25)$)--($(#1+#3,#2+#4) + (1.5, -0.25)$);
		\draw [loosely dotted,very thick,black!60] ($(#1+1.5*#3,#2+1.5*#4) + (-1.5, -0.25)$)--($(#1+3.5*#3,#2+3.5*#4) + (-1.5, -0.25)$);
		\draw [loosely dotted,very thick,black!60] ($(#1+1.5*#3,#2+1.5*#4) + (1.5, -0.25)$)--($(#1+3.5*#3,#2+3.5*#4) + (1.5, -0.25)$);
		\draw [very thick,black!40] ($(#1+4*#3,#2+4*#4) + (-1.5, -0.25)$)--($(#1+4*#3,#2+4*#4) + (1.5, -0.25)$);		
}
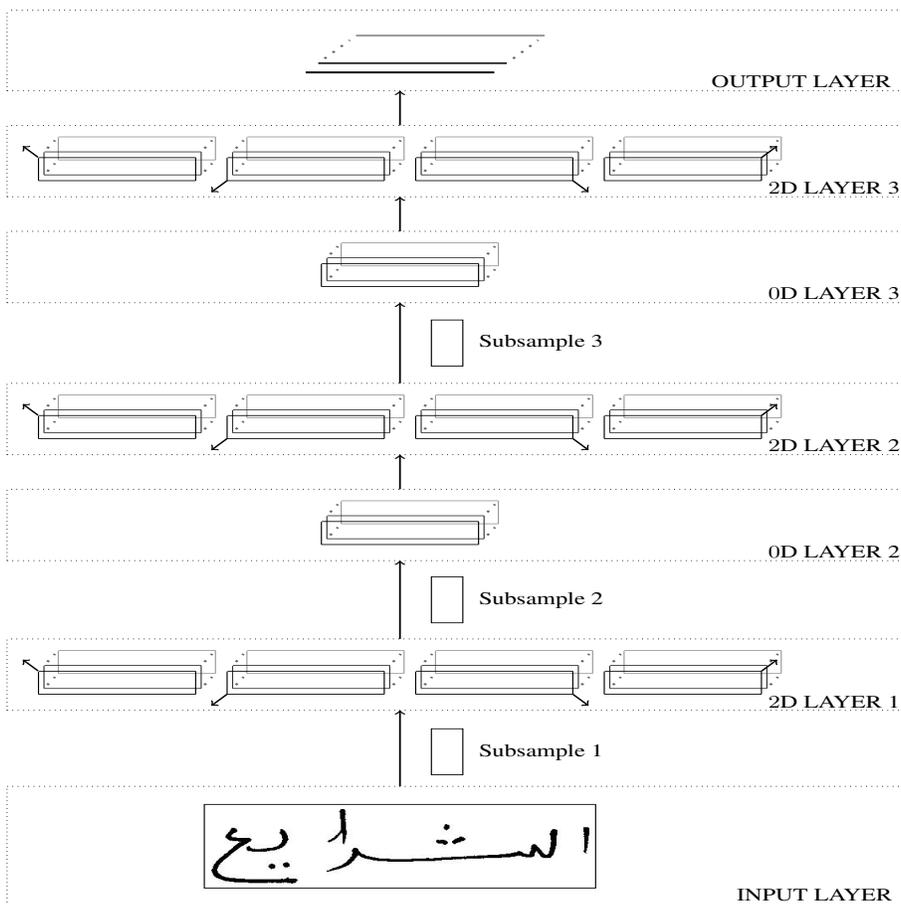
\begin{figure}[ht]\label{fig:hybrid}
\centering
\resizebox{12cm}{12cm}{%
\begin{tikzpicture}[minimum size=7.5mm,inner sep=0pt]
\node (pic) at (0,-0.2) {\framebox{\includegraphics[width=6cm]{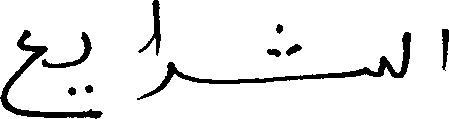}}};
		\draw [black!100,dotted] (-6.25,-1.5) -- (8,-1.5);   
		\draw [black!100,dotted] (-6.25, 1.1) -- (8,1.1);   
		\draw [black!100,dotted] (-6.25,-1.5) -- (-6.25,1.1);   
		\draw [black!100,dotted] (8,-1.5) -- (8,1.1);   
		\node (ta) at (5.35,-1.25) [right] {INPUT LAYER};	
\drawMDRNN{0}{3.6}{1}
\drawTanh{0}{6.85}{2}
\drawMDRNN{0}{9.15}{2}
\drawTanh{0}{12.45}{3}
\drawMDRNN{0}{14.75}{3}
\drawOutC{0}{16.85}{0.2}{0.2}
\draw [->,thick] (0,1.1) -- (0,2.75);
\node (s1) at (0.75,1.85) [rectangle,fill=white,draw,minimum width=0.5cm, minimum height=1.0cm] {};
\node (s1t) at (1.25,1.85) [right] {Subsample 1};	
\draw [->,thick] (0,4.3) -- (0,6.0);
\node (s2) at (0.75,5.15) [rectangle,fill=white,draw,minimum width=0.5cm, minimum height=1.0cm] {};
\node (s2t) at (1.25,5.15) [right] {Subsample 2};	
\draw [->,thick] (0,7.55) -- (0,8.3);
\draw [->,thick] (0,9.85) -- (0,11.6);
\node (s3) at (0.75,10.725) [rectangle,fill=white,draw,minimum width=0.5cm, minimum height=1.0cm] {};
\node (s3t) at (1.25,10.725) [right] {Subsample 3};	
\draw [->,thick] (0,13.15) -- (0,13.9);
\draw [->,thick] (0,15.45) -- (0,16.2);
\end{tikzpicture}
}\caption{Architecture of the hierarchical MDRNN used for the experiments: It is equivalent to \citet[Figure 2]{ag08}. A \(2\)D layer contains \(2^2\) distinct layers (for each combination of scanning direction left/right and up/down one layer). To reduce the number of weights between two \(2\)D layers, a \(0\)D layer is inserted, which contains units with \(\tanh\) as activation function. They have dimension \(0\) because they have no recurrent connections. These layers can be seen as feed-forward or convolutional layer. Each \(2\)D layer (or its allocated \(0\)D layer) reduces its size in in x and y dimension using a two-dimensional subsampling. Simultaneously the number of feature maps (z-dimension) increases to have no bottleneck between input and output layer.}
\label{Fig:MDRNN}
\end{figure}
To compare the different cell types in RNNs with each other we take \(10\) RNNs with different weight initializations of each cell type and calculate the minimum, the maximum and the median of the best label error rate (LER) on a validation set of these \(10\) RNNs. In all tables we present these three LERs to compare the cell types.\\
We think it is more important to have stable cells in the lower MD layers because of two reasons: First, when we have just a few cells in a layer, the saturation of one cell has a greater effect on the performance of the network. Second, in lower layers there are longer time series so having an unstable cell in such a layer, it has time to saturate. So our first experiment compares the recognition results when we substitute the LSTM cells in the lowest layer (which is ``2D layer 1'' in Figure \ref{Fig:MDRNN}) by the newly developed cells.\\
In the second experiment we compare the LSTM cell and the LeakyLP cell also in the higher \multiD layers (``2D layer 2 and 3 in Figure \ref{Fig:MDRNN}), to evaluate if the LeakyLP cell work better also in long time series.\\
In \citet[3.1.1]{jb12} it is mentioned, that an important hyper parameter for a training is the learning rate, so another experiment is to train all networks with stochastic gradient decent with different learning rates \(\delta\in\left\{1\cdot10^{-3},5\cdot10^{-4},2\cdot10^{-4},1\cdot10^{-4}\right\}\) and compare the best LER according a fixed learning rate.
\subsection{The IFN/ENIT Database}
\label{S:IFN_ENIT_database}
This database contains handwritten names of towns and villages of Tunisia. The set is divided into 7 (a-f,s) sets, where 5 (a-e) are available for training and validation \citep[for details see][]{pe02}. With all information we got from A. Graves, we were able to get comparable results to \citet{ag08}. Therefor we divide the sets a-e into \(30000\) training samples and \(2493\) validation samples. All network are trained 100 epochs with a fixed learning rate \(\delta=1\cdot10^{-4}\). The LER is calculated on the validation set.
\subsubsection{Different Cells in the Lowest MD Layer}
In our first experiment we substitute the LSTM cell in the lowest MD layer. We take some of the cells described in this paper. In Table \ref{Tabular:arabic_firstlayer} the results are shown. The first row is the same RNN layout used in \citet{ag08}.
\begin{table}[t]
\centering
\begin{tabular}{llll}\hline\hline
 & \mc{3}{>{\columncolor{colorb}}c}{Label-Error-Rate in Percent}\\
\mc{1}{>{\columncolor{colorb}}c}{Celltype}& \mc{1}{>{\columncolor{colorb}}c}{min   } & \mc{1}{>{\columncolor{colorb}}c}{max    } & \mc{1}{>{\columncolor{colorb}}c}{median }\\\hline
\mc{1}{>{\columncolor{colorb}}r}{LSTM }   & \mc{1}{>{\columncolor{colorb}}r}{8,58\%} & \mc{1}{>{\columncolor{colorb}}r}{14,73\%} & \mc{1}{>{\columncolor{colorb}}r}{10,58\%}\\\hline
\mc{1}{>{\columncolor{colorb}}r}{Stable}  & \mc{1}{>{\columncolor{colorb}}r}{8,78\%} & \mc{1}{>{\columncolor{colorb}}r}{11,75\%} & \mc{1}{>{\columncolor{colorb}}r}{ 9,55\%}\\
\mc{1}{>{\columncolor{colorb}}r}{Leaky }  & \mc{1}{>{\columncolor{colorb}}r}{8,87\%} & \mc{1}{>{\columncolor{colorb}}r}{10,47\%} & \mc{1}{>{\columncolor{colorb}}r}{ 9,10\%}\\
\mc{1}{>{\columncolor{colorb}}r}{LeakyLP} & \mc{1}{>{\columncolor{color1}}r}{8,24\%} & \mc{1}{>{\columncolor{color1}}r}{ 9,40\%} & \mc{1}{>{\columncolor{color1}}r}{ 8,93\%}\\\hline\hline
\end{tabular}
\caption{Different cell types in the lowest MD layer}
\label{Tabular:arabic_firstlayer}
\end{table}
We can see, that the LeakyLP cell performs the best. Nevertheless the worst RNN with LeakyLP cells in the lowest MD layer performs worth than the best RNN with LSTM cells. So we cannot say, that LeakyLP is always better. But it can be observed that the variance of the RNN performance is very high with LSTM cells in the lowest MD layer. Our interpretation is that LSTM cells have a comparable performance like the LeakyLP cells in the lowest layer, when they do not saturate. Note, that the Leaky cell has one gate less, so they are faster and have less trainable weights. 
\subsubsection{Different Cells in Other MD Layers} Now we want to compare the best new developed cell---the LeakyLP cell---with the LSTM cell in the other MD layers. So we also substitute the LSTM cell in the upper MD layers. We enumerate the \(2\)D layers like shown in Figure \ref{Fig:MDRNN}.
\begin{table}
\centering
\begin{tabular}{cccccc}\hline\hline
\mc{3}{>{\columncolor{colorb}}c}{Celltype in 2D layer}& \mc{3}{>{\columncolor{colorb}}c}{Label-Error-Rate in Percent}\\
1 & 2 & 3   & \mc{1}{>{\columncolor{colorb}}c}{min   } & \mc{1}{>{\columncolor{colorb}}c}{max    } & \mc{1}{>{\columncolor{colorb}}c}{median }\\\hline
LSTM & LSTM & LSTM   & \mc{1}{>{\columncolor{colorb}}r}{8,58\%} & \mc{1}{>{\columncolor{colorb}}r}{14,73\%} & \mc{1}{>{\columncolor{colorb}}r}{10,58\%}\\
LeakyLP & LSTM & LSTM   & \mc{1}{>{\columncolor{color1}}r}{8,24\%} & \mc{1}{>{\columncolor{color1}}r}{ 9,40\%} & \mc{1}{>{\columncolor{colorb}}r}{ 8,93\%}\\
LeakyLP & LeakyLP & LSTM & \mc{1}{>{\columncolor{colorb}}r}{8,35\%} & \mc{1}{>{\columncolor{colorb}}r}{11,27\%} & \mc{1}{>{\columncolor{color1}}r}{ 8,91\%}\\
LeakyLP & LeakyLP & LeakyLP &  \mc{1}{>{\columncolor{colorb}}r}{8,92\%} & \mc{1}{>{\columncolor{colorb}}r}{11,69\%} & \mc{1}{>{\columncolor{colorb}}r}{ 9,74\%}\\\hline\hline
\end{tabular}
\caption{Different cells in other layers}
\label{Tabular:arabic_otherlayers}
  \end{table}
In Table \ref{Tabular:arabic_otherlayers} we can see that substituting the LSTM cells only in the lowest or in the both lowest layer perform slightly better. The best results can be achieved when we use LeakyLP cells in \(2\)D layer 1 and LSTM cells in \(2\)D layer 3. Using LSTM in the middle layer seems to work slightly better than using the LeakyLP cells instead. This fits to our intuition mentioned before that the LSTM cells perform better when they do not have a too long time series and when there are enough cells in one layer which do not saturate.
\subsubsection{Performance of Cells Regarding Learning-Rate} When we take a look at the update equations and the proofs of the \NGEC it can be assumed, that the gradient going through the cells is lower for LeakyLP cells in contrast to LSTM cells. So we think the learning rate have to be larger for LeakyLP cells. In Table \ref{Tabular:arabic_learningrate} we compare the networks with either only LSTM or LeakyLP cells. There we can see that the learning rate have to be much higher for the LeakyLP cells. In addition, the RNNs with LeakyLP cells are more robust to the choice of the learning rate.  
\begin{table}[t]
\centering
\begin{tabular}{lllll}\hline\hline
 & & \mc{3}{>{\columncolor{colorb}}c}{Label-Error-Rate in Percent}\\
\mc{1}{>{\columncolor{colorb}}c}{Celltype}    & \mc{1}{>{\columncolor{colorb}}c}{BP-delta}    & \mc{1}{>{\columncolor{colorb}}c}{min   } & \mc{1}{>{\columncolor{colorb}}c}{max    } & \mc{1}{>{\columncolor{colorb}}c}{median }\\\hline
\mc{1}{>{\columncolor{colorb}}c}{LSTM   } &\(1\cdot10^{-4}\) & \mc{1}{>{\columncolor{color1}}r}{ 8,58\%} & \mc{1}{>{\columncolor{color1}}r}{14,73\%} & \mc{1}{>{\columncolor{colorb}}r}{10,58\%}\\
\mc{1}{>{\columncolor{colorb}}c}{LSTM   } &\(2\cdot10^{-4}\) & \mc{1}{>{\columncolor{colorb}}r}{ 9,15\%} & \mc{1}{>{\columncolor{colorb}}r}{16,86\%} & \mc{1}{>{\columncolor{color1}}r}{10,51\%}\\
\mc{1}{>{\columncolor{colorb}}c}{LSTM   } &\(5\cdot10^{-4}\) & \mc{1}{>{\columncolor{colorb}}r}{ 9,03\%} & \mc{1}{>{\columncolor{colorb}}r}{21,77\%} & \mc{1}{>{\columncolor{colorb}}r}{11,44\%}\\
\mc{1}{>{\columncolor{colorb}}c}{LSTM   } &\(1\cdot10^{-3}\) & \mc{1}{>{\columncolor{colorb}}r}{10,21\%} & \mc{1}{>{\columncolor{colorb}}r}{30,20\%} & \mc{1}{>{\columncolor{colorb}}r}{11,44\%}\\
\hline
\mc{1}{>{\columncolor{colorb}}c}{LeakyLP} &\(1\cdot10^{-4}\) & \mc{1}{>{\columncolor{colorb}}r}{ 8,92\%} & \mc{1}{>{\columncolor{colorb}}r}{11,69\%} & \mc{1}{>{\columncolor{colorb}}r}{ 9,74\%}\\
\mc{1}{>{\columncolor{colorb}}c}{LeakyLP} &\(2\cdot10^{-4}\) & \mc{1}{>{\columncolor{colorb}}r}{ 8,38\%} & \mc{1}{>{\columncolor{colorb}}r}{ 9,09\%} & \mc{1}{>{\columncolor{colorb}}r}{ 8,81\%}\\
\mc{1}{>{\columncolor{colorb}}c}{LeakyLP} &\(5\cdot10^{-4}\) & \mc{1}{>{\columncolor{color1}}r}{ 8,25\%} & \mc{1}{>{\columncolor{color1}}r}{ 8,95\%} & \mc{1}{>{\columncolor{color1}}r}{ 8,78\%}\\
\mc{1}{>{\columncolor{colorb}}c}{LeakyLP} &\(1\cdot10^{-3}\) & \mc{1}{>{\columncolor{colorb}}r}{ 8,29\%} & \mc{1}{>{\columncolor{colorb}}r}{ 9,20\%} & \mc{1}{>{\columncolor{colorb}}r}{ 8,88\%}\\
\mc{1}{>{\columncolor{colorb}}c}{LeakyLP} &\(2\cdot10^{-3}\) & \mc{1}{>{\columncolor{colorb}}r}{ 8,95\%} & \mc{1}{>{\columncolor{colorb}}r}{12,81\%} & \mc{1}{>{\columncolor{colorb}}r}{ 9,55\%}\\
\hline\hline
\end{tabular}
\caption{Performance of cells regarding learning-rate}
\label{Tabular:arabic_learningrate}
  \end{table}
\subsection{The Rimes Database}
One task of the Rimes database is the handwritten word recognition \citep[for more details see][]{grosicki08,grosicki11}. It contains \(59292\) images of french single words. It is divided into distinct subsets; a training set of \(44196\) samples, a validation set of \(7542\) samples and a test set of \(7464\) samples. We train the MD RNNs by using the training set for training and calculate the LER over the validation set, so the network is trained on \(44196\) training samples each epoch. The network used in this section differs only in the subsampling rate between two layers from the network used in \citet{ag08}. When there is a subsampling between layers, the factors are \(3\times2\) instead of \(4\times3\) or \(4\times2\). The rest of the experiment is the same like described in Section \ref{S:IFN_ENIT_database}.
\subsubsection{Different Cells in the Lowest MD Layer}
In Table \ref{Tabular:french_firstlayer} we can see that substituting the LSTM in the lowest layer by one of the three cells improves the performance of the network, even the Leaky cell with one gate less. 
\begin{table}[t]
\centering
\begin{tabular}{llll}\hline\hline
 & \mc{3}{>{\columncolor{colorb}}c}{Label-Error-Rate in Percent}\\
\mc{1}{>{\columncolor{colorb}}c}{Celltype}& \mc{1}{>{\columncolor{colorb}}c}{min   } & \mc{1}{>{\columncolor{colorb}}c}{max    } & \mc{1}{>{\columncolor{colorb}}c}{median }\\\hline
\mc{1}{>{\columncolor{colorb}}r}{LSTM }   & \mc{1}{>{\columncolor{colorb}}r}{14,96\%} & \mc{1}{>{\columncolor{colorb}}r}{17,63\%} & \mc{1}{>{\columncolor{colorb}}r}{16,50\%}\\\hline
\mc{1}{>{\columncolor{colorb}}r}{Stable}  & \mc{1}{>{\columncolor{color1}}r}{14,45\%} & \mc{1}{>{\columncolor{colorb}}r}{16,02\%} & \mc{1}{>{\columncolor{color1}}r}{ 15,11\%}\\
\mc{1}{>{\columncolor{colorb}}r}{Leaky }  & \mc{1}{>{\columncolor{colorb}}r}{14,77\%} & \mc{1}{>{\columncolor{colorb}}r}{16,39\%} & \mc{1}{>{\columncolor{colorb}}r}{ 15,85\%}\\
\mc{1}{>{\columncolor{colorb}}r}{LeakyLP} & \mc{1}{>{\columncolor{colorb}}r}{14,63\%} & \mc{1}{>{\columncolor{color1}}r}{15,78\%} & \mc{1}{>{\columncolor{colorb}}r}{ 15,30\%}\\\hline\hline
\end{tabular}
\caption{Different cell types in the lowest MD layer}
\label{Tabular:french_firstlayer}
\end{table}
\subsubsection{Different Cells in Other MD Layers} We want to see the effect of the substitution of the LSTM cell by the LeakyLP cell in the upper MD layers. In Table \ref{Tabular:french_otherlayers} we can see that using LeakyLP cells in both lowest layers perform very well. So we also take this setup to try different learning rates.
\begin{table}[t]
\centering
\begin{tabular}{cccccc}\hline\hline
\mc{3}{>{\columncolor{colorb}}c}{Celltype in 2D layer}& \mc{3}{>{\columncolor{colorb}}c}{Label-Error-Rate in Percent}\\
1 & 2 & 3   & \mc{1}{>{\columncolor{colorb}}c}{min   } & \mc{1}{>{\columncolor{colorb}}c}{max    } & \mc{1}{>{\columncolor{colorb}}c}{median }\\\hline
LSTM & LSTM & LSTM   & \mc{1}{>{\columncolor{colorb}}r}{14,96\%} & \mc{1}{>{\columncolor{colorb}}r}{17,63\%} & \mc{1}{>{\columncolor{colorb}}r}{16,50\%}\\\hline
LeakyLP & LSTM & LSTM   & \mc{1}{>{\columncolor{colorb}}r}{14,63\%} & \mc{1}{>{\columncolor{colorb}}r}{15,78\%} & \mc{1}{>{\columncolor{colorb}}r}{ 15,30\%}\\
LeakyLP & LeakyLP & LSTM & \mc{1}{>{\columncolor{color1}}r}{14,21\%} & \mc{1}{>{\columncolor{color1}}r}{15,57\%} & \mc{1}{>{\columncolor{color1}}r}{14,92\%}\\
LeakyLP & LeakyLP & LeakyLP &  \mc{1}{>{\columncolor{colorb}}r}{14,94\%} & \mc{1}{>{\columncolor{colorb}}r}{16,18\%} & \mc{1}{>{\columncolor{colorb}}r}{15,52\%}\\\hline\hline
\end{tabular}
\caption{Different cells in other layers}
\label{Tabular:french_otherlayers}
  \end{table}
\paragraph{Performance of Cells Regarding Learning-Rate.}
Using different learning rates we can see that the RNN with LeakyLP cells in the both lowest layers and the LSTM cells in the top layer can significantly improve the performance . Even the maximal LER of this RNN works better than the best network with LSTM cells in each layer.
\begin{table}[t]
\centering
\begin{tabular}{lllll}\hline\hline
 & & \mc{3}{>{\columncolor{colorb}}c}{Label-Error-Rate in Percent}\\
\mc{1}{>{\columncolor{colorb}}c}{Celltype}    & \mc{1}{>{\columncolor{colorb}}c}{BP-delta}    & \mc{1}{>{\columncolor{colorb}}c}{min   } & \mc{1}{>{\columncolor{colorb}}c}{max    } & \mc{1}{>{\columncolor{colorb}}c}{median }\\\hline
\mc{1}{>{\columncolor{colorb}}c}{LSTM   } &\(1\cdot10^{-4}\) & \mc{1}{>{\columncolor{colorb}}r}{14,96\%} & \mc{1}{>{\columncolor{colorb}}r}{17,63\%} & \mc{1}{>{\columncolor{colorb}}r}{16,50\%}\\
\mc{1}{>{\columncolor{colorb}}c}{LSTM   } &\(2\cdot10^{-4}\) & \mc{1}{>{\columncolor{color1}}r}{14,41\%} & \mc{1}{>{\columncolor{colorb}}r}{16,88\%} & \mc{1}{>{\columncolor{colorb}}r}{15,61\%}\\
\mc{1}{>{\columncolor{colorb}}c}{LSTM   } &\(5\cdot10^{-4}\) & \mc{1}{>{\columncolor{colorb}}r}{15,05\%} & \mc{1}{>{\columncolor{color1}}r}{16,27\%} & \mc{1}{>{\columncolor{color1}}r}{15,47\%}\\
\hline
\mc{1}{>{\columncolor{colorb}}c}{LeakyLP} &\(1\cdot10^{-4}\) & \mc{1}{>{\columncolor{colorb}}r}{14,94\%} & \mc{1}{>{\columncolor{colorb}}r}{16,18\%} & \mc{1}{>{\columncolor{colorb}}r}{15,52\%}\\
\mc{1}{>{\columncolor{colorb}}c}{LeakyLP} &\(5\cdot10^{-4}\) & \mc{1}{>{\columncolor{color1}}r}{12,68\%} & \mc{1}{>{\columncolor{color1}}r}{13,95\%} & \mc{1}{>{\columncolor{color1}}r}{13,57\%}\\
\hline
\mc{1}{>{\columncolor{colorb}}c}{LeakyLP in 2D layer 1 \& 2} &\(2\cdot10^{-4}\) & \mc{1}{>{\columncolor{colorb}}r}{13,26\%} & \mc{1}{>{\columncolor{colorb}}r}{14,04\%} & \mc{1}{>{\columncolor{colorb}}r}{13,65\%}\\
\mc{1}{>{\columncolor{colorb}}c}{LeakyLP in 2D layer 1 \& 2} &\(5\cdot10^{-4}\) & \mc{1}{>{\columncolor{color1}}r}{12,08\%} & \mc{1}{>{\columncolor{color1}}r}{13,42\%} & \mc{1}{>{\columncolor{color1}}r}{12,87\%}\\
\hline\hline
\end{tabular}
\caption{Performance of cells regarding learning-rate}
\label{Tabular:french_learningrate}
\end{table}
\end{section}

\begin{section}{Conclusion}
In this paper we took a look at the one-dimensional LSTM cell and discussed the benefits of this cell. We found two properties, that probably make these cells so powerful in the one dimensional case. Expanding these properties to the multi dimensional case, we saw that the LSTM does not fulfill one of these properties any more. We solved this problem by changing the architecture of the cell. In addition we presented a more general idea how to create one dimensional or multi dimensional cells. We compare some newly developed cells with the LSTM cell on two data sets and we can improve the performance using the new cell types. Due to this we think that substituting the multi-dimensional LSTM cells by the multi-dimensional \emph{LeakyLP cell} could improve the performance of many system working with a multi-dimensional space.
\end{section}

\appendix

\begin{section}{Proofs}
\subsection{Proof of \ref{theo:ltd_1D}}
\label{proof:theo:ltd_1D}
\begin{proof}
   Let \(c\) be a \oneD LSTM cell. To get the \der \(\frac{\partial s_c\left(t_2\right)}{\partial s_c\left(t_{1}\right)}\) according the truncated gradient between two time steps \(t_1,t_2\in\IN\) we have to take a look at \(\frein{}\).
\begin{align}
\label{E:05_LTD_1}
   \frac{\partial s_c\left(t_2\right)}{\partial s_c\left(t_{1}\right)}&=\frac{\partial \frein{\fy(t_2),\yin(t_2),s_c(t_2-1)}}{\partial s_c\left(t_{1}\right)}\\&=\underbrace{\frac{\partial \left(\yin(t_2)y_\iota(t_2)\right)}{\partial s_c\left(t_{1}\right)}}_{\eqtr0}+\frac{\partial s_c(t_2-1)}{\partial s_c\left(t_{1}\right)}y_\phi(t_2)+s_c(t_2-1)\underbrace{\frac{\partial y_\phi(t_2)}{\partial s_c\left(t_{1}\right)}}_{\eqtr0}\notag\\
&\eqtr \frac{\partial s_c(t_2-1)}{\partial s_c\left(t_{1}\right)}y_\phi(t_2)\notag\\
\label{E:05_LTD_2}
&\eqtr \prod_{t=t_{1}+1}^{t_2} y_\phi(t)
\end{align}
In addition, \(\forall t\in\IN\) we have
\begin{align}
\label{E:05_LSTM_derivation}
   \frac{\partial s_c\left(t\right)}{\partial \yin\left(t\right)}=y_\iota(t)\qquad\text{and}\qquad\frac{\partial y_c\left(t\right)}{\partial s_c\left(t\right)}=\foutder{s_c(t)}y_\omega(t).
\end{align}
We will prove the properties successively.\\
{\bf \NGEC:}
For the LSTM cell the FG \(f_\phi=\fl\) ensures \(y_\phi(t)\in\left(0,1\right)\), so using these bounds in \eqref{E:05_LTD_2} with
\begin{align*}
   \frac{\partial s_c\left(t\right)}{\partial s_c\left(t_{in}\right)}\eqtr \prod_{t'=t_{in}+1}^{t} y_\phi(t')\in\left(0,1\right)
\end{align*}
the LSTM cell has an \NGEC.\\
{\bf \LTD:} Therefore, we choose
\begin{align*}
y_\iota\left(t\right)&\in\left\{ \begin{array}{ccl} [1-\eps,1) &\text{ if } &t=t_{in} \\ (0,\eps]&\text{otherwise}& \end{array} \right.,\\
y_\phi\left(t\right)&\in\left\{ \begin{array}{ccl} [1-\eps,1) &\text{ if }& t_{in}<t\leq t_{out} \\ (0,\eps]&\text{otherwise}& \end{array} \right.,\\
\end{align*}
with a later chosen \(\eps>0\). Let \(t_1,t_2\in\IN,t_1\leq t_2\) be two arbitrary dates, where we want to calculate the gradient \(\frac{\partial s_c\left(t_2\right)}{\partial \yin\left(t_1\right)}\). First, we want to show that the LSTM cell allows \LTD for \(t_1=t_{in}\text{ and } t_{in}\leq t_2\leq t_{out}\):\\
We have \(y_\iota(t_1)\in[1-\eps,1)\) and \(\forall t=t_{in}+1,\ldots,t_{out}:y_\phi(t)\in[1-\eps,1)\).
Then, we can estimate the \der from \eqref{E:05_LTD_1} and \eqref{E:05_LSTM_derivation} by
\begin{align*}
\frac{\partial s_c\left(t_2\right)}{\partial \yin\left(t_1\right)}=
\frac{\partial s_c\left(t_2\right)}{\partial s_c\left(t_1\right)}
\frac{\partial s_c\left(t_1\right)}{\partial \yin\left(t_1\right)}
&\eqtr y_\iota\left(t_{1}\right)\prod_{t=t_{1}+1}^{t_2} y_\phi\left(t\right)\\
&\intr\Bigg[\left(1-\eps\right)\prod_{t=t_{1}+1}^{t_2}\left(1-\eps\right),1\Bigg)\subseteq\Bigg[\left(1-\eps\right)^{t_{out}-t_{in}+1},1\Bigg).
\end{align*}
To fulfill the equation for \LTD we choose \(\eps\) depending on \(\delta\) such that
\begin{align*}
   1-\delta&\leq\left(1-\eps\right)^{t_{out}-t_{in}+1}\\
\Leftrightarrow\quad \eps&\leq 1-\left(1-\delta\right)^{\frac{1}{t_{out}-t_{in}+1}}
\end{align*}
holds. Second, we have to show, that the \der is in \([0,\delta]\), when \(t_1=t_{in}\text{ and } t_{in}\leq t_2\leq t_{out}\) is not fulfilled.\\
In the case of \(t_1\neq t_{in}\) when \(\eps\leq\delta\) we can use the \NGEC which leads to
\begin{align*}
\frac{\partial s_c\left(t_2\right)}{\partial \yin\left(t_1\right)}=
\underbrace{\frac{\partial s_c\left(t_2\right)}{\partial s_c\left(t_1\right)}}_{\intr [0,1]}
\underbrace{\frac{\partial s_c\left(t_1\right)}{\partial \yin\left(t_1\right)}}_{\in (0,\eps]}
\subseteq [0,\eps]\subseteq [0,\delta].
\end{align*}
When \(t_1=t_{in}\) we have two cases: \(t_2<t_{in}\) or \(t_2>t_{out}\).
For the case \(t_2<t_{in}\) the \der is zero \((\subset [0,\delta])\), because the cell is causal. For \(t_2>t_{out}\) we can split the \der at \(t_{out}\) and get
\begin{align*}
\frac{\partial s_c\left(t_2\right)}{\partial \yin\left(t_1\right)}&\eqtr
\underbrace{y_\iota(t_1) \prod_{t=t_1+1}^{t_{out}} y_\phi(t)}_{\in(0,1)}
\prod_{t=t_{out}+1}^{t_2}\underbrace{y_\phi(t)}_{\in(0,\eps]}\\
&\intr\left(0,\eps^{t_2-t_{out}}\right]\subset [0,\eps]\subseteq [0,\delta].
\end{align*}
For \(\eps\leq\min\left\{\delta,1-\left(1-\delta\right)^{\frac{1}{t_{out}-t_{in}+1}}\right\}\) the LSTM cell allows \LTD.\\
{\bf \COD:} To prove that the LSTM cell has no \COD, we show that there are gate activations such that in Definition \ref{defi:COG} we get \(\delta_2>\delta_1\). Therefore, we assume that all gate activations are arbitrary (\(y_\gamma(t)\in(0,1)\)), closed (\(y_\gamma(t)\in(0,\eps]\)) or opened (\(y_\gamma(t)\in[1-\eps,1)\)) with a later chosen \(\eps>0\). We take a look at the right side of \eqref{E:05_LSTM_derivation}. For \(s_c(t)=0\) we get \(\foutder{s_c(t)}=1\). In Definition \ref{defi:COG} we have to satisfy \(\exists \fy(t):\frac{\partial y_c\left(t\right)}{\partial s_c\left(t\right)}\in\left[0,\delta_2\right]\) an choose the OG \(y_\omega(t)\in(0,\eps]\) with
\begin{align}
\label{E:05_COG_proof}
\eps\leq\delta_2.   
\end{align}
But then for \(t'=1,\ldots,t-1\) we can choose the IG and FG open with the same \(\eps\) so that 
\begin{align*}
y_\phi(t'),y_\iota(t')\in\left[1-\eps,1\right).
\end{align*}
When for all time steps \(t'=1,\ldots,t\) there is a positive input \(\yin(t')\in[c,1),c\in(0,1)\subset\IR\) and an internal state \(s_c(t'-1)< c\frac{(1-\eps)}\eps\), the internal state is growing over time, because
\begin{align*}
   s_c(t')&=\yin(t')y_\iota(t')+s_c(t'-1)y_\phi(t')\\
&\geq c(1-\eps)+s_c(t'-1)(1-\eps)\\
&\geq s_c(t'-1) + c(1-\eps)- s_c(t'-1)\eps\\
&> s_c(t'-1) + c(1-\eps)- c\frac{(1-\eps)}{\eps}\eps\\
&> s_c(t'-1).
\end{align*}
For large \(s_c(t)\geq c\frac{(1-\eps)}\eps\gg1\) we can estimate
\begin{align*}
   \underbrace{\tanh(s_c(t))}_{\approx\exp\left(-2s_c(t)\right)} \leq\exp\left(-s_c(t)\right)\leq\exp\left(-c\frac{(1-\eps)}\eps\right).
\end{align*}
This yields in \eqref{E:05_LSTM_derivation} to the bound
\begin{align}
   \left|\frac{\partial y_c\left(t\right)}{\partial s_c\left(t\right)}\right|&=\left|\foutder{s_c(t)}y_\omega(t)\right|\\
&\leq \exp\left(-c\frac{(1-\eps)}\eps\right)
\end{align}
so in Definition \ref{defi:COG} we get
\begin{align}
\label{E:05_COG_proof_1}
   \delta_1\leq\exp\left(-c\frac{(1-\eps)}\eps\right).
\end{align}
But when we combine \eqref{E:05_COG_proof}, \eqref{E:05_COG_proof_1} and the restriction in Definition \ref{defi:COG}, we have
\begin{align*}
   \eps\leq\delta_2<\delta_1\leq\exp\left(-c\frac{(1-\eps)}\eps\right),
\end{align*}
but there exist \(\eps,c\), such that the inequality is not fulfilled, which is a contradiction.\\
Summarized, the \oneD LSTM cell allows an \LTD and has an \NGEC, but does not allow \COD.
\end{proof}
\subsection{Proof of \ref{theo:07_LTD_MD_LSTM}}
\label{proof:theo:07_LTD_MD_LSTM}
\begin{proof}
   Let \(c\) be an \multiD LSTM cell of dimension \(D\), \(\fp,\fp_1,\fp_2,\fp_{in},\fp_{out}\in\IN^D, \fp_{in}\leq \fp_{out}\) arbitrary dates and \(\foutshort=\tanh\) the sigmoid function. Besides \(\eps>0\) is a later chosen value. In the first step we want to show that there are activations of the forget gates, so that
\begin{align}
\label{E:07_LTD_MD_proof}
\frac{\partial s_c^{\fp}}{\partial s_c^{\fp_{in}}}\intr\left\{ \begin{array}{ccl}
 \left[(1-\eps)^{\left\|\fp-\fp_{in}\right\|_1},1\right]&\text{for}&\fp_{in}\leq \fp\leq \fp_{out}\\
\left[0,D\eps\right]&\text{otherwise}&\quad
\end{array}\right.
\end{align}
is fulfilled. The prove is done using induction over \(k=\|\fp-\fp_{in}\|_1\) with \(\fp\geq \fp_{in}\). The base \(k=0\) is clear. Let be \(k\geq1\). We define
\begin{align*}
   P_\fp:=\left\{d\in\left\{1,\ldots,D\right\}\;|\;\fp_d^-\geq\fp_{in}\right\}
\end{align*}
the set of dimensions \(d\), in which are \(\fp_{in}\)-\(\fp_d^-\)-paths. Note, that this set cannot be empty, because \(\fp>\fp_{in}\) for \(k\geq1\). When we have a dimension \(d\in P_\fp\) then \(\left\|\fp_d^- -\fp_{in}\right\|_1=k-1\) and we assume
\begin{align}
\label{E:induction_case}
   \frac{\partial s_c^{\fp_d^-}}{\partial s_c^{\fp_{in}}}\intr\left[\left(1-\eps\right)^{\left\|\fp_d^- -\fp_{in}\right\|_1},1\right]=\left[\left(1-\eps\right)^{k-1},1\right].
\end{align}
Then we choose the activations of the FG to be
\begin{align}
\label{E:values_fg_}
   y_{\phi,d}^\fp\in\left\{ \begin{array}{ccl}
 \left[\frac{1-\eps}{\left|P_\fp\right|},\frac{1}{\left|P_\fp\right|}\right)&\text{for}& d\in P_\fp\text{ and }\fp_{in}< \fp\leq \fp_{out}\\
\left[0,\eps\right]&\text{otherwise}&\quad
\end{array}\right..
\end{align}
Then we can estimate the \der for \(\fp_{in}\leq\fp\leq\fp_{out}\) using \eqref{E:induction_case} and \eqref{E:values_fg_} to
\begin{align}
\label{E:07_LTD_MD_proof_1}
   \frac{\partial s_c^\fp}{\partial s_c^{\fp_{in}}}&\eqtr\sum_{d\in P_\fp} \frac{\partial s_c^{\fp_d^-}}{\partial s_c^{\fp_{in}}}y_{\phi,d}^\fp\in\left[\sum_{d\in P_\fp} \frac{\partial s_c^{\fp_d^-}}{\partial s_c^{\fp_{in}}}\frac{1-\eps}{\left|P_\fp\right|},\sum_{d\in P_\fp} \frac{\partial s_c^{\fp_d^-}}{\partial s_c^{\fp_{in}}}\frac{1}{\left|P_\fp\right|}\right)\notag\\
\Rightarrow\frac{\partial s_c^\fp}{\partial s_c^{\fp_{in}}}&\intr\left[\left|P_\fp\right|\left(1-\eps\right)^{k-1}\frac{1-\eps}{\left|P_\fp\right|},\left|P_\fp\right|\frac{1}{\left|P_\fp\right|}\right)\notag\\
\Leftrightarrow\frac{\partial s_c^\fp}{\partial s_c^{\fp_{in}}}&\intr\left[\left(1-\eps\right)^{\|\fp-\fp_{in}\|_1},1\right),
\end{align}
so \eqref{E:07_LTD_MD_proof} is fulfilled for \(\fp_{in}\leq\fp\leq\fp_{out}\).\\
If we have \(\fp<\fp_{in}\) in \eqref{E:07_LTD_MD_proof}, the \der is \(0\), because we have a causal system.\\ For \(\fp>\fp_{out}\) in \eqref{E:07_LTD_MD_proof}, we choose \(\eps\leq\frac{1}{D}\leq\frac{1}{\left|P_\fp\right|}\) in \eqref{E:values_fg_} to ensure \(\forall\fp\in\IN^D\):\(\left|\frac{\partial s_c^\fp}{\partial s_c^{\fp_{in}}}\right|\leq1\) (see \eqref{E:07_LTD_MD_proof_1}) and we get 
\begin{align}
\label{E:07_LTD_MD_proof_2}
   \frac{\partial s_c^\fp}{\partial s_c^{\fp_{in}}}&\eqtr\sum_{d=1,\ldots,D} \frac{\partial s_c^{\fp_d^-}}{\partial s_c^{\fp_{in}}}y_{\phi,d}^\fp
\in\left(0,D\eps \max_{d=1,\ldots,D} \frac{\partial s_c^{\fp_d^-}}{\partial s_c^{\fp_{in}}}\right]\subseteq\left(0,D\eps\right],
\end{align}
and \eqref{E:07_LTD_MD_proof} is fulfilled.\\
In the second step let \(\fp_1\leq\fp_2\) be the date, for which we want to calculate the truncated gradient \(\frac{\partial s_c^{\fp_2}}{\partial \yin^{\fp_1}}\).
We choose the IG activation as 
\begin{align}
\label{E:07_LTD_MD_proof_3}
y_\iota^\fp&\in\left\{ \begin{array}{ccl} [1-\eps,1) &\text{ if }& \fp=\fp_{in} \\ (0,\eps]&\text{otherwise}&\end{array} \right.
\end{align}
and we get \(\frac{\partial s_c^{\fp}}{\partial \yin^{\fp}}=y_\iota^\fp\). Using \eqref{E:07_LTD_MD_proof_1}, \eqref{E:07_LTD_MD_proof_2} and \eqref{E:07_LTD_MD_proof_3}, we can estimate the partial \der by
\begin{align*}
   \frac{\partial s_c^{\fp_2}}{\partial \yin^{\fp_1}}&=
\frac{\partial s_c^{\fp_2}}{\partial s_c^{\fp_{1}}}
\frac{\partial s_c^{\fp_{1}}}{\partial \yin^{\fp_{1}}}\\
\Rightarrow\frac{\partial s_c^{\fp_2}}{\partial s_c^{\fp_{1}}}&\intr\left\{ \begin{array}{ccl}
 \left[(1-\eps)(1-\eps)^{\left\|\fp_2-\fp_{in}\right\|_1},1\right]&\text{for}&\fp_1=\fp_{in}\text{ and }\fp_{in}\leq \fp_2\leq \fp_{out}\\
\left[0,D\eps\right]&\text{otherwise}&\quad\\
\end{array} \right..
\end{align*}
and setting
\begin{align*}
\eps:=\min\left\{\frac{\delta}{D},1-\left(1-\delta\right)^{\frac{1}{\left\|\fp_{in}-\fp_{out}\right\|_1+1}}\right\}
\end{align*}
the conditions of Definition \ref{defi:LTD_MD} are fulfilled.
\end{proof}
\subsection{Proof of \ref{theo:07_NGEC_MD_LSTM}}
\label{proof:theo:07_NGEC_MD_LSTM}
\begin{proof}
    Let \(c\) be an \multiD cell of dimension \(D\) with the internal state \(s_c\) and \(\fp_{in},\fp_k\in\IN^D, \fp_{in}\leq \fp_k\) two dates. Let \(\fp_k\)  be a date \(k\) steps further in each dimension than a fixed date \(\fp_{in}\). So the distance between them is \(\left\|\fp_{in}-\fp_k\right\|_1=Dk\). Let \(\Pi\) be the set of all \(\fp_{in}\)-\(\fp_k\)-paths, then there exist \(\left|\Pi\right|=\#\{\overrightarrow{\fp_{in}\fp_k}\}\) paths (see Definition \ref{defi:md_path}). We assume
\begin{align*}
   y_{\phi,d}^\fp\in\left[\eps,1-\eps\right]
\end{align*}
with \(\eps\in(0,0.5)\) and we can estimate the partial \der, using the truncated gradient, with
\begin{align*}
   \frac{\partial s_c^{\fp_k}}{\partial s_c^{\fp_{in}}}&\eqtr\sum_{\pi\in\Pi}\prod_{i=1}^k y_{\phi,d}^{\pi_i}\\
&\intr\left[\eps^k\#\{\overrightarrow{\fp_{in}\fp_k}\},\left(1-\eps\right)^k\#\{\overrightarrow{\fp_{in}\fp_k}\}\right].
\end{align*}
For \(D=1\) we get \(\left|\Pi\right|=1\) and the cell has a NGEC. When \(D\geq2\) we can count the number of paths using the Stirling's approximation and we can estimate the number of paths with
\begin{align*}
 \#\{\overrightarrow{\fp_{in}\fp_k}\}=\frac{\left(\sum\limits_{i=1}^D\left(\overrightarrow{\fp_{in}\fp_k}\right)_i\right)!}{\prod\limits_{i=1}^D\left(\overrightarrow{\fp_{in}\fp_k}\right)_i!}=\frac{\left(Dk\right)!}{\left(k!\right)^D}
\xrightarrow{k\gg1}\frac{\sqrt{2\pi Dk}\left(\frac{Dk}{e}\right)^{Dk}}{\left(\sqrt{2\pi k}\left(\frac{k}{e}\right)^k\right)^D}
=\frac{\sqrt{D}D^{Dk}}{\sqrt{2\pi k}^{D-1}}.
\end{align*}
When we combine it with the FG activations we can estimate the \der for great \(k\) with the Stirling's approximation and get 
\begin{align}
\label{E:partial_derivation_dim_m}
   \frac{\partial s_c^{\fp_k}}{\partial s_c^{\fp_{in}}}&\intr\left[\eps^{Dk}\#\{\overrightarrow{\fp_{in}\fp_k}\},\left(1-\eps\right)^{Dk}\#\{\overrightarrow{\fp_{in}\fp_k}\}\right]\\
\stackrel{k\gg 1}{\Rightarrow}&\intr\left[\frac{\sqrt{D}}{\sqrt{2\pi k}^{D-1}}\left(D\eps\right)^{Dk},\frac{\sqrt{D}}{\sqrt{2\pi k}^{D-1}}\left(D\left(1-\eps\right)\right)^{Dk}\right].\notag
\end{align}
The upper bound of this interval can grow for great \(k\), if \(\left[D\left(1-\eps\right)\right]>1\) and this is the case for \(D\geq2\).
So the \multiD LSTM cell can have an exploding gradient for \(D\geq2\). When the weights to the FGs are initialized with small values, we have \(y_{\phi,d}^\fp\approx0.5\). Then we have an exploding gradient when \(D\geq3\), when the training is starting. In the worst case we have \(y_{\phi,d}^\fp\approx1\) and the \der in \eqref{E:partial_derivation_dim_m} goes for great \(k\) to
 \begin{align*}
   \frac{\partial s_c^{\fp_k}}{\partial s_c^{\fp_{in}}}\approx\frac{\sqrt{D}}{\sqrt{2\pi}^{D-1}}k^{\frac{1-D}{2}}\left(D\right)^{Dk}.
\end{align*}
\end{proof}
\subsection{Proof of \ref{theo:MD_LTD_Stable}}
\label{proof:theo:MD_LTD_Stable}
\begin{proof}
Let \(c\) be an \multiD LSTM Stable cell of dimension \(D\geq2\) (for \(D=1\) the proof is equivalent to the \oneD case of the LSTM cell), \(\fp,\fp_1,\fp_2,\fp_{in},\fp_{out}\in\IN^D, \fp_{in}\leq \fp_{out}\) arbitrary dates and \(\foutshort=\tanh\) the sigmoid function. Besides \(\eps>0\) is a later chosen value.\\
In the first step we want to show that there are activations of the forget gates, so that
\begin{align}
\label{E:09_LTD_MD_proof}
\frac{\partial s_c^{\fp}}{\partial s_c^{\fp_{in}}}\intr\left\{ \begin{array}{ccl}
 \left[(1-(D-1)\eps)^{2\left\|\fp-\fp_{in}\right\|_1},1\right]&\text{for}&\fp_{in}\leq \fp\leq \fp_{out}\\
\left[0,\eps\right]&\text{otherwise}&\quad
\end{array}\right.
\end{align}
is fulfilled. The prove is done using induction over \(k=\|\fp_-\fp_{in}\|_1\). The base \(k=0\) is clear. Let be \(k\geq1\). We define
\begin{align*}
   P_\fp:=\left\{d\in\left\{1,\ldots,D\right\}\;|\;\fp_d^-\geq\fp_{in}\right\}
\end{align*}
the set of dimensions \(d\), in which are \(\fp_{in}\)-\(\fp_d^-\)-paths. Note, that this set cannot be empty, because \(\fp>\fp_{in}\) for \(k\geq1\). When we have a dimension \(d\in P_\fp\) then \(\left\|\fp_d^- -\fp_{in}\right\|_1=k-1\) and we assume
\begin{align}
\label{E:09_induction_case}
   \frac{\partial s_c^{\fp_d^-}}{\partial s_c^{\fp_{in}}}\intr\left[\left(1-(D-1)\eps\right)^{2\left\|\fp_d^- -\fp_{in}\right\|_1},1\right]=\left[\left(1-(D-1)\eps\right)^{2(k-1)},1\right].
\end{align}
When we choose the activations of the LGs to be
\begin{align*}
   y_{\lambda,d}^\fp&\in\left\{ \begin{array}{ccl}
 \left[1-\eps,1\right)&\text{for}& d\in P_\fp\text{ and }\fp_{in}< \fp\leq \fp_{out}\\
\left(0,\eps\right]&\text{otherwise}&\quad
\end{array}\right.,
\end{align*}
we can estimate \(\frac{\sum_{d\in P_\fp} y_{\la,d}^\fp}{\sum_{d'=1}^D y_{\la,d'}^\fp}\in\left(1-(D-1)\eps,1\right]\), because
\begin{align}
\label{E:09_values_lg}
1\geq\frac{\sum_{d\in P_\fp} y_{\la,d}^\fp}{\sum_{d'=1}^D y_{\la,d'}^\fp}
&=\frac{\sum_{d\in P_\fp} y_{\la,d}^\fp}{\sum_{d\in P_\fp} y_{\la,d}^\fp+\sum_{d\in\{1,\ldots,D\}\setminus P_\fp} y_{\la,d'}^\fp}\\
&\geq\frac{\left|P_\fp\right|(1-\eps)}{\left|P_\fp\right|(1-\eps)+\underbrace{\left(D-\left|P_\fp\right|\right)}_{\leq D-1}\eps}\notag\\
&\geq\frac{\left|P_\fp\right|(1-(D-1)\eps)}{\left|P_\fp\right|(1-(D-1)\eps)+\left(D-1\right)\eps}\notag\\
&\geq(1-(D-1)\eps)\frac{\left|P_\fp\right|}{\left|P_\fp\right|-\eps(D-1)\left(\left|P_\fp\right|-1\right)}\notag\\
&\geq(1-(D-1)\eps)\notag.
\end{align}
Setting the FG to
\begin{align}
\label{E:09_values_fg_}
   y_\phi^\fp&\in\left\{ \begin{array}{ccl}
 \left[1-\eps,1\right)&\text{for}& \fp_{in}< \fp\leq \fp_{out}\\
\left(0,\eps\right]&\text{otherwise}&\quad
\end{array}\right.
\end{align}
we can estimate the \der for \(\fp_{in}\leq\fp\leq\fp_{out}\) using \eqref{E:09_induction_case},\eqref{E:09_values_lg} and \eqref{E:09_values_fg_} to
\begin{align}
\label{E:09_LTD_MD_proof_1}
   \frac{\partial s_c^\fp}{\partial s_c^{\fp_{in}}}
&\eqtr y_\phi^\fp\left(\sum\limits_{d\in P_\fp} \frac{\partial s_c^{\fp_d^-}}{\partial s_c^{\fp_{in}}}\frac{y_{\la,d}^\fp}{\sum_{d'=1}^D y_{\la,d'}^\fp}+\underbrace{\sum_{d\in\{1,\ldots,D\}\setminus P_\fp} \frac{\partial s_c^{\fp_d^-}}{\partial s_c^{\fp_{in}}} \frac{y_{\la,d}^\fp}{\sum_{d'=1}^D y_{\la,d'}^\fp}}_{=0}\right)\notag\\
&\intr\left((1-\eps)\left(1-(D-1)\eps\right)^{2(k-1)}(1-(D-1)\eps),1\right)\notag\\
\Rightarrow\quad\frac{\partial s_c^\fp}{\partial s_c^{\fp_{in}}}&\intr\left(\left(1-(D-1)\eps\right)^{2k},1\right)
\end{align}
so \eqref{E:09_LTD_MD_proof} is fulfilled for \(\fp_{in}\leq\fp\leq\fp_{out}\).\\
If we have \(\fp<\fp_{in}\) in \eqref{E:09_LTD_MD_proof}, the \der is \(0\), because we have a causal system.\\
For \(\fp>\fp_{out}\) the FG is closed (see \eqref{E:09_values_fg_}), and using the upper bounds of \eqref{E:09_induction_case} and \eqref{E:09_values_lg} we get
\begin{align}
\label{E:09_LTD_MD_proof_2}
   \frac{\partial s_c^\fp}{\partial s_c^{\fp_{in}}}
&\eqtr y_\phi^\fp\left(\sum\limits_{d=1}^D \frac{\partial s_c^{\fp_d^-}}{\partial s_c^{\fp_{in}}}\frac{y_{\la,d}^\fp}{\sum_{d'=1}^D y_{\la,d'}^\fp}\right)\\
&\intr(0,\eps]\notag
\end{align}
and \eqref{E:09_LTD_MD_proof} is fulfilled.\\
In the second step let \(\fp_1\leq\fp_2\) be the date, for which we want to calculate the truncated gradient \(\frac{\partial s_c^{\fp_2}}{\partial \yin^{\fp_1}}\). We choose the IG activation as 
\begin{align}
\label{E:09_LTD_MD_proof_3}
y_\iota^\fp&\in\left\{ \begin{array}{ccl} [1-\eps,1) &\text{ if }& \fp=\fp_{in} \\ (0,\eps]&\text{otherwise}&\end{array} \right.
\end{align}
and we get \(\frac{\partial s_c^{\fp}}{\partial \yin^{\fp}}=y_\iota^\fp\).
Using \eqref{E:09_LTD_MD_proof_1}, \eqref{E:09_LTD_MD_proof_2} and \eqref{E:09_LTD_MD_proof_3}, we can estimate the partial \der by
\begin{align*}
   \frac{\partial s_c^{\fp_2}}{\partial \yin^{\fp_1}}&=
\frac{\partial s_c^{\fp_2}}{\partial s_c^{\fp_{1}}}
\frac{\partial s_c^{\fp_{1}}}{\partial \yin^{\fp_{1}}}\\
\Rightarrow\frac{\partial s_c^{\fp_2}}{\partial s_c^{\fp_{1}}}&\intr\left\{ \begin{array}{ccl}
 \left[(1-\eps)(1-(D-1)\eps)^{2\left\|\fp_2-\fp_{in}\right\|_1},1\right]&\text{for}&\fp_1=\fp_{in}\text{ and }\fp_{in}\leq \fp_2\leq \fp_{out}\\
\left[0,\eps\right]&\text{otherwise}&\quad\\
\end{array} \right..
\end{align*}
and setting
\begin{align*}
\eps:=\min\left\{\delta,\left(1-\left(1-\delta\right)^{\frac{1}{2\left\|\fp_{in}-\fp_{out}\right\|_1+1}}\right)\frac{1}{D-1}\right\}
\end{align*}
the conditions of Definition \ref{defi:LTD_MD} are fulfilled.
\end{proof}
\subsection{Proof of \ref{theo:MD_NGEC_Stable}}
\label{proof:theo:MD_NGEC_Stable}
\begin{proof}
    Let \(c\) be a \multiD LSTM Stable cell of dimension \(D\) with the internal state \(s_c\) and \(\fp_{in},\fp\in\IN^D, \fp_{in}\leq \fp\) two arbitrary dates and \(\left\|\fp_{in}-\fp\right\|_1=k\). Let all gate activations be arbitrary in \(\left[0,1\right]\). We show that
\begin{align}
\label{E:induction_1}
   \frac{\partial s_c^{\fp}}{\partial s_c^{\fp_{in}}}&\intr\left[0,1\right]
\end{align}
is fulfilled \(\forall k\in\IN\) using induction over \(k\). For the base case \(k=0\) we get \(\frac{\partial s_c^{\fp}}{\partial s_c^{\fp_{in}}}=\frac{\partial s_c^{\fp_{in}}}{\partial s_c^{\fp_{in}}}=1\).\\
Let \eqref{E:induction_1} be fulfilled for \(k-1\). That means if \( \fp_d^-\geq \fp_{in}\) we have \(\left\|\fp_d^- -\fp_{in}\right\|_1=k-1\) and this leads to \(\frac{\partial s_c^{\fp_d^-}}{\partial s_c^{\fp_{in}}}\intr\left[0,1\right]\). If \( \fp_d^-\ngeq \fp_{in}\) then there is no \(\fp_{in}\)-\(\fp_d^-\)-path and we have \(\frac{\partial s_c^{\fp_d^-}}{\partial s_c^{\fp_{in}}}=0\) for this dimension. Then we can calculate the \der
\begin{align*}0
\frac{\partial s_c^\fp}{\partial s_c^{\fp_{in}}}&\eqtr
y_\phi^\fp\sum_{d=1}^D \frac{\partial s_c^{\fp_d^-}}{\partial s_c^{\fp_{in}}}\frac{y_{\lambda,d}^\fp}{\sum\limits_{d'=1}^D y_{\la,d'}^\fp} \in\Bigg[0,\max_{d\in P_\fp}\Bigg\{ \frac{\partial s_c^{\fp_d^-}}{\partial s_c^{\fp_{in}}}\underbrace{\frac{y_{\lambda,d}^\fp}{\sum\limits_{d'=1}^D y_{\la,d'}^\fp}}_{\leq1}\Bigg\}\Bigg]\\
&\intr\left[0,1\right],
\end{align*}
which gives us the desired interval.
\end{proof}
\subsection{Proof of \ref{theo:11_Leaky}}
\label{proof:theo:11_Leaky}
\begin{proof}\\
{\bf\NGEC:} The cell has an \NGEC, because all gates have the same bounds as the \multiD Stable cell.\\
{\bf\LTD:} To prove the \LTD, we use the proof of Theorem \ref{theo:MD_LTD_Stable}. The difference between the \multiD Stable cell and the \multiD Leaky cell is that the activations of the FG and IG are dependent on each other for the Leaky cell. Let \(\fp_{in},\fp\in\IN^D,\fp_{in}\leq\fp\) be two arbitrary dates like in Theorem \ref{theo:MD_LTD_Stable}. The IG has just the a restriction that for \(\fp=\fp_{in}\) it has to hold \(y_\iota^\fp\in\left[1-\eps,1\right)\) . Here, the FG can have an arbitrary activation, so we chose \(y_\phi^\fp=1-y_\iota^\fp\). For all \(\fp>\fp_{in}\) the FG have to be in the ranges, shown in \eqref{E:09_values_fg_}, while the IG has no restriction and we choose \(y_\iota^\fp=1-y_\phi^\fp\), so the \multiD Leaky cell has the \LTD.\\
{\bf\COD:} The proof that the \multiD Leaky cell allows \COD can be done by estimating the bounds of \(s_c^\fp\). From the update equations of the cell we get
\begin{align*}
   \left|s_c^{\fp^-}\right|\leq\max_{i=1,\ldots,D}\left|s_c^{\fp_d^-}\right|.
\end{align*}
Now we can estimate the internal state using the ranges \(\yin^\fp\in\left[-1,1\right]\), recursion over \(\fp\)
\begin{align*}
   \left|s_c^\fp\right|=\left|\left(1-y_\phi^\fp\right) \yin^\fp + y_\phi^\fp s_c^{\fp^-}\right|\leq\max\left\{\left|\yin^\fp\right|,\left|s_c^{\fp^-_1}\right|,\ldots,\left|s_c^{\fp^-_D}\right|\right\}\leq\max_{\fq<\fp}\left\{\left|\yin^\fq\right|\right\}\leq 1
\end{align*}
and get \(s_c^\fp\in\left[-1,1\right]\). To fulfill the derivatives in Definition \ref{defi:COG_MD}, for \(\delta_1\) we choose \(y_\omega^\fp\in\left[1-\eps,1\right)\) and get
\begin{align}
\label{E:19_COG_LEAKY_OPEN}
\delta_1\leq\min\limits_{s_c^\fp}\left\{\foutder{s_c^\fp}\right\}(1-\eps)=\foutder{1}(1-\eps).
\end{align}
For \(\delta_2\) we choose \(y_\omega^\fp\in(0,\eps]\) and get 
\begin{align}
\label{E:19_COG_LEAKY_CLOSE}
\delta_2\geq\max\limits_{s_c^\fp}\left\{\foutder{s_c^\fp}\right\}\eps=\eps.
\end{align}
To fulfill the derivatives in Definition \ref{defi:COG_MD} we use \eqref{E:19_COG_LEAKY_OPEN}, \eqref{E:19_COG_LEAKY_CLOSE} and \(\foutder{1}>\frac13\) and with
\begin{align*}
   \eps\leq \delta_2<\delta_1&\leq\foutder{1}\left(1-\eps\right)\\
\Rightarrow \eps&\leq \frac14< \frac{\foutder{1}}{\foutder{1}+1}
\end{align*}
the \COD is proven.
\end{proof}
\end{section}

\begin{section}{Theory to Create First Order \multiD Cells}
\label{ss:theory_1st_order}
If one wants to take a closer look at the theory of linear shift invariant (LSI)-systems and their frequency analysis and analyse a first order LSI-system regarding its free selectable parameters using the \(\FF\)- and \(\ZZ\)-transform, it is highly recommended to be familar with these theories \citep[for a good overview and more details see][]{ap00,ds00}.
Adding the knowledge of reducing the \multiD case to the \oneD case (see Section \ref{S:Stable}) we create new cell types for the \multiD case.\\
\subsection{Analysing a First Order LSI-System}

The update equations of a first order LSI-system with one input \(u\), one internal state \(x\) and one output \(y\) can be written as
\begin{align}
\label{E:13_lsi_x}
x[n]&=h_1(u[n],x[n-1])=\alpha_0u[n]+\alpha_1x[n-1],\\
\label{E:13_lsi_y}
y[n]&=h_2(x[n],x[n-1])=b_0x[n]+b_1x[n-1]
\end{align}
with the free selectable coefficients \(\alpha_0,\alpha_1,b_0,b_1\in \IR\). Let \(U(z)=\ZZ\left\{u[n]\right\}\) be the \(\ZZ\)-transformed signal of \(u[n]\) and \(X(z),Y(z)\) respectively. Then we can write the so called \emph{transfer functions}
\begin{align*}
   H_1(z)&:=\frac{X(z)}{U(z)}=\frac{\alpha_0}{1-\alpha_1 z^{-1}},\\
   H_2(z)&:=\frac{Y(z)}{X(z)}=b_0+b_1 z^{-1},\\
	H(z)&:=\frac{Y(z)}{U(z)}=H_2(z)H_1(z).
\end{align*}
To analyse \eqref{E:13_lsi_x} and \eqref{E:13_lsi_y} according their \emph{frequency response} we use the relationship between the \(\FF\)-transform and the \(\ZZ\)-transform:
\begin{remark}
\label{rem:transferfunction}
Let \(u[n]=e^{j\omega n}\) be a harmonic input sequence with the imaginary number \(j^2=-1\) and \(H(z)=\frac{Y(z)}{U(z)}\) be a transfer functions of an LSI-system. When the poles of \(H(z)\) are inside the circle \(|z|=1\), we can  change from \(\ZZ\)- to \(\FF\)-transform using the substitution
\begin{align*}
   H(\omega)=H(z)\biggr\rvert_{z=e^{j\omega}}
\end{align*}
with the harmonic sequence \(y[n]=H(\omega)u[n]=H(\omega)e^{j\omega n}\) with the same frequency \(\omega\) but with a different amplitude and a different phase dependent on the frequency \(\omega\).
\end{remark}
We only want to analyse the amplitude of this harmonic sequence
\begin{align*}
\left|y[n]\right|=\left|H(z)e^{j\omega n}\right|=\left|H(z)\right|=\left|H_2(z)\right|\left|H_1(z)\right|
\end{align*}
and do that by analysing both transfer functions \(H_1(z)\) and \(H_2(z)\) separately.\\
The amplitude of \(H_1(\omega)=\left.H_1(z)\right|_{z=e^{j\omega}}\) is calculated by
\begin{align*}
   \left|H_1(\omega)\right|=\frac{\left|\alpha_0\right|}{\sqrt{\left(1-\alpha_1 \cos(\omega)\right)^2+\alpha_1^2 \sin^2(\omega)}}.
\end{align*}
Like mentioned before, in many tasks, the information signal has a low frequency. To have the largest amplitude at \(\omega=0\) we have to choose \(\alpha_1\geq0\).
As mentioned in Remark \ref{rem:transferfunction} the poles of \(H_1(z)=\frac{\alpha_0}{1-\alpha_1 z^{-1}}=\frac{z\alpha_0}{z-\alpha_1}\) have to be in the circle \(|z|=1\), so we have the additional constraint \(|\alpha_1|<1\).
This leads to the bounds \(\alpha_1\in[0,1)\). But for \(\alpha_1\to1\) we have \(H_1\left(0\right)\to\infty\), so we have to choose \(\alpha_0\) dependent on \(\alpha_1\). We set a maximum gain of \(\max\limits_{\omega}\left|H_1(\omega)\right|=\left|H_1(0)\right|=1\), so we get the constraint
\begin{align}
\label{E:leakyLP_alphabound}   
\left|\alpha_0\right|\leq1-\alpha_1.
\end{align}
In the same way we analyse \(H_2(z)\):
\begin{align*}
   \left|H_2(\omega)\right|&=\left|b_0+b_1 e^{-j\omega}\right|=\sqrt{\left(b_0+b_1\cos(\omega)\right)^2+b_1^2\sin^2(\omega)}
\end{align*}
To get the maximal gain at low frequency the parameters \(b_0\) and \(b_1\) must have the same sign.
\subsection{Creating a First Order Cell}
\label{S:13_leakyLP_creating_a_first_order_cell}
{
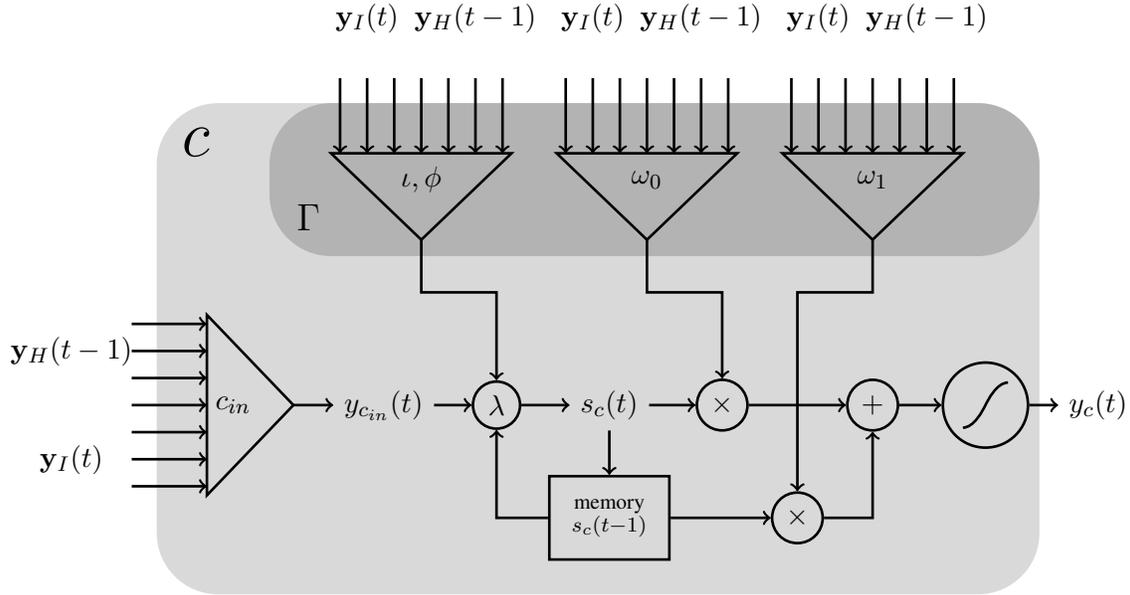
\begin{figure}
\centering
\begin{tikzpicture}[scale=1,line width =1pt]
\draw[fill,color=black!15,rounded corners=8mm] (1.5,-2.5) rectangle  (13.2,4);
\draw[fill,color=black!30,rounded corners=8mm] (3,2) rectangle  (13.2,4);
\node at (3.5,2.5) {\Large{$\Gamma$}};
\node at (2,3.5) {\Huge{$c$}};
\drawuniteasy{2.5}{0}{$c_{in}$}{0};
\drawuniteasy{5}{3}{$\iota,\phi$}{-90};
\drawuniteasy{8}{3}{$\omega_0$}{-90};
\drawuniteasy{11}{3}{$\omega_1$}{-90};
 \node (times_in) at (6,0) [circle,draw,inner sep=1mm] {$\lambda$};
  \node (past_2_plus) at (9,0) [circle,draw,inner sep=1mm] {$\times$};
 \node (currstate) at (7.5,0) {$s_c(t)$};
 \node (input) at (4.5,0) {$\yin(t)$};
 \node (squash_state) at (12.5,0) [circle,draw,inner sep=4mm] {};
 \sigmoid[(12.5,0)]
 \node (times_og) at (10,-1.5) [circle,draw,inner sep=1mm] {$\times$};
 \node (add) at (11,0) [circle,draw,inner sep=1mm] {$+$};
 \node (out) at (14,0) {$y_c(t)$};
\node (state1) at (7.5,-1.5)  [rectangle,draw,inner sep=3mm] {$\substack{\text{memory}\\s_c(t-1)}$};
 
\draw[->] (2.5+0.8,0) -- (input);
\draw[->] (input) -- (times_in);
\draw[->] (5,3-0.8) -- (5,1.5) -- (6,1.5) -- (times_in);
\draw[->] (8,3-0.8) -- (8,1.5) -- (9,1.5)-- (past_2_plus);
\draw[->] (11,3-0.8) -- (11,1.5) -- (10,1.5) -- (times_og);
\draw[->] (times_og) -- (11,-1.5) -- (add);
\draw[->] (past_2_plus) -- (add);
\draw[->]  (add) -- (squash_state);
\draw[->]  (squash_state) -- (out);

\draw[->]  (times_in) -- (currstate); 
 \draw[->] (currstate) -- (past_2_plus); 
 \draw[->] (currstate) -- (state1);
 \draw[->] (state1) -- (6,-1.5) -- (times_in);
 \draw[->] (state1) -- (times_og);
\end{tikzpicture}
\caption{{Schematic diagram of a one-dimensional LeakyLP cell:} The internal state is a convex combination of the new input \(c_{in}\) and the previous state \(s_c(t-1)\). The previous state \(s_c(t-1)\) and the current state \(s_c(t)\) are gated (\(\omega_0\) and \(\omega_1\)) and accumulated afterwards. The output is squashed by \(\tanh\) into the interval \([-1,1]\).}
\label{Fig:LeakyLP_1d}
\end{figure}
}
With these constraints for the parameters we now can define a new cell type. The parameters \(\alpha_0,\alpha_1,b_0,b_1\) should be activations of gates like in LSTM cells.
We have to find the right activation functions to fulfill the inequalities above.
Using the weight-space symmetries in a network with at least one hidden layer \citep[5.1.1]{cb06}, without loss of generality we set \(\alpha_0,\alpha_1,b_0,b_1\geq0\). To fulfill the bounds for \(H_1\), we set \(\alpha_1\) as activation of a gate with activation function \(\fl\). So we have \(\alpha_1\in\left(0,1\right)\). This is comparable with the FG in the previous sections.
To select the \(\alpha_0\) we choose \(\alpha_0:=1-\alpha_1\) (see \eqref{E:leakyLP_alphabound}). So the value of \(\alpha_0\) is comparable with the activation of the IG.
For \(H_2\) we set both values \(b_0,b_1\) as activations of a gate with activation function \(\fl\) which leads to \(\max\limits_\omega \left|H_2(\omega)\right|=\max\left\{b_0+b_1\right\}=2\), so the amplitude response is bounded by \(2\).\\

With these bounds we can define a cell with a cell input \(\yin^\fp=u[n]\), a previous internal state \(s_c^{\fp^-}=x[n-1]\), an internal state \(s_c^\fp=x[n]\) and a cell output \(y_c^\fp=y[n]\). We substitute the coefficients by time dependent gate activations
\begin{center}
\begin{tabular}{ll}
\(\alpha_0:=1-y_\phi^\fp=y_\iota^\fp\)&IG\\
\(\alpha_1:=y_\phi^\fp\)&FG\\
\(b_0:=y_{\omega_0}^\fp\)&OG\\
\(b_1:=y_{\omega_1}^\fp\)&OG of the previous internal state
\end{tabular}
\end{center}
which leads to the transfer functions
\begin{align}
   H_1^{y_\phi^\fp}(z)&=\frac{\alpha_0}{1-\alpha_1z^{-1}}=\frac{1-y_\phi^\fp}{1-y_\phi^\fp z^{-1}},\notag\\
   H_2^{y_{\omega_0}^\fp;y_{\omega_1}^\fp}(z)&=b_0+b_1 z^{-1}=y_{\omega_0}^\fp+y_{\omega_1}^\fp z^{-1},\notag\\
\label{E:H_z_filter}
   H(z)=H^{y_\phi^\fp;y_{\omega_0}^\fp;y_{\omega_1}^\fp}(z)&=\frac{Y(z)}{U(z)}=\alpha_0\frac{b_0+b_1 z^{-1}}{1-\alpha_1z^{-1}}=\frac{1-y_\phi^\fp}{1-y_\phi^\fp z^{-1}}\left(y_{\omega_0}^\fp+y_{\omega_1}^\fp z^{-1}\right).
\end{align}
and the update equations
\begin{align}
x[n]=&\alpha_0 u[n]+ \alpha_1 x[n-1]&\Leftrightarrow s_c^\fp&=(1-y_\phi^\fp)\yin^\fp+y_\phi^\fp s_c^{\fp^-},\notag\\
\label{E:update_equation_y}
y[n]=& b_0 x[n]+b_1 x[n-1]&\Leftrightarrow y_c^\fp&= y_{\omega_0}^\fp s_c^\fp+y_{\omega_1}^\fp s_c^{\fp^-}.
\end{align}
\begin{figure}[t]
\centering
\begin{tabular}{lr}
\begin{tikzpicture}
		\begin{axis}[width=0.45\textwidth,height=0.25\textheight,
			x tick label style={/pgf/number format/1000 sep=},
			y tick label style={/pgf/number format/1000 sep=},
			ylabel={frequency response},
			xlabel={frequency \(f\)},
 			xmin = 0.0,
 			xmax = 0.5,
 			ymin = 0,
 			ymax = 1,
ytick={0,0.5,1,1.5,2,2.5,3},
xtick={0,0.1,0.2,0.3,0.4,0.5},
legend style= {at={(0.5,-0.32)},anchor= north,legend columns=3},
ymajorgrids,
xmajorgrids,
			enlarge x limits=0.0]
			\addplot[line width=1pt,mark=none,black,dash pattern=on 4pt off 2pt] table[x=f, y=H_1_01] \tablefrequencyresponse; \addlegendentry{\(H_1^{0.1}\)}
			\addplot[line width=1pt,mark=none,gray,dash pattern=on 4pt off 2pt] table[x=f, y=H_1_05] \tablefrequencyresponse; \addlegendentry{\(H_1^{0.5}\)}
			\addplot[line width=1pt,mark=none,lightgray,dash pattern=on 4pt off 2pt] table[x=f, y=H_1_09] \tablefrequencyresponse; \addlegendentry{\(H_1^{0.9}\)}
		\end{axis}
\end{tikzpicture}&
\begin{tikzpicture}
		\begin{axis}[width=0.45\textwidth,height=0.25\textheight,
			x tick label style={/pgf/number format/1000 sep=},
			y tick label style={/pgf/number format/1000 sep=},
			ylabel={frequency response},
			xlabel={frequency \(f\)},
 			xmin = 0.0,
 			xmax = 0.5,
 			ymin = 0,
 			ymax = 1,
ytick={0,0.5,1,1.5,2,2.5,3},
xtick={0,0.1,0.2,0.3,0.4,0.5},
legend style= {at={(0.5,-0.32)},anchor= north,legend columns=3},
ymajorgrids,
xmajorgrids,
			enlarge x limits=0.0]			
			\addplot[line width=1pt,mark=none,black,dotted] table[x=f, y=H_2_01] \tablefrequencyresponse; \addlegendentry{\(H_2^{0.9;0.1}\)}
			\addplot[line width=1pt,mark=none,gray,dotted] table[x=f, y=H_2_02] \tablefrequencyresponse; \addlegendentry{\(H_2^{0.8;0.2}\)}
			\addplot[line width=1pt,mark=none,lightgray,dotted] table[x=f, y=H_2_05] \tablefrequencyresponse; \addlegendentry{\(H_2^{0.5;0.5}\)}
		\end{axis}
\end{tikzpicture}\\
\begin{tikzpicture}
		\begin{axis}[width=0.45\textwidth,height=0.25\textheight,
			x tick label style={/pgf/number format/1000 sep=},
			y tick label style={/pgf/number format/1000 sep=},
			ylabel={frequency response},
			xlabel={frequency \(f\)},
 			xmin = 0.0,
 			xmax = 0.5,
 			ymin = 0,
 			ymax = 1,
ytick={0,0.5,1,1.5,2,2.5,3},
xtick={0,0.1,0.2,0.3,0.4,0.5},
legend style= {at={(0.5,-0.32)},anchor= north,legend columns=2},
ymajorgrids,
xmajorgrids,
			enlarge x limits=0.0]			
			\addplot[line width=1pt,mark=none,black,dash pattern=on 4pt off 2pt] table[x=f, y=H_1_05] \tablefrequencyresponse; \addlegendentry{\(H_1^{0.5}\)}
			\addplot[line width=1pt,mark=none,black,dotted] table[x=f, y=H_2_05] \tablefrequencyresponse; \addlegendentry{\(H_2^{0.5;0.5}\)}
			\addplot[line width=1pt,mark=none,black] table[x=f, y=H_05_05] \tablefrequencyresponse; \addlegendentry{\(H^{0.5;0.5;0.5}\)}
		\end{axis}
\end{tikzpicture}&
\begin{tikzpicture}
		\begin{axis}[width=0.45\textwidth,height=0.25\textheight,
			x tick label style={/pgf/number format/1000 sep=},
			y tick label style={/pgf/number format/1000 sep=},
			ylabel={frequency response},
			xlabel={frequency \(f\)},
 			xmin = 0.0,
 			xmax = 0.5,
 			ymin = 0,
 			ymax = 1,
ytick={0,0.5,1,1.5,2,2.5,3},
xtick={0,0.1,0.2,0.3,0.4,0.5},
legend style= {at={(0.5,-0.32)},anchor= north,legend columns=2},
ymajorgrids,
xmajorgrids,
			enlarge x limits=0.0]			
			\addplot[line width=1pt,mark=none,black,dash pattern=on 4pt off 2pt] table[x=f, y=H_1_01] \tablefrequencyresponse; \addlegendentry{\(H_1^{0.1}\)}
			\addplot[line width=1pt,mark=none,black,dotted] table[x=f, y=H_2_02] \tablefrequencyresponse; \addlegendentry{\(H_2^{0.8;0.2}\)}
			\addplot[line width=1pt,mark=none,black] table[x=f, y=H_01_02] \tablefrequencyresponse; \addlegendentry{\(H^{0.1;0.8;0.2}\)}
		\end{axis}
\end{tikzpicture}
\end{tabular}
\caption{Frequency response of \(H_1\) (dashed), \(H_2\) (dotted) and \(H\) (solid) for special parameters. Top-left: The frequency response of an IIR filter is able to block even low frequency signals, but it cannot be zero at \(f=0.5\). Top-right: The frequency response of an FIR filter cannot be lower than the lightgray dotted line, but for \(f=0.5\) it can be zero. Bottom: When these both filters are concatenated, the resulting frequency response can combine the benefits of each filter.}
\label{Fig:frequency_response}
\end{figure}
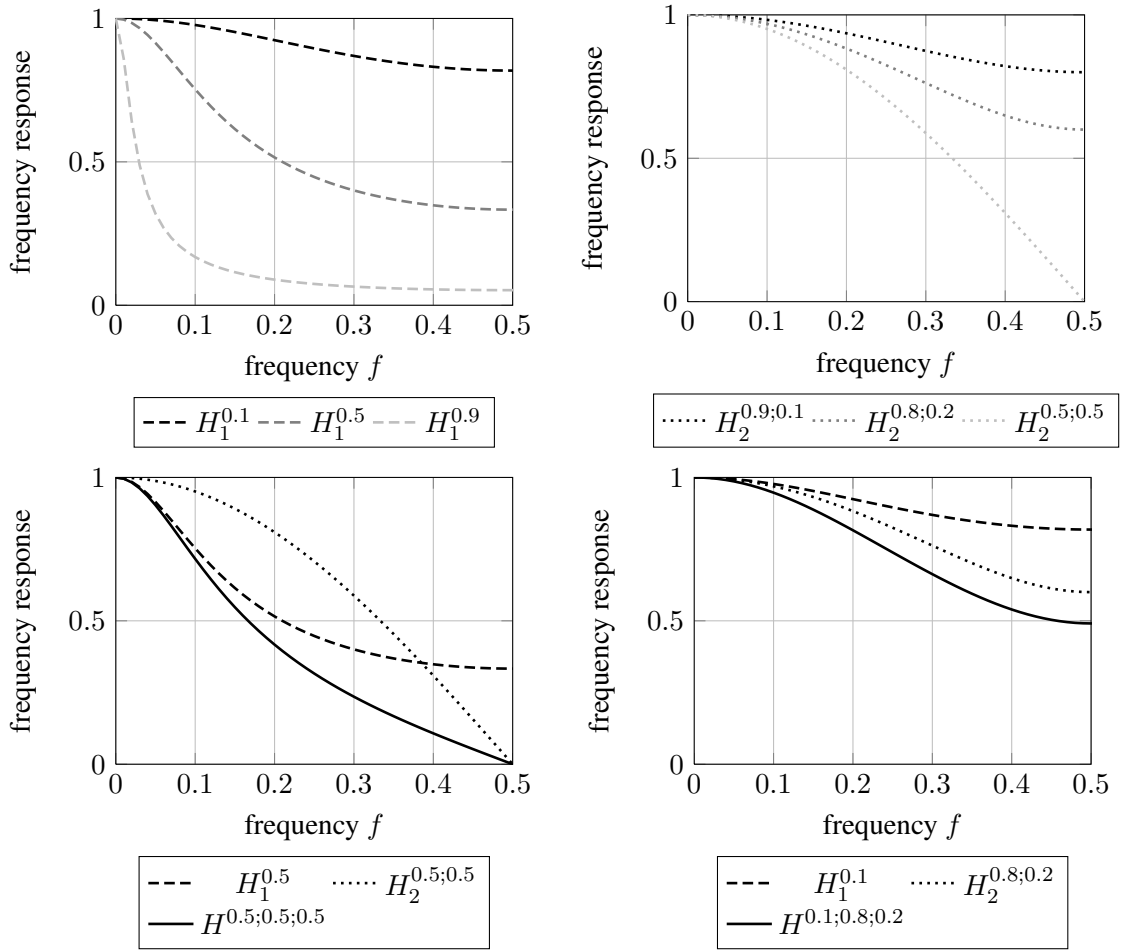
The output of the cell is already bounded in \([-2,2]\), but to fulfill Definition \ref{Def:md_cell} we change \eqref{E:update_equation_y} to
\begin{align}
\label{E:13_leakyLP_output_equation}
   y_c^\fp=\fout{y_{\omega_0}^\fp s_c^\fp+y_{\omega_1}^\fp s_c^{\fp^-}}
\end{align}
with \(\foutshort=\tanh\) to ensure \(y_c^\fp\in\left[-1,1\right]\). This additional non-linearity is not necessary but leads to a better performance. This new cell type called \emph{\multiD Leaky lowpass (LeakyLP) cell} is defined in Definition \ref{defi:md_leakylp_cell}. A block diagram of a \oneD LeakyLP cell is shown in Figure \ref{Fig:LeakyLP_1d} and different frequency responses in Figure \ref{Fig:frequency_response}.
\subsection{General First Order \multiD Cells}
With the theory of this section we can easily create new cell types. In general, a cell has a number of gates \(\gamma_1,\gamma_2,\ldots \in\Gamma_c\). For \(D=1\) a previous state \(s_c^{\fp^-}\) is given directly. Otherwise the previous state is calculated as trainable convex combination of \(D\) previous states, like described in Section \ref{S:Stable}.
In Table \ref{T:cells} cell layouts are depicted whereby type A is the cell developed in Section \ref{S:13_leakyLP} (compare to \eqref{E:H_z_filter}).
For the other types we briefly want to describe the main ideas.
\subsubsection{The \multiD Butterworth Lowpass Filter}
The cell of type B is a special case of the LeakyLP cell. When we set \(y_{\omega_0}^\fp=y_{\omega_1}^\fp=0.5\) there is a direct relation between the cutoff frequency of a discrete Butterworth lowpass filter and the activation of \(y_\phi^\fp\): Let \(f_\text{cutoff}\) be the frequency, where amplitude response is reduced to \(\frac{1}{\sqrt{2}}\) of the maximal gain. We can calculate \(f_\text{cutoff}\) by
\begin{align}
\label{E:f_cutoff}
   f_\text{cutoff}&=\frac{1}{\pi}\arctan\left(\frac{1-y_\phi^\fp}{1+y_\phi^\fp}\right)\\
\Leftrightarrow y_\phi^\fp&=\frac{1+\tan(\pi f_\text{cutoff})}{1-\tan(\pi f_\text{cutoff})}\notag
\end{align}
with the bounds \(f_\text{cutoff}\in(0,0.5)\) and \(y_\phi^\fp\in(-1,1)\) \citep[for more details see][2.2;6.4.2]{ds00}. For \(y_\phi^\fp\in\left(0,1\right)\) we get \(f_\text{cutoff}\in\left(0,0.25\right)\). In Figure \ref{Fig:frequency_response2} (left) we can see, that even for a negative value of \(y_\phi^\fp\) and a highpass characteristic of \(H_1(z)\) the impulse response \(H(z)\) has a lowpass characteristic.
\subsubsection{Adding an Additional State Gate}
In \ref{S:13_leakyLP_creating_a_first_order_cell} we fulfilled \eqref{E:leakyLP_alphabound} for the \multiD LeakyLP cell by setting \(\alpha_0:=1-\alpha_1\), so \(\alpha_0\) is directly connected with \(\alpha_1\). Another solution would be to add an additional value \(\gamma\in\left(0,1\right)\) and choose \(\alpha_0:=\gamma\left(1-\alpha_1\right)\).
So we can extend the \multiD LeakyLP cell by adding an additional gate \(\gamma_4\) for the previous state (see type C). Unfortunately this does not lead to a better performance and one more gate has to be calculated.
\subsubsection{Another Solution for the Output}
The cell of type D is another solution to choose \(b_0\) and \(b_1\) in Section \ref{S:13_leakyLP_creating_a_first_order_cell}. For the LeakyLP cell we calculate the output as described in \eqref{E:13_leakyLP_output_equation}. Now we set \(b_0=\gamma_2^\fp\gamma_3^\fp\) and \(b_1=\left(1-
\gamma_2^\fp\right)\gamma_3^\fp\), and get
\begin{align*}
y_c^\fp=\gamma_3^\fp\left(\gamma_2^\fp s_c^\fp+\left(1-\gamma_2^\fp\right)s_c^{\fp^-}\right).   
\end{align*}
This cell actually works as well as the \multiD LeakyLP cell and has the same number of gates. In this case we do not need a squashing function \(\foutshort\), because we already have \(y_c^\fp\in[-1,1]\).
\subsubsection{An \multiD Cell as \multiD PID-Controller}
Type E has a completely different interpretation: In controlling engineering a PID-controller gets an error as input.
In our case the gate activations have to decide, if the proportional (P), the integral (I) or the derivative (D) term of the error is important for the output.
When \(\gamma_1^\fp\approx 0\) we have \(\yin^\fp\approx s_c^\fp\) so the internal state is proportional to the input.
Then \(\gamma_2^\fp\) gates the proportional part (P) of the input.
The second gate \(\gamma_3^\fp\) gates the difference between the last and the current input, which can be seen as a discrete \der (D).
If \(\gamma_1^\fp\approx 1\) the internal state is an exponential moving average of \(\yin^\fp\) which is an integral term.
So \(\gamma_2^\fp\) gates a mainly integral part of the input (I), whereas \(\gamma_3^\fp\) gates a mainly proportional part of the input (P).
Dependent on \(\gamma_1^\fp\) type E can be a PD-controller, a PI-controller or a mix of these both.
In Figure \ref{Fig:frequency_response2}(right) can be seen the frequency response of this cell for different gate activations. 
\begin{table}
\begin{center}
\resizebox{\textwidth}{!}{
\begin{tabular}{clcc}
\hline\hline
Type&\(\frein{\cdot}\)&\(\fraus{\cdot}\) &\(H(z)\) for \(\fout{x}=x\) \\ \hline
A&\(\left(1-\gamma_1^\fp\right)\yin^\fp+\gamma_1^\fp s_c^{\fp^-}\)&\(\fout{\gamma_2^\fp s_c^\fp+\gamma_3^\fp s_c^{\fp^-}}\)&\(\frac{(1-\gamma_1^\fp)}{1-\gamma_1^\fp z^{-1}}\left(\gamma_2^\fp + \gamma_3^\fp z^{-1}\right)\)\\
B&\(\left(1-\gamma_1^\fp\right)\yin^\fp+\gamma_1^\fp s_c^{\fp^-}\)&\(\frac{s_c^\fp+s_c^{\fp^-}}2\)&\(\frac{(1-\gamma_1^\fp)}{1-\gamma_1^\fp z^{-1}}\frac{1+z^{-1}}{2}\)\\
C&\(\left(1-\gamma_1^\fp\right)\yin^\fp+\gamma_1^\fp\gamma_4^\fp s_c^{\fp^-}\)&\(\fout{\gamma_2^\fp s_c^\fp+\gamma_4^\fp s_c^{\fp^-}}\)&\(\frac{(1-\gamma_1^\fp)}{1-\gamma_1^\fp\gamma_4^\fp z^{-1}}\left(\gamma_2^\fp + \gamma_3^\fp z^{-1}\right)\)\\
D&\(\left(1-\gamma_1^\fp\right)\yin^\fp+\gamma_1^\fp s_c^{\fp^-}\)&\(\gamma_3^\fp\left(\gamma_2^\fp s_c^\fp+\left(1-\gamma_2^\fp\right) s_c^{\fp^-}\right)\)&\(\frac{(1-\gamma_1^\fp)}{1-\gamma_1^\fp z^{-1}}\gamma_3^\fp\left(\gamma_2^\fp+\left(1-\gamma_2^\fp\right) z^{-1}\right)\)\\
E&\(\left(1-\gamma_1^\fp\right)\yin^\fp+\gamma_1^\fp s_c^{\fp^-}\)&\(\fout{\gamma_2^\fp s_c^\fp+\gamma_3^\fp\left( s_c^\fp-s_c^{\fp^-}\right)}\)&\(\frac{(1-\gamma_1^\fp)}{1-\gamma_1^\fp z^{-1}}\left(\gamma_2^\fp+\gamma_3^\fp\left(1- z^{-1}\right)\right)\)\\\hline\hline
\end{tabular}}
\caption{Update equations and transfer function for different cell layouts. The column \(s_c^\fp\) contains the update equations to calculate the internal state  and column \(y_c^\fp\) contains the update equation for the output. These equations lead to the transfer function \(H(z)=H^{\gamma_1^\fp,\gamma_2^\fp,\ldots}(z)\).}
\label{T:cells}
\end{center}
\end{table}
\begin{figure}[t]
\centering
\begin{tabular}{lr}

\begin{tikzpicture}
		\begin{axis}[width=0.45\textwidth,height=0.25\textheight,
			x tick label style={/pgf/number format/1000 sep=},
			y tick label style={/pgf/number format/1000 sep=},
			ylabel={frequency response},
			xlabel={frequency \(f\)},
 			xmin = 0.0,
 			xmax = 0.5,
 			ymin = 0,
 			ymax = 3,
ytick={0,0.5,1,1.5,2,2.5,3},
xtick={0,0.1,0.2,0.3,0.4,0.5},
legend style= {at={(0.5,-0.32)},anchor= north,legend columns=2},
ymajorgrids,
xmajorgrids,
			enlarge x limits=0.0]			
			\addplot[line width=1pt,mark=none,black,dash pattern=on 4pt off 2pt] table[x=f, y=H_1_m05] \tablefrequencyresponse; \addlegendentry{\(H_1^{-0.5}\)}
			\addplot[line width=1pt,mark=none,black,dotted] table[x=f, y=H_2_05] \tablefrequencyresponse; \addlegendentry{\(H_2^{0.5;0.5}\)}
			\addplot[line width=1pt,mark=none,black,solid] table[x=f, y=H_m05_05] \tablefrequencyresponse; \addlegendentry{\(H^{-0.5;0.5;0.5}\)}
			\addplot[line width=1pt,mark=none,lightgray,dotted] coordinates{(0,0.7071) (0.5,0.7071)};
			\addplot[line width=1pt,mark=none,lightgray,dotted] coordinates{(0.3975,0) (0.3975,3)};
		\end{axis}
\end{tikzpicture}&
\begin{tikzpicture}
		\begin{axis}[width=0.45\textwidth,height=0.25\textheight,
			x tick label style={/pgf/number format/1000 sep=},
			y tick label style={/pgf/number format/1000 sep=},
			ylabel={frequency response},
			xlabel={frequency \(f\)},
 			xmin = 0.0,
 			xmax = 0.5,
 			ymin = 0,
 			ymax = 1,
ytick={0,0.25,0.5,0.75,1},
xtick={0,0.1,0.2,0.3,0.4,0.5},
legend style= {at={(0.5,-0.32)},anchor= north,legend columns=2},
ymajorgrids,
xmajorgrids,
			enlarge x limits=0.0]
			\addplot[line width=1pt,mark=none,black,dash pattern=on 4pt off 2pt] table[x=f, y=H_1_05] \tablefrequencyresponse; \addlegendentry{\(H_1^{0.5}\)}
			\addplot[line width=1pt,mark=none,black,solid] table[x=f, y=PID_05_01] \tablefrequencyresponse; \addlegendentry{\(H^{0.5;0.9;0.1}\)}
			\addplot[line width=1.5pt,mark=none,gray,solid] table[x=f, y=PID_05_05] \tablefrequencyresponse; \addlegendentry{\(H^{0.5;0.5;0.5}\)}
			\addplot[line width=1pt,mark=none,lightgray,solid] table[x=f, y=PID_05_09] \tablefrequencyresponse; \addlegendentry{\(H^{0.5;0.1;0.9}\)}
		\end{axis}
\end{tikzpicture}
\end{tabular}
\caption{Frequency response of \(H_1\) (dashed), \(H_2\) (dotted) and \(H\) (solid) for special layouts and parameters of  Table \ref{T:cells}. Left (type B): A butterworth lowpass filter with a negative gate activation \(\gamma_0^\fp=-0.5\) leads to the cutoff frequency \(f_\text{cutoff}=\frac{1}{\pi}\arctan\left(\frac{1+0.5}{1-0.5}\right)\approx0.3976\). Right (type E): Different frequency responses of a PID controller. Having a fixed \(\gamma_0^\fp=0.5\) the frequency response is dependent on the activations of \(\gamma_1^\fp\) and \(\gamma_2^\fp\) and can have lowpass (black), allpass (gray) and highpass (lightgray) characteristic.}
\label{Fig:frequency_response2}
\end{figure}
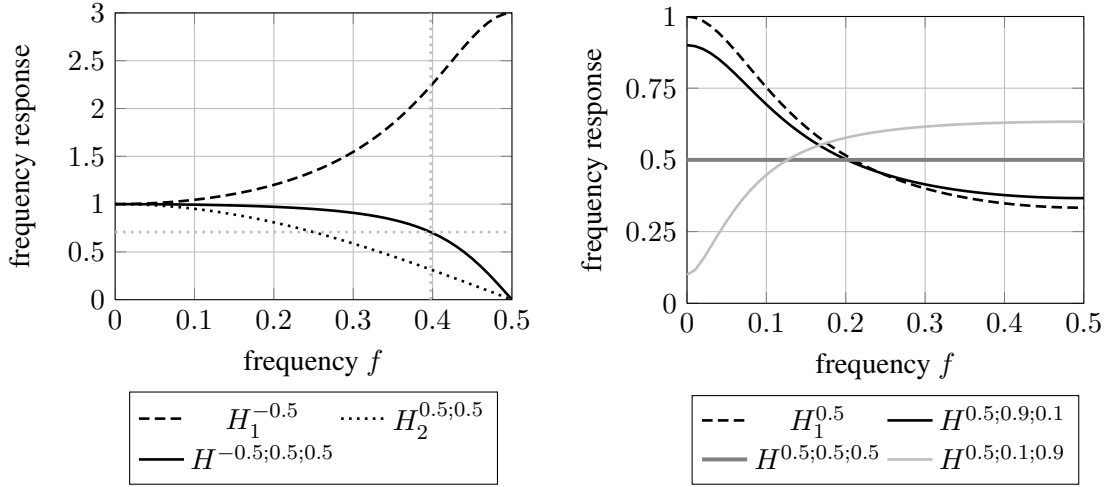
\end{section}

\bibliographystyle{apalike}
\bibliography{bib}

\end{document}